\documentclass{article} 
\usepackage{iclr2024_conference,times}

\usepackage{hyperref}
\usepackage{url}
\usepackage{amsmath}
\usepackage{amssymb}
\usepackage{mathtools}
\usepackage{amsthm}
\usepackage{bm}
\usepackage{glossaries}
\usepackage{multicol}
\usepackage{algorithm}
\usepackage{algpseudocode}
\usepackage[textsize=tiny, textwidth=20mm]{todonotes}

\title{Time-Efficient Reinforcement Learning\\ with Stochastic Stateful Policies}

\author{Firas Al-Hafez$^1$, Guoping Zhao$^2$, Jan Peters$^{1,3}$, Davide Tateo$^1$\\
$^1$ Intelligent Autonomous Systems,
$^2$ Locomotion Laboratory\\
\small$^3$ German Research Center for AI (DFKI),
Centre for Cognitive Science, Hessian.AI\\
\small TU Darmstadt, Germany\\
\texttt{\small\{name.surname\}}\texttt{\small@tu-darmstadt.de} \\
}

\newacronym{rl}{RL}{Reinforcement Learning}
\newacronym{il}{IL}{Imitation Learning}
\newacronym{drl}{DRL}{Deep Reinforcement Learning}

\newacronym{avi}{AVI}{Approximate Value-Iteration}
\newacronym{api}{API}{Approximate Policy-Iteration}
\newacronym[plural=MDPs, firstplural=Markov Decision Processes (MDPs)]{mdp}{MDP}{Markov Decision Process}
\newacronym[plural=POMDPs, firstplural=Partially Observable Markov Decision Processes (POMDPs)]{pomdp}{POMDP}{Partially Observable Markov Decision Process}
\newacronym{kl}{KL}{Kullback-Leibler Divergence}
\newacronym{gae}{GAE}{Generalized Advantage Estimation}
\newacronym{trpo}{TRPO}{Trust Region Policy Optimization}
\newacronym{sac}{SAC}{Soft-Actor Critic}
\newacronym{ddpg}{DDPG}{Deep Deterministic Policy Gradient}
\newacronym{td3}{TD3}{Twin Delayed DDPG}
\newacronym{bptt}{BPTT}{Backpropagation Through Time}
\newacronym{s2pg}{S2PG}{Stochastic Stateful Policy Gradient}
\newacronym{pgt}{PGT}{Policy Gradient Theorem}
\newacronym{dpgt}{DPGT}{Deterministic Policy Gradient Theorem}
\newacronym{ppo}{PPO}{Proximal Policy Optimization}
\newacronym{gail}{GAIL}{Generative Adversarial Imitation Learning}
\newacronym{lsiq}{LS-IQ}{Least-Squares Inverse Q-Learning}
\newacronym[plural=LSTMs]{lstm}{LSTM}{Long-Short Term Memory}
\newacronym[plural=GRUs]{gru}{GRU}{Gated Recurrent Units}
\newacronym{rdpg}{RDPG}{Recurrent Deterministic Policy Gradient}
\newacronym{rsvg}{RSVG}{Recurrent Stochastic Value Gradient}
\newacronym{ode}{ODE}{Ordinary Differential Equation}
\newacronym[plural=FFNs, firstplural=Feed-Forward Networks (FFNs)]{ffn}{FFN}{Feed-Forward Network}

\newacronym{rtd3}{RTD3}{Recurrent Twin-Delayed Deep Deterministic Policy Gradient}
\newacronym{rsac}{RSAC}{Recurrent Soft Actor-Critic}

\newacronym[plural=RNNs, firstplural=Recurrent Neural Networks (RNNs)]{rnn}{RNN}{Recurrent Neural Network}

\newacronym[plural=CPGs, firstplural=Central Pattern Generators (CPGs)]{cpg}{CPG}{Central Pattern Generator}

\newcommand{\cmp}[2]{#1[#2]}

\DeclareMathOperator*{\xpt}{\mathbb{E}}
\DeclareMathOperator*{\var}{Var}

\newcommand{\expect}[2]{\xpt_{\substack{#1}} \left[ #2 \right]}

\newcommand{\probability}[1]{Pr\lbrace #1 \rbrace}

\newcommand{\tndensop}[1]{\mathrm{Tr}^{\pi_\vtheta}_{#1}}
\newcommand{\dettndensop}[1]{\mathrm{Tr}^{\mu_\vtheta}_{#1}}
\newcommand{\tndensity}[1]{\tndensop{#1}(s_{#1},z_{#1})}
\newcommand{\dettndensity}[1]{\dettndensop{#1}(s_{#1},z_{#1})}

\def\eqref#1{equation~\ref{#1}}

\def\1{\bm{1}}

\def\vzero{{\bm{0}}}
\def\vone{{\bm{1}}}

\def\vtheta{{\bm{\theta}}}
\def\vpsi{{\bm{\psi}}}

\def\vsigma{{\bm{\sigma}}}
\def\vupsilon{{\bm{\upsilon}}}

\DeclareMathAlphabet{\mathsfit}{\encodingdefault}{\sfdefault}{m}{sl}
\SetMathAlphabet{\mathsfit}{bold}{\encodingdefault}{\sfdefault}{bx}{n}

\DeclarePairedDelimiterX{\divx}[2]{(}{)}{%
  #1\;\delimsize\|\;#2%
}

\newcommand{\tauh}{\bar{\tau}}
\newcommand{\pibptt}{\nu}

\DeclareMathOperator{\tr}{tr}
\DeclareMathOperator{\diag}{diag}

\newtheoremstyle{theoremcond}
  {\topsep}
  {\topsep}
  {\upshape}
  {0pt}
  {\bfseries}
  {. ---}
  { }
  {\thmname{#1}\thmnumber{ #2}\textnormal{\thmnote{ (#3)}}}

\theoremstyle{theoremcond}
\newtheorem{conditions}{Regularity Conditions}[section]

\newcommand{\fnorm}[1]{\bigl\|#1\bigr\|_{\textnormal{F}}}
\newcommand{\fnorms}[1]{\left\| #1 \right\|_{\textnormal{F}}}
\definecolor{red}{rgb}{1.,0.,0.}
\definecolor{green}{rgb}{0.,1.,0.}
\definecolor{blue}{rgb}{0.,0.,1.}
\definecolor{violet}{rgb}{.5,0,.5}
\definecolor{brown}{rgb}{.75,.5,.25}
\definecolor{orange}{rgb}{1.,.5,0.}
\definecolor{magenta}{rgb}{1.,0.,1.}

{
  \textit{Proof of Proposition {\color{red}{#1}}.} \begin{normalfont}%
}%
{\hfill$\blacksquare$\end{normalfont}}

{
  \textit{Proof of Lemma {\color{red}{#1}}.} \begin{normalfont}%
}%
{\hfill$\blacksquare$\end{normalfont}}

\theoremstyle{plain}
\newtheorem{theorem}{Theorem}[section]

\newtheorem{lemma}[theorem]{Lemma}
\newtheorem{corollary}[theorem]{Corollary}
\theoremstyle{definition}

\newtheorem{assumption}[theorem]{Assumption}
\theoremstyle{remark}

\glsdisablehyper

\iclrfinalcopy 

\begin{document}

\maketitle

\begin{abstract}
Stateful policies play an important role in reinforcement learning, such as handling partially observable environments, enhancing robustness, or imposing an inductive bias directly into the policy structure. The conventional method for training stateful policies is \gls{bptt}, which comes with significant drawbacks, such as slow training due to sequential gradient propagation and the occurrence of vanishing or exploding gradients. The gradient is often truncated to address these issues, resulting in a biased policy update. We present a novel approach for training stateful policies by decomposing the latter into a stochastic internal state kernel and a stateless policy, jointly optimized by following the \emph{stateful policy gradient}. We introduce different versions of the stateful policy gradient theorem, enabling us to easily instantiate stateful variants of popular reinforcement learning and imitation learning algorithms. Furthermore, we provide a theoretical analysis of our new gradient estimator and compare it with \gls{bptt}.
We evaluate our approach on complex continuous control tasks, e.g. humanoid locomotion, and demonstrate that our gradient estimator scales effectively with task complexity while offering a faster and simpler alternative to \gls{bptt}.
\end{abstract}

\section{Introduction}

Stateful policies are a fundamental tool for solving complex \gls{rl} problems. These policies are particularly relevant for \gls{rl} in a \gls{pomdp}, where the history of interactions needs to be processed at each time step. Stateful policies, such as a \gls{rnn}, compress the history into a latent recurrent representation, allowing them to deal with the ambiguity of environment observations. \citet{ni2022recurrent} have shown that \gls{rnn} policies can be competitive and even outperform specialized algorithms in many \gls{pomdp} tasks. Besides \glspl{pomdp}, stateful policies can be used to incorporate inductive biases directly into the policy, e.g. a stateful oscillator to learn locomotion~\cite{ijspeert2007swimming,bellegarda2022cpg}, or to solve Meta-\gls{rl} tasks by observing the history of rewards \cite{ni2022recurrent}.

Existing methods for stateful policy learning mostly rely either on black-box/evolutionary optimizers or on the \gls{bptt} algorithm. While black-box approaches struggle with high-dimensional parameter spaces, making them less suitable for neural approximators, \gls{bptt} has shown considerable success in this domain~\cite{bakker2001reinforcement,wierstra2010recurrent,meng2021memory}. However, \gls{bptt} suffers from significant drawbacks, including the need for sequential gradient propagation through trajectories, which significantly slows down training time and can lead to vanishing or exploding gradients. To address these issues, practical implementations often use a truncated history, introducing a bias into the policy update and limiting memory to the truncation length. Furthermore, integrating \gls{bptt} into standard \gls{rl} algorithms is not straightforward due to the requirement of handling sequences of varying lengths instead of single states.

In this paper, we propose an alternative solution to compute the gradient of a stateful policy. We decompose the stateful policy into a stochastic policy state kernel and a conventional stateless policy, which are jointly optimized by following the \gls{s2pg}. By adopting this approach, we can train arbitrary policies with an internal state without \gls{bptt}. \gls{s2pg} not only provides an unbiased gradient update but also accelerates the training process and avoids issues like vanishing or exploding gradients. In addition, our approach can be applied to any existing \gls{rl} algorithm by modifying a few lines of code. 
We also show how to extend the \gls{s2pg} theory to all modern \gls{rl} approaches by introducing different versions of the policy gradient theorem for stateful policies. This facilitates the instantiation of stateful variations of popular \gls{rl} algorithms like \gls{sac}~\cite{sac}, \gls{td3}~\cite{fujimoto2018}, \gls{ppo}~\cite{ppo}, as well as \gls{il} algorithms like \gls{gail}~\cite{Ho2016} and \gls{lsiq}~\cite{alhafez2023}. Our approach is also applicable to settings where the critic has only access to observations. In such cases, we combine our method with Monte-Carlo rollouts, as typically done in \gls{ppo}. In the off-policy scenario, we use a critic with privileged information, a common setting in robot learning~\cite{pinto2018asymmetric,lee2020learning, peng20182}.

To evaluate our approach, we conduct a theoretical analysis on the variance of \gls{s2pg} and \gls{bptt}, introducing two new bounds and highlighting the behavior of each estimator in different regimes. Empirically, we compare the performances in common continuous control tasks within \glspl{pomdp} when using \gls{rl}. To illustrate the potential use of \gls{s2pg} for inductive biases in policies, we give an example of a trainable \gls{ode} of an oscillator used in the policy together with an \gls{ffn} to achieve walking policies under complete blindness. Finally, We also demonstrate the scalability of our method by introducing two complex locomotion tasks under partial observability and dynamics randomization. Our results indicate that these challenging tasks can be effectively solved using stateful \gls{il} algorithms. Overall, our experiments demonstrate that \gls{s2pg} offers a simple and efficient alternative to \gls{bptt} that scales well with task complexity.

\textbf{Related Work.} \glspl{rnn} are the most popular type of stateful policy. Among \glspl{rnn}, gated recurrent networks, such as \gls{lstm}~\cite{hochreiter1997} or \gls{gru}~\cite{cho2014} networks, are the most prominent ones \cite{ni2022recurrent, wierstra2010recurrent, meng2021memory, heess2015memory, espeholt2018impala, yang2021recurrent}. \citet{wierstra2010recurrent} introduced the \emph{recurrent policy gradient} using \glspl{lstm} and the GPOMDP~\cite{baxter2001} algorithm, where \gls{bptt} was used to train the trajectory. Many subsequent works have adapted \gls{bptt}-based training of \glspl{rnn} to different \gls{rl} algorithms. \citet{heess2015memory} introduced the actor-critic \gls{rdpg} and \gls{rsvg}, which were later updated to the current state-of-the-art algorithms  like \gls{rtd3} and \gls{rsac}~\cite{yang2021recurrent}. Many works utilize separate \glspl{rnn} for the actor and the critic in \gls{rtd3} and \gls{rsac}, as it yields significant performance gains compared to a shared \gls{rnn} \cite{ ni2022recurrent, meng2021memory, heess2015memory, yang2021recurrent}. Although \glspl{lstm} have a more complex structure, \glspl{gru} have demonstrated slightly better performance in continuous control tasks \cite{ni2022recurrent}. Stateful Policies can also be used to encode an inductive bias into the policy. For instance, a \gls{cpg} is a popular choice for locomotion tasks \cite{ijspeert2008central, bellegarda2022visual, campanaro2021cpg}. We conduct experiments with \glspl{cpg} in Section \ref{sec:cpg_exp}, where we also present relevant related work. Using a similar approach to the one presented in this paper, \citet{zhang2016learning} introduce a stochastic internal state transition kernel. In this approach, states are considered as a memory that the policy can read and write. However, they did not extend their approach to the actor-critic case, which is the main focus of our work. Other approaches such as DVRL~\cite{dvrl} and SLAC~\cite{slac} force the recurrent state to be a belief state, exploiting a learned environment model. An alternative to stateful policies, particularly useful in \glspl{pomdp} settings, is to use a history of observations. These policies use either fully connected MLPs, time convolution~\cite{lee2020learning}, or transformers~\cite{lee2023supervised,radosavovic2023learning}. It can be shown that for some class of \glspl{mdp}, these architectures allow learning near-optimal policies~\cite{efroni2022provable}. The main drawback is that they require high dimensional input, store many transitions in the buffer, and cannot encode inductive biases in the latent space.
Finally, the setting where privileged information from the simulation is used is widespread in sim-to-real robot learning \cite{lee2020learning, peng20182}. Compared to so-called teacher-student approaches \cite{lee2020learning}, which learn a privileged policy and then train a recurrent policy using behavioral cloning, our approach can learn a recurrent policy online using privileged information for the critic.

\newpage
\section{Stochastic Stateful Policy Gradients}

\noindent\textbf{Preliminaries.~~} An \glsfirst{mdp} is a tuple $(\mathcal{S}, \mathcal{A}, P, r, \gamma, \iota )$, where $\mathcal{S}$ is the state space, $\mathcal{A}$ is the action space, $P: \mathcal{S} \times \mathcal{A} \times \mathcal{S} \rightarrow \mathbb{R}^+$ is the transition kernel,  $r: \mathcal{S} \times \mathcal{A} \rightarrow \mathbb{R}$ is the reward function, $\gamma$ is the discount factor, and $\iota: \mathcal{S} \rightarrow \mathbb{R}^+$ is the initial state distribution. At each step, the agent observes a state $s \in \mathcal{S}$ from the environment, predicts an action $a \in \mathcal{A}$ using the policy $\pi:\mathcal{S} \times \mathcal{A} \rightarrow \mathbb{R}^+$, and transitions with probability $P(s'|s, a)$ into the next state $s' \in \mathcal{S}$, where it receives the reward $r(s, a)$. 
 We define an occupancy measure \mbox{$\rho^{\pi_{\vtheta}}(s,z) = \sum_{t=0}^{\infty} \gamma^t \probability{s = s_t \wedge z = z_t | \pi_{\vtheta}}$} giving metric on how frequently the tuple $(s, z)$ is visited. 
The \gls{pomdp} is tuple $(\mathcal{S}, \mathcal{A}, \mathcal{O}, P, O, r, \gamma, \iota )$ extending  the \gls{mdp} by the observation space $\mathcal{O}$ and the conditional probability density $O: \mathcal{O}\times\mathcal{S}\rightarrow \mathbb{R}^+$, such that for every state $s\in\mathcal{S}$, the probability density of $o\in\mathcal{O}$ is $O(o|s)$. A trajectory $\tau$ is a (possibly infinite) sequence of states, actions, and observations sampled by interacting with the environment under a policy $\pi$. We define the observed history $h_t=\langle s_0, \dots s_{t}, \rangle$ as the sequence of states taken up to the current timestep $t$. In the partially observable case, the history is given by observations and actions such that $h_t=\langle o_0, a_{0}\dots o_{t-1}, a_{t-1}, o_{t} \rangle$.
In general, an optimal policy for a \gls{pomdp} depends on the full observed history $h_t$, i.e. $a\sim\pi(\cdot|h_t)$. However, most of the time, we can substitute $h_t$ with sufficient statistics $z_t=S(h_t)$, namely the internal state of the policy.

\subsection{Policy Gradient of Stateful Policies}
\citet{wierstra2010recurrent} presented the first implementation of the policy gradient for recurrent policies. This formulation is based on the likelihood ratio trick, i.e. the score function estimator.
To obtain this formulation, the authors consider a generic history-dependant policy $\pibptt(a|h_t)$, obtaining the following gradient formulation
{\small
\begin{equation}
\nabla_\vtheta\mathcal{J}(\pibptt_\vtheta) = \expect{\tau}{\sum_{t=0}^{T-1}\nabla_\vtheta\log\pibptt_\vtheta(a_t|h_t)J(\tau)}.
\label{eq:recurrent_pg}
\end{equation}}\noindent
It is important to notice that each $\nabla_\vtheta\log\pibptt_\vtheta(a_t|h_t)$ term in \eqref{eq:recurrent_pg} depends on the whole history $h_t$ at each timestep $t$. 
If we replace the history $h_t$ with sufficient statistic $z_t=S(h_t)$, then \gls{bptt} is necessary to compute the gradient, as the internal state $z_t$ depends on dynamics induced by the parameters of the policy $\vtheta$. Unfortunately, while this approach provides an easy way to provide unbiased gradient estimates for recurrent policies, this solution does not fit long-horizon tasks well for three different reasons. First, this method is inherently sequential, limiting the utilization of potential gradient computations that could be performed in parallel. Second, when dealing with long trajectories, the gradients may suffer from the problems of exploding or vanishing gradients, which can hinder the learning process. Finally, we must store the complete trajectory sequentially or use truncated histories,  leading to a biased gradient estimate. Unfortunately, this last assumption is problematic for modern actor-critic frameworks. 

To solve these issues, we consider a different policy structure. Instead of looking at policies with an internal state, we model our policy as a joint probability distribution over actions and next internal states, i.e. $(a,z')\sim\pi_\vtheta(\cdot|s,z)$. By considering a stochastic policy state transition, we derive a policy gradient formulation that does not require the propagation of the policy gradient through time but only depends on the local information available, i.e. the (extended) state transition.

\begin{lemma}{\textbf{(Stochastic Stateful Policy Gradient)}}
Let $\pi_\vtheta(a,z'|s,z)$ be a parametric policy representing the joint probability density used to generate the action $a$ and the next internal state $z'$, and let $\tauh$ be the extended trajectory including the internal states. Then, the \gls{s2pg} can be written as
{\small
\begin{equation}
    \nabla_\vtheta\mathcal{J}(\pi_\vtheta) = \expect{\tauh}{\sum_{t=0}^{T-1} \nabla_\vtheta\log\pi_\vtheta(a_t,z_{t+1}|s_t,z_t) J(\tauh)}.\label{eq:stateful_pg}
\end{equation}}
\label{lemma:stateful_pg}
\vspace{-0.3cm}
\end{lemma}
This lemma is straightforward to derive using the classical derivation of the policy gradient. The complete proof can be found in Appendix~\ref{app:proof_pg}. In contrast to the gradient in \eqref{eq:recurrent_pg}, the gradient in \eqref{eq:stateful_pg} depends on the information of the \emph{current} timestep $t$, i.e. $a_t$, $z_{t+1}$,$s_t$ and $z_t$. We can deduce the following from Lemma \ref{lemma:stateful_pg}

\textbf{The gradient of a stateful policy can be computed stochastically}, by 
treating the internal state of policy as a random variable. This modification results in an algorithm that does not require \gls{bptt}, but only access to samples from $z$. 

\textbf{We allow to trade off speed and accuracy while keeping the gradient estimation \emph{unbiased}}. While truncating the history to save computation time biases the gradient, estimating the expectation with a finite amount of samples of $z$ does not. In contrast to \gls{bptt}, this formulation supports parallel computation of the gradient, allowing to fully exploit the parallelization capabilities of modern automatic differentiation tools and parallel simulation environments. As the batch size for gradient computation increases, the variance of our gradient estimator decreases, as shown later in Section \ref{sec:var_analysis}.

\begin{figure*}[t]
    \centering
    \includegraphics[width=\textwidth]{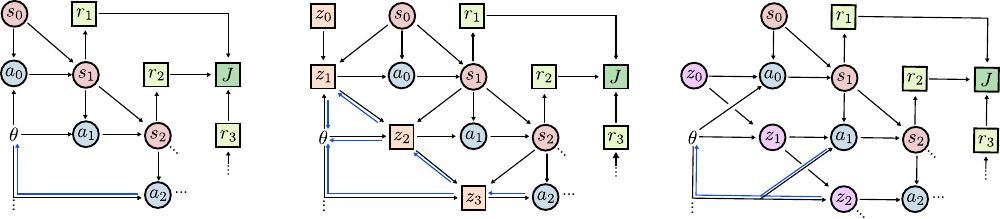}
    \caption{Stochastic computational graph illustrating, from right to left, a comparison of a stateless policy, a stateful policy trained with \gls{bptt}, and a stateful policy incorporating a stochastic internal state transition kernel.  Deterministic nodes (squares) allow for the passage of analytic gradients, while stochastic nodes (circles) interrupt the deterministic paths. The input node to the graph is represented by $\theta$. The blue lines indicate the deterministic paths for which analytic gradients are available for a specific action. For a detailed demonstration of how stochastic node gradients are calculated compared to deterministic ones, refer to Figure \ref{fig:grad_comp} in Appendix \ref{app:gradient_comparison}
}
    \label{fig:stochastic_computation_graphs}
\end{figure*}

To provide an alternative perspective, we present the core principle underlying Lemma \ref{lemma:stateful_pg} within the framework of stochastic computation graphs, as originally proposed by \citet{schulman2015gradient}.  The key rationale behind Lemma \ref{lemma:stateful_pg} is to shield the deterministic path through the sequence of internal states by using a stochastic node, here our stochastic internal state kernel. Figure~\ref{fig:stochastic_computation_graphs} effectively elucidates this notion by contrasting the gradient paths of the internal state $z$ of various policies.
While \gls{s2pg} successfully addresses the issue of exploding and vanishing gradients associated with long sequences, it may be affected by high variance due to the inclusion of a stochastic internal state kernel. To address this, we propose to compensate for the additional variance by using state-of-the-art actor-critic methods while exploiting the same policy structure. 



\subsection{Actor-Critic Methods for Stateful Policies}
To define actor-critic methods, we need to extend the definitions of value functions into the stateful policy setting.
Let $s$ be some state of the environment, $z$ some state of the policy, $a$  a given action, and $z'$ the next policy state. We define the state-action value function of our policy $\pi(a,z'|s,z)$ as
{\small
\begin{align}
Q^{\pi}(s,z,a,z') &= r(s,a) +\gamma \expect{s'}{\expect{a', z''}{Q^{\pi}(s',z',a',z'')}}  = r(s,a) + \gamma \expect{s'}{V^{\pi}(s',z')}
\label{eq:stateful_q}
\end{align}}\noindent
with $s'\sim P(\cdot|s,a)$ and $(a',z'')\sim\pi(\cdot|s',z')$. 

Furthermore, we define the state-value function as
{\small
\begin{equation}
V^{\pi}(s,z) = \expect{a,z'}{Q^{\pi}(s,z,a,z')} \qquad \text{where } (a,z')\sim\pi(\cdot|s,z) \, . 
\label{eq:stateful_v}
\end{equation}}\noindent
Notice that these definitions have the same meaning as the definition of value and action-value function of a Markov policy. Indeed, the value function $V(s,z)$ is the expected discounted return achieved by the policy $\pi$ starting from the state $s$ with initial internal state $z$. Thus, we can write the expected discounted return of a stateful policy as $\mathcal{J}(\pi) = \int \iota(s,z)V(s,z) ds dz$ with the joint initial state distribution $\iota(s,z)$. 

The fundamental theorem of actor-critic algorithms is the \gls{pgt}~\cite{sutton1999policy}. This theorem lays the connection between value functions and policy gradients. Many very successful practical approaches are based on this idea and constitute a fundamental part of modern \gls{rl}. Fortunately, using our stochastic internal state kernel and the definitions of value and action-value function in equations~\eqref{eq:stateful_v} and~\eqref{eq:stateful_q}, we can provide a straightforward derivation of the \gls{pgt} for stateful policies.

\begin{theorem}{\textbf{(Stateful Policy Gradient Theorem)}}
Let $\rho^{\pi_{\vtheta}}(s,z)$ be the occupancy measure and $Q^{\pi_{\vtheta}}(s, z, a, z')$ be the value funciton of a parametric policy $\pi_{\vtheta}(a,z'|s,z)$, then the stateful policy gradient theorem is given by 
{\small
\begin{align*}
\nabla_\vtheta\mathcal{J}_{\vtheta} =& \int_\mathcal{S}\int_\mathcal{Z} \rho^{\pi_{\vtheta}}(s,z) \int_\mathcal{A}\int_\mathcal{Z} \nabla_\vtheta\pi_{\vtheta}(a,z'|s,z) \, Q^{\pi_{\vtheta}}(s, z, a, z')\, dz' \,  da \, dz \,  ds.
\end{align*}}
\label{th:spgt}
\vspace{-0.3cm}
\end{theorem} \noindent
The proof follows the lines of the original \gls{pgt} and is given in Appendix~\ref{app:proof_pgt}. Differently from \mbox{\citet{sutton1999policy}}, which considers only the discrete action setting, we provide the full derivation for continuous state and action spaces. 
Theorem \ref{th:spgt} elucidates the key idea of this work:

\textbf{We learn a $Q$-function to capture the values of internal state transitions}, compensating the additional variance in the policy  gradient estimate induced by our stochastic internal state transition kernel. This approach enables the efficient training of stateful policies without the need for \gls{bptt}, effectively reducing the computational time required to train stateful policies to that of stateless policies. Figure \ref{fig:grad_comp} in Appendix \ref{app:gradient_comparison} further illustrates and reinforces this concept.

While the \gls{pgt} is a fundamental building block of actor-critic methods, some of the more successful approaches, namely the \gls{trpo} ~\cite{schulman2015} and the \gls{ppo}~\cite{ppo} algorithm, are based on the Performance Difference Lemma~\cite{kakade2002approximately}.
We can leverage the simplicity of our policy structure to derive Lemma~\ref{lemma:performance_difference}, a modified version of the original Lemma for stateful policies, allowing us to implement recurrent versions of \gls{ppo} and \gls{trpo} with \gls{s2pg}. The Lemma and the proof can be found in Appendix~\ref{app:proof_perfomance_diff}.

\subsection{The Deterministic Limit}
The last important piece of policy gradient theory is the definition of the \gls{dpgt} \cite{dpg}. This theorem can be seen as the deterministic limit of the \gls{pgt}, at least under some class of probability distributions. It is particularly important as is the basis for many successful actor-critic algorithms, such as \gls{ddpg}~\cite{ddpg} and \gls{td3}~\cite{fujimoto2018}. We can derive the \gls{dpgt} for our stateful policies:

\begin{theorem}{\textbf{(Stateful \glsdesc*{dpgt})}}
Let $\mu_\vtheta(s,z) = \left[\mu_\vtheta^a(s,z), 
     \mu_\vtheta^z(s,z) \right]^\top$ be a deterministic policy, where $\mu_\vtheta^a(s,z)$ represent the action selection model and $\mu_\vtheta^z(s,z)$ represents the internal state transition model. Then, under mild regularity assumptions, the policy gradient of the stateful deterministic policy is
{\small
\begin{align*}
\nabla_\vtheta\mathcal{J}(\mu_\vtheta) = & \int_\mathcal{S}\int_\mathcal{Z} \rho^{\mu_\vtheta}(s,z) \left( \nabla_\vtheta \mu^a_\vtheta(s,z)\nabla_a\left.Q^{\mu_\vtheta}(s,z,a,\mu_\vtheta^z(s,z))\right\vert_{a=\mu_\vtheta^a(s,z)} \right. \\
& \hspace{2.5cm}\left. + \nabla_\vtheta \mu^z_\vtheta(s,z)\left.\nabla_{z'}Q^{\mu_\vtheta}(s,z,\mu_\vtheta^a(s,z),z')\right\vert_{z'=\mu_\vtheta^z(s,z)} \right) ds dz.
\end{align*}} %
\label{th:stateful_dpg} %
\vspace{-0.3cm}
\end{theorem} %
The proof is given in Appendix~\ref{app:proof_dpgt}.
It is important to notice that, while the deterministic policy gradient theorem allows computing the gradient of a deterministic policy, it assumes the knowledge of the $Q$-function for the deterministic policy $\mu_\vtheta$. However, to compute the $Q$-function of a given policy it is necessary to perform exploration. Indeed, to accurately estimate the $Q$-values, actions different from the deterministic policy needs to be taken. 
We want to stress that in this scenario, instead of using the standard $Q$-function  formulation, we use our definition of $Q$-functions for stateful policies: therefore, we must also explore the internal state transition. This key concept explains why in our formulation it is not necessary to perform \gls{bptt} even in the deterministic policy scenario.


\subsection{The Partially Observable Setting}
\label{subsec:partially_observable}
Within this section, we introduce the \gls{pomdp} settings used within this work and adapt the previous theory accordingly. Therefore, we extend Lemma~\ref{lemma:stateful_pg} as follows.
\begin{corollary}
The stateful policy gradient in partially observable environments is
{\small
\begin{equation*}
    \nabla_\vtheta\mathcal{J}(\pi_\vtheta) = \expect{\substack{\tauh \\ o_t \sim O(s_t)}}{\sum_{t=0}^{T-1} \nabla_\vtheta\log\pi_\vtheta(a_t,z_{t+1}|o_t, z_t) J(\tauh)}.
\end{equation*}}
\label{col:pomdp}
\vspace{-0.3cm}
\end{corollary}
The proof is given in Appendix~\ref{app:proof_pg_po}.
Corollary \ref{col:pomdp} offers a particularly intriguing perspective: If we have knowledge of $J(\tauh)$, it is possible to learn an internal state representation $z_t$ that effectively encapsulates past information \emph{without} looking back in time. In other words, it enables the compression of past information without explicitly considering previous time steps. This stands in stark contrast to \gls{bptt}, where all past information is used to build a belief of the current state. Similarly to Corollary \ref{col:pomdp}, we can extend Theorem \ref{th:spgt}
{\small
\begin{align}
\nabla_\vtheta\mathcal{J}_{\vtheta} =& \int_\mathcal{S}\int_\mathcal{O}\int_\mathcal{Z}\rho^{\pi_{\vtheta}}(s,z) \, O(o|s) \int_\mathcal{A}\int_\mathcal{Z} \nabla_\vtheta\pi_{\vtheta}(a,z'|o, z) \, Q^{\pi_{\vtheta}}(s, z, a, z')\, da ds\, . 
\label{eq:spgt_pomdp}
\end{align}}\noindent
It is important to note the distinction between the $Q$-function, which receives states, and the policy, which relies on observations only. While it is possible to learn a $Q$-function on observations via bootstrapping, 
the task of temporal credit assignment in \glspl{pomdp}  is very challenging, often requiring specialized approaches \cite{dvrl, slac}. This difficulty stems from the following causality dilemma: On one hand, the policy can learn a condensed representation of the history assuming accurate $Q$-value estimates. On the other hand, the $Q$-function can be learned using bootstrapping techniques assuming a reliable condensed representation of the history. But doing both at the same time often results in unstable training. We show in Section \ref{sec:exp} that it is possible to combine Monte-Carlo estimates with a critic that uses observations only using \gls{ppo} to learn in a true \gls{pomdp} setting. However, we limit ourselves to the setting with privileged information in the critic, as shown in Equation \ref{eq:spgt_pomdp}, for approaches that learn a value function via bootstrapping.

\section{Variance Analysis of the Gradient Estimators}\label{sec:var_analysis}
To analyze the performance of our estimator from the theoretical point of view, we provide the upper bounds of \gls{bptt} and \gls{s2pg} in the Gaussian policy setting with constant covariance matrix $\Sigma$. As done in previous work \cite{papini2022smoothing,zhao2011analysis}, we define the variance of a random vector as the trace of its covariance matrix (see equation~\eqref{eq:var_trace} in Appendix~\ref{app:variance_analysis}). Furthermore, for both gradient estimators, we assume that the function approximator for the mean has the following structure 
%
{\small
\begin{align}
    \mu^a_\vtheta(s_t, z_t)=f_\vtheta(s_t, z_t) \quad \mu^z_\vtheta(s_t, z_t)= \eta_\vtheta(s_t, z_t) \,\,  \text{(for \gls{s2pg})} \quad z_{t+1} = \eta_\vtheta(s_t, z_t)\,\, \text{(for \gls{bptt})} . 
\end{align}}\noindent
Hence, $f_\vtheta(s_t, z_t)$ constitutes the mean of the Gaussian for both \gls{s2pg} and \gls{bptt}, while $\eta_\vtheta(s_t, z_t)$ constitutes the internal transition function for \gls{bptt} and the mean of the Gaussian distribution of the hidden state in \gls{s2pg}. Figure \ref{fig:network_architectures} in Appendix \ref{app:net_archi} illustrates the policy structure. Additionally, $\Upsilon$ constitutes the covariance matrix for the Gaussian distribution of the hidden state in \gls{s2pg}. Our theoretical analysis is based on prior work \cite{papini2022smoothing,zhao2011analysis}, which analyzes the variance of a REINFORCE-style policy gradient estimator for stateless policies and a univariate Gaussian distribution. We extend the letter to policy gradient estimators using \gls{bptt} and \gls{s2pg}, and broaden the analysis to the more general multivariate Gaussian case to allow investigation on arbitrarily large action spaces. In the following, we derive upper bounds for the variance of the policy gradient for \gls{bptt} and \gls{s2pg} and highlight the behavior of each estimator in different regimes.
Therefore, we exploit the concept of the Frobenius norm of a matrix $\fnorm{A}$ and consider the following assumptions:
\begin{assumption}
$r(s, a, s')\in \left[-R, R\right]$ for $R>0$ and $\forall s_i, z_i, a_i$. \label{assump:reward}
\end{assumption}
\begin{assumption}
$\fnorm{\frac{\partial}{\partial \theta} f_\theta(s_i, z_i)} \leq F,\,  \fnorm{\frac{\partial}{\partial \theta} \eta_\theta(s_{i}, z_{i}) } \leq H,\,  \fnorm{\frac{\partial}{\partial z_i} f_\theta(s_i, z_i)} \leq K \quad \forall s_i, z_i$. \label{assump:grad1}
\end{assumption}
\begin{assumption}
$\fnorm{\frac{\partial}{\partial z_i} \eta_\theta(s_{i}, z_{i})} \leq Z \quad \forall s_i, z_i$.\label{assump:grad2}
\end{assumption}

In this setting, the upper bound on the variance of \gls{bptt} is given by:
\begin{theorem}{\textbf{(Variance Upper Bound \gls{bptt})}}
Let $\nu_\vtheta(a_t | h_t) $ be a Gaussian policy of the form $\mathcal{N}(a_t|\mu^a_\vtheta(h_t), \Sigma)$ and let $T$ be the trajectory length. Then, under the assumptions \ref{assump:reward}, \ref{assump:grad1} and \ref{assump:grad2}, the upper bound on the variance of the REINFORCE-style policy gradient estimate using \gls{bptt} is
{\small
\begin{align}
     \var\left[ \nabla_\theta \hat{J}_{\text{BPTT}}(\Theta) \right]  &\leq \frac{R^2 \fnorm{\Sigma^{-1}} (1-\gamma^T)^2}{N (1-\gamma)^2} \, \sum_{t=0}^{T-1} \biggl( F + H \biggr(\sum_{i=0}^{t-1} Z^{t-i-1} \biggr) K\biggr)^2 \nonumber\\
     &= \frac{R^2 \fnorm{\Sigma^{-1}} (1-\gamma^T)^2}{N (1-\gamma)^2} \, \bigl( TF^2 + \underbrace{2FHK\tilde{Z} + H^2K^2\bar{Z}}_{\Delta_\text{\gls{bptt}}}\bigr) \, , 
     \label{eq:bptt_variance_bound}
\end{align}}
where we define $\tilde{Z} = \sum_{t=0}^{T-1} \sum_{i=0}^{t-1} Z^{t-i-1}$ and $\bar{Z} = \sum_{t=0}^{T-1} \left(\sum_{i=0}^{t-1} Z^{t-i-1}\right)^2$ for brevity. 
\label{lemma:variance_bptt}
\end{theorem}
Similarly, the upper bound on the variance of \gls{s2pg} is given by:
\begin{theorem}{\textbf{(Variance Upper Bound \gls{s2pg})}}
Let $\pi_\vtheta(a_t, z_{t+1} | s_t, z_t) $ be a Gaussian policy of the form $\mathcal{N}(a_t, z_{t+1} | 
 \mu_\vtheta(s_t, z_t), \Sigma, \Upsilon)$ and let $T$ be the trajectory length. Then, under the assumptions \ref{assump:reward} and \ref{assump:grad1}, the upper bound on the variance of the REINFORCE-style policy gradient estimate using \gls{s2pg} is 
\vspace{-0.2cm}
{\small
\begin{align}
    \var\left[ \nabla_\theta \hat{J}_{\text{S2PG}}(\Theta) \right] &\leq \frac{R^2 \fnorm{\Sigma^{-1}}(1-\gamma^T)^2}{N(1-\gamma)^2}
  \sum_{t=0}^{T-1}\biggl(F^2 + 
 H^2 \cdot \tfrac{\fnorms{\Upsilon^{-1}}}{{\fnorms{\Sigma^{-1}}}} \biggr)\nonumber\\
 &=  \frac{R^2 \fnorm{\Sigma^{-1}}(1-\gamma^T)^2}{N(1-\gamma)^2}
  \bigl(TF^2 + \underbrace{
 TH^2 \cdot \tfrac{\fnorms{\Upsilon^{-1}}}{{\fnorms{\Sigma^{-1}}}}}_{\Delta_\text{\gls{s2pg}}} \bigr)
  \, . \label{eq:s2pg_variance_bound}
\end{align}}
 \label{lemma:variance_s2pg}
 \vspace{-0.4cm}
\end{theorem}
The proof of both Theorems can be found in Appendix \ref{proof:bptt_var} and \ref{proof:s2pg_var}. 
When looking at Theorem \ref{lemma:variance_bptt} and \ref{lemma:variance_s2pg}, it can be seen that both bounds consist of a constant factor in front and a sum over $F^2$, which is a constant defining the upper bound on the gradient of $f_\theta(s_i, z_i)$. Additionally, both bounds introduce an additional term, which we consider as $\Delta$. In fact, the latter is a factor defining the additional variance added when compared to the variance of the stateless policy gradient. Hence, for $\Delta = 0$, both bounds reduce to the bounds found by prior work \cite{papini2022smoothing,zhao2011analysis}, when considering the univariate Gaussian case.   


To understand the difference in the variance of both policy gradient estimators, we need to compare $\Delta_\textit{\gls{bptt}}$ and $\Delta_\textit{\gls{s2pg}}$. For \gls{s2pg}, the additional variance is induced by $TH^2$, which is the squared norm of the gradient of internal transition kernel $\eta_\theta(s_i, z_i)$  w.r.t. $\theta$, weighted by the ratio of the norms of the covariance matrices. In contrast, Theorem \ref{lemma:variance_bptt} highlights the effect of backpropagating gradients through a trajectory as $\Delta_\textit{\gls{bptt}}$ additionally depends on the constants $\tilde{Z}$ and $\bar{Z}$. 
First of all, for $H\gg1$, $K\gg1$ and $\bar{Z}\gg 1$, we can observe that the quadratic term $H^2K^2\bar{Z}^2$ dominates $\Delta_{\textit{BPTT}}$. Secondly, we need to distinguish two cases: as detailed in Lemma~\ref{lemma:z_constants} in Appendix~\ref{proof:bptt_var}, for $Z\geq 1$, the term $H^2K^2\bar{Z}^2$ grows exponentially with the length of the trajectories $T$, indicating that exploding gradients cause exploding variance, while, for $Z<1$, the gradients grows linearly with $T$. That is, while the variance of \gls{bptt} heavily depends on the architecture used for $\eta_\theta (s_i, z_i)$, the variance of \gls{s2pg} depends much less on the architecture of the internal transition kernel allowing the use of arbitrary functions for $\eta_\theta (s_i, z_i)$ at the potential cost of higher -- yet not exploding -- variance. We note that, for \gls{s2pg}, we can define a tighter bound than the one in Theorem \ref{lemma:variance_s2pg} under the assumption of diagonal covariance matrices as shown in Lemma \ref{lemma:tighter_bound_s2pg} in Appendix \ref{proof:s2pg_var}.



\section{Experiments}\label{sec:exp}

In this section, we evaluate our method across four different types of tasks. These tasks include MuJoCo Gym tasks with full partial observability, MuJoCo Gym tasks with a partially observable policy and a privileged critic, a memory task, and complex \gls{il} tasks with a partially observable policy and a privileged critic. In Section \ref{sec:cpg_exp}, we show further experiments in which we use the \gls{ode} of a \glspl{cpg} as a stateful policy to induce an inductive bias into the policy. To conduct the evaluation, we compare recurrent versions of three popular \gls{rl} algorithms—\gls{sac}, \gls{td3}, and \gls{ppo} — that employ \gls{s2pg} and \gls{bptt}. In Appendix \ref{app:algos}, we present the algorithm boxes for all \gls{s2pg} variants. Furthermore, we compare against stateless versions of these algorithms, which use a window of the last 5 or 32 observations as the input to the policy. To show the importance of memory, we include basic version of these algorithms, where only the observation is given to the policy (vanilla) and where the full state is given to the policy (oracle). To compare with specialized algorithms for \glspl{pomdp}, we compare with SLAC~\cite{slac}. SLAC is a model-based approach that learns complex latent variable models that are used for the critic, while the policy takes a window of observations as input. Agents that utilize our stochastic policy transition kernel are denoted by the abbreviation ``RS'',  which stands for Recurrent Stochastic. The network architectures are presented in Appendix~\ref{app:net_archi}. The initial internal state is always set to 0. When an algorithm has a replay buffer, we update its internal states when sampling transitions to resemble \gls{bptt}. For a fair comparison, all methods, except SLAC, are implemented in the same framework, MushroomRL \cite{mushroomrl}. To properly evaluate the algorithms' training time, we allocate 4 cores on our cluster and average the results and the training time over 10 seeds for all experiments.

\clearpage
\begin{figure}[t]
    \begin{center}
    \begin{tabular}{ c c }
    \hspace{2cm}\textbf{Privileged Information}\hspace{1cm} & \hspace{2cm} \textbf{No Privileged Information}   \\ 
    \end{tabular}
    \vspace{-5pt}
    \end{center} 
     \includegraphics[width=\textwidth]{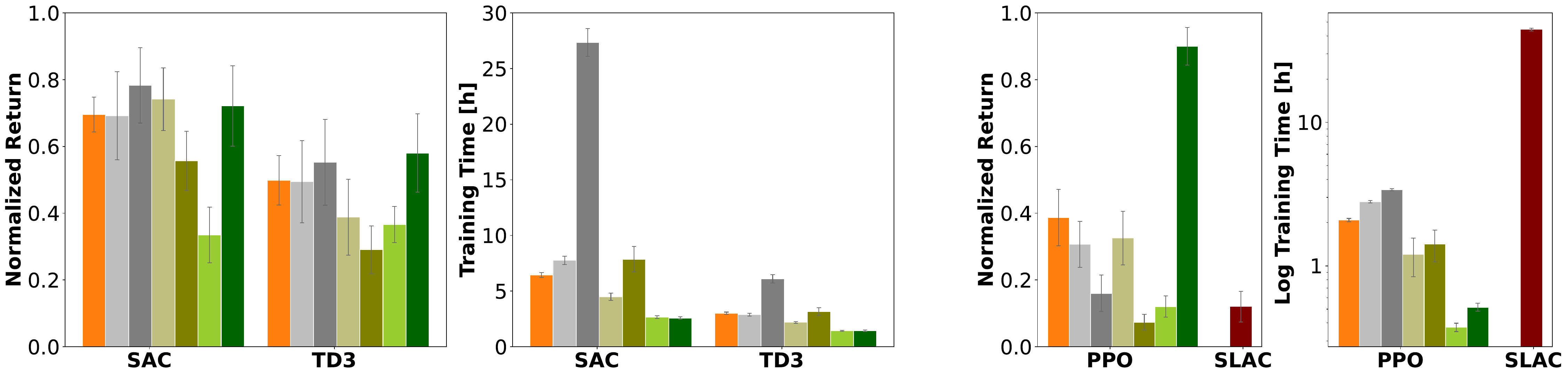}
     \includegraphics[width=\textwidth]{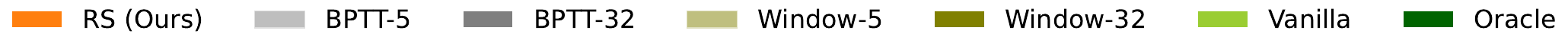}
     \vspace{-15pt}
    \caption{
    Comparison of different \gls{rl} agents on \gls{pomdp} tasks using our stateful gradient estimator, \gls{bptt} with a truncation length of 5 and 32, the SLAC algorithm and stateless versions of the algorithms using a window of observation with length 5 and 32. The "Oracle" approach is a vanilla version of the algorithm using full-state information. Results show the mean and confidence interval of the normalized return and the training time needed for 1 Mio. steps across all \gls{pomdp} Gym tasks. }
    \label{fig:overall_results}
\end{figure}
\textbf{Privileged Critic Tasks. } The first set of tasks includes the typical MuJoCo Gym locomotion tasks to test our approach in the \gls{rl} setting. As done in \cite{ni2022recurrent}, we create partial observability by hiding information, specifically the velocity, from the state space. Furthermore, we randomize the mass of each link. As described in Section~\ref{subsec:partially_observable}, it is challenging to learn a stochastic policy state while using bootstrapping. Therefore, we use a privileged critic with a vanilla \gls{ffn} for off-policy approaches. On the left of Figure~\ref{fig:overall_results}, we show the results for a \gls{td3} and a \gls{sac}-based agents. The overall results show that our approach has major computational benefits w.r.t \gls{bptt} with long histories at the price of a slight drop in asymptotic performance for SAC and TD3. We found that our approach perform better on high-dimensional observation spaces -- e.g., Ant and Humanoid -- and worse on low-dimensional ones. A more detailed discussion about this is given in Appendix \ref{app:rand_dyn_rl_tasks}. While the window approach has computational benefits compared to \gls{bptt}, it did not perform well with increased window length.  The full experimental campaign, learning curves, and a more detailed task descriptions are in Appendix~\ref{app:rand_dyn_rl_tasks}. 

\textbf{No Privileged information Tasks.} The aim of these tasks is to show that our approach works in this setting once combining the critic with Monte-Carlo rollouts. We use the same tasks as in the privileged information setting, without mass randomization. Note that our algorithm does not use a recurrent critic in contrast to \gls{bptt}, which uses a separate \gls{rnn} for the actor and critic as in \cite{ni2022recurrent}. The overall results are shown on the right of Figure \ref{fig:overall_results}. As can be seen, our method reliably achieves superior performance compared to \gls{bptt} and the window approach. While the training time of \gls{ppo}-\gls{bptt} with a truncation length of 5 is very similar to our approach, there is a notable difference when using longer truncation lengths. We found that both \gls{bptt} and the window approach perform worse with a longer history. When compared to specialized methods, such as SLAC \cite{slac}, we found that our method performs better while being much faster to train. One reason for the long training of SLAC time is that it also needs to pre-train the latent variable model before starting to learn the policy and the critic. In Appendix~\ref{app:ppo_res}, we provide the return plots and further discussion. 

\textbf{Imitation Learning Tasks.} The main results of our paper are presented in Figure~\ref{fig:imitation_learning_time}. Here, we use two novel locomotion tasks introduced by the LocoMujoco benchmark \citet{alhafez2023_loco}. The first task is an Atlas locomotion task, where a carry-weight is randomly sampled, but the weight is hidden from the policy. The second task is a Humanoid locomotion task under dynamics randomization. The goal is to imitate a certain kinematic trajectory -- either walking or running -- without observing the type of the humanoid. To generate the different humanoids, we randomly sample a scaling factor. Then the links are scaled linearly, the masses are scaled cubically, the inertias are scaled quintically, and the actuator torques are scaled cubically w.r.t. the scaling factor. For both tasks, the policies only observe the positions and the velocities of the joints. Forces are not observed. As can be seen at the bottom of Figure~\ref{fig:imitation_learning_time}, our approach can be easily extended to complex \gls{il} tasks using \gls{gail} and \gls{lsiq}. We observe that in the \gls{il} setting our approach is able to outperform the \gls{bptt} baselines even in terms of samples as shown in Figure~\ref{fig:imitation_learning_gail_steps} and Figure~\ref{fig:imitation_learning_lsiq_steps} in Appendix~\ref{app:imitaion_learning}, where we also further discuss the results. 

\textbf{Memory Task.} To show that our approach can encode longer history, we present two memory tasks in Appendix~\ref{app:mem_task}. The first task is to move a point mass to a randomly sampled goal from a randomly sampled initial state. While the goal is shown to the policy at the beginning of the trajectory, it is hidden when the point mass moves away from the initial state. Similarly, in the second task, the positions of two doors in a maze are shown to the policy when close to the initial state and are hidden once going further away. Figure~\ref{fig:point_mass} and \ref{fig:point_mass_door} demonstrate that our approach outperforms \gls{bptt} on the first task while being slightly weaker on the second task. We expect that the reason for the drop in performance is caused by the additional variance of our method in combination with the additional absorbing states in the environment. The additional variance in our policy leads to increased exploration, which increases the chances of touching the wall and reaching an absorbing state. Nonetheless, our approach has a significantly shorter training time, analogously to the results shown in Figure \ref{fig:overall_results}, for both tasks.

\begin{figure}[t]
\centering
\includegraphics[width=\textwidth]{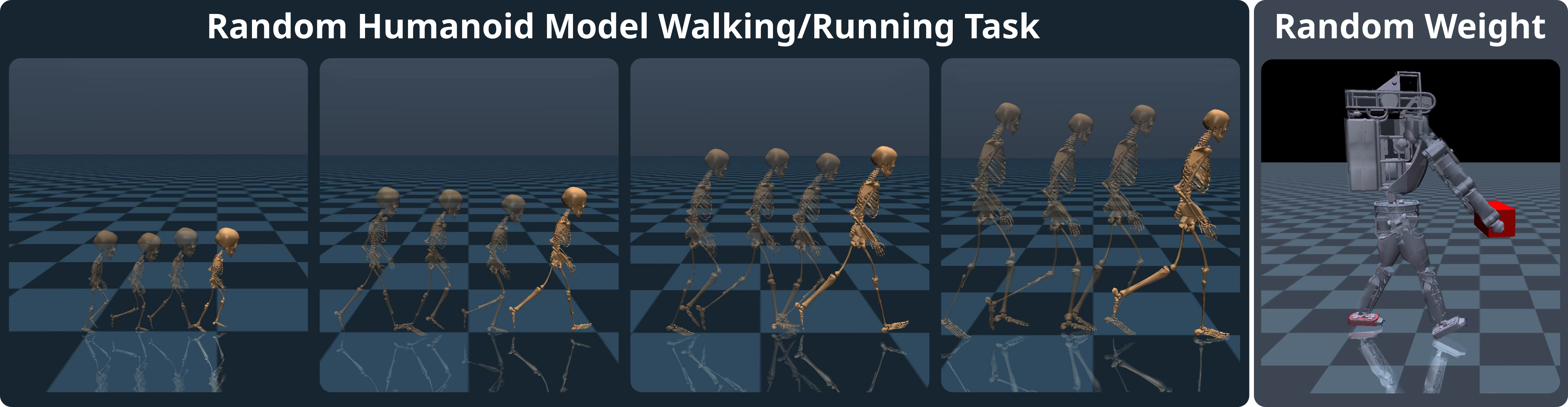}
\vspace{-20pt}
\begin{multicols}{3}
\includegraphics[scale=0.275]{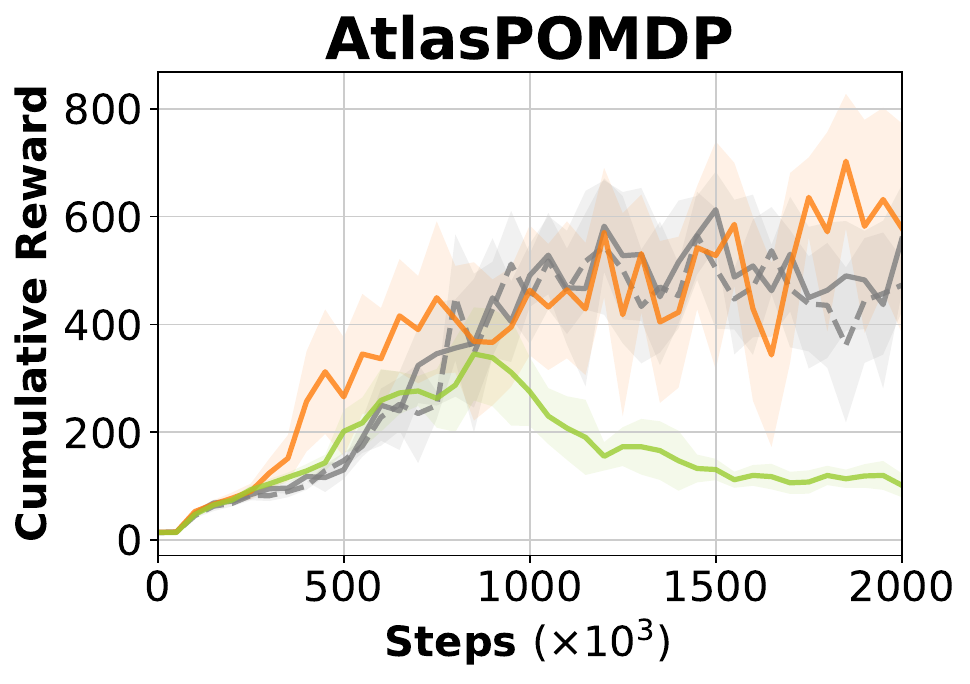}\columnbreak
\includegraphics[scale=0.275]{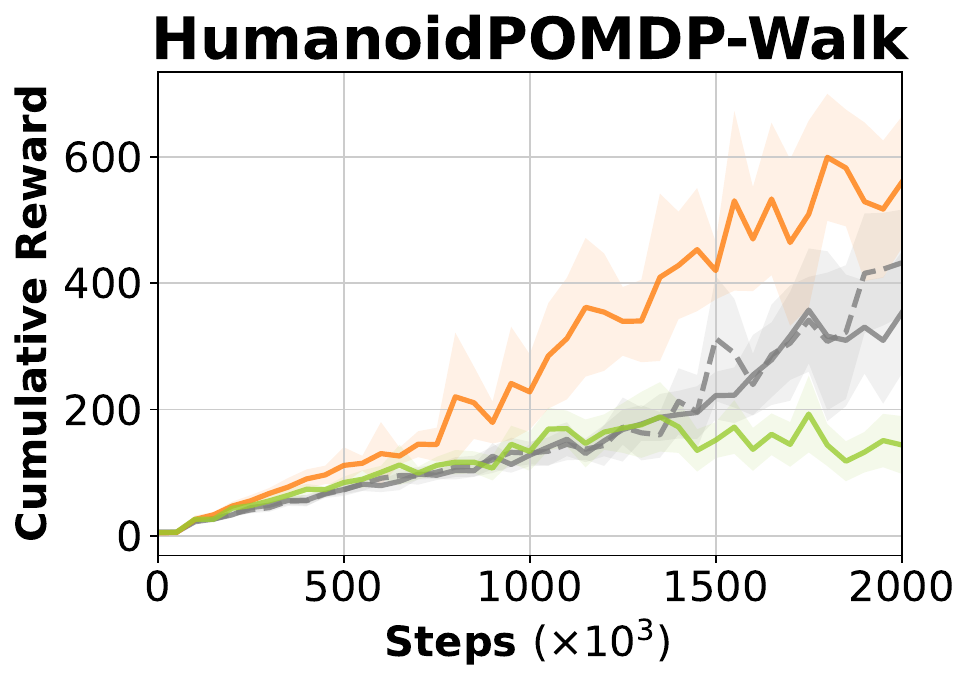}\columnbreak
\includegraphics[scale=0.275]{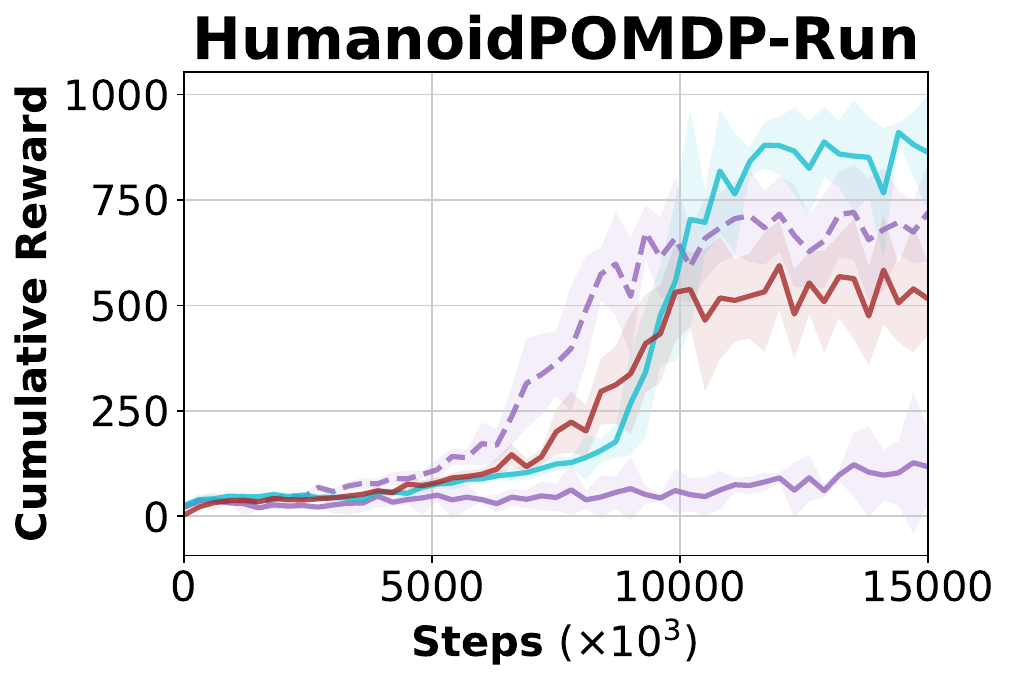}
\end{multicols}
\vspace{-15pt}
\begin{center}
    \includegraphics[scale=0.35]{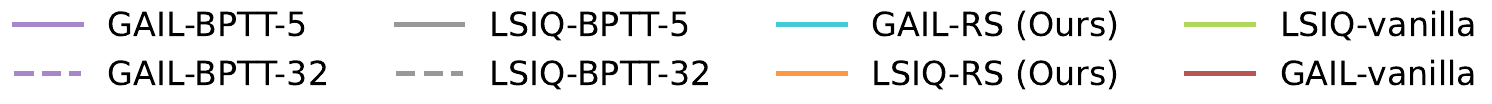}
\end{center}
\vspace{-10pt}
\centering
\caption{\textbf{Up-Left}: \gls{pomdp} Humanoid walking/running task at 4.5~km/h respectively 9~km/h. The different humanoids should resemble -- from left to right -- an adult, a teenager ($\sim$12 years), a child ($\sim$5 years), and a toddler ($\sim$1-2 years). The target speed is scaled for the smaller humanoid according to size. The type of humanoid is hidden to the policy. A single policy is able to learn all gaits. \mbox{\textbf{Up-Right}: Random} carry-weight task using the Atlas humanoid. The weight is randomly sampled between 0.1~kg and 10~kg at the beginning of the episode and is hidden to the policy. \mbox{\textbf{Bottom}:} Results comparing different versions of \gls{gail} and \gls{lsiq}.} 
\label{fig:imitation_learning_time}
\end{figure}

\section{Conclusions}
This work introduced \gls{s2pg}, an alternative approach for estimating the policy gradient of stateful policies without using \gls{bptt}, exploiting stochastic transition kernels. \gls{s2pg} is easy to implement and computationally efficient. We provide a complete foundation theory of this novel estimator, allowing its implementation in state-of-the-art deep \gls{rl} methods, and conduct a theoretical analysis on the variance of \gls{s2pg} and \gls{bptt}.
While this method still cannot replace \gls{bptt} in every setting, in the most challenging scenarios and in the \gls{il} settings, \gls{s2pg} can considerably improve the performance and learning speed.
Unfortunately, while \gls{s2pg} can replace  \gls{bptt} in the computation of the policy gradient, it is still not able to properly learn a value function and an internal policy state at the same time in the partially observable setting. While we show that our approach works well when the critic is estimated with Monte-Carlo rollouts, e.g., the \gls{ppo} algorithm, for methods that require bootstrapping we limit ourselves to the setting where the critic has privileged information. We plan to investigate possible solutions to this issue in future works. Furthermore, we will investigate how our approach scales using massively parallel environments. Finally, our method could be used in combination with complex policy structures such as Neural ODEs.

\bibliography{bibliography}

\begin{thebibliography}{47}
\providecommand{\natexlab}[1]{#1}
\providecommand{\url}[1]{\texttt{#1}}
\expandafter\ifx\csname urlstyle\endcsname\relax
  \providecommand{\doi}[1]{doi: #1}\else
  \providecommand{\doi}{doi: \begingroup \urlstyle{rm}\Url}\fi

\bibitem[Al-Hafez \& Steil(2021)Al-Hafez and Steil]{alhafez2021}
Firas Al-Hafez and Jochen Steil.
\newblock Redundancy resolution as action bias in policy search for robotic manipulation.
\newblock In  \citet{corl2021}.

\bibitem[Al-Hafez et~al.(2023{\natexlab{a}})Al-Hafez, Tateo, Arenz, Zhao, and Peters]{alhafez2023}
Firas Al-Hafez, Davide Tateo, Oleg Arenz, Guoping Zhao, and Jan Peters.
\newblock {LS-IQ}: Implicit reward regularization for inverse reinforcement learning.
\newblock In \emph{Proceeding of the International Conference on Learning Representations}, Kigali, Rwanda, May 2023{\natexlab{a}}.

\bibitem[Al-Hafez et~al.(2023{\natexlab{b}})Al-Hafez, Zhao, Peters, and Tateo]{alhafez2023_loco}
Firas Al-Hafez, Guoping Zhao, Jan Peters, and Davide Tateo.
\newblock {LocoMuJoCo}: A comprehensive imitation learning benchmark for locomotion.
\newblock In \emph{6th Robot Learning Workshop, NeurIPS}, New Orleans, Louisiana, United States, December 2023{\natexlab{b}}.

\bibitem[Bakker(2001)]{bakker2001reinforcement}
Bram Bakker.
\newblock Reinforcement learning with long short-term memory.
\newblock In \emph{Proceeding of the Fifteenth Conference on Neural Information Processing Systems}, Vancouver, British Columbia, Canada, December 2001.

\bibitem[Baxter \& Bartlett(2001)Baxter and Bartlett]{baxter2001}
Jonathan Baxter and Peter~L. Bartlett.
\newblock Infinite-horizon policy-gradient estimation.
\newblock \emph{Jounral of Artificial Intelligence Research}, 15, 2001.

\bibitem[Bellegarda \& Ijspeert(2022{\natexlab{a}})Bellegarda and Ijspeert]{bellegarda2022cpg}
Guillaume Bellegarda and Auke Ijspeert.
\newblock Cpg-rl: Learning central pattern generators for quadruped locomotion.
\newblock \emph{IEEE Robotics and Automation Letters}, 7\penalty0 (4):\penalty0 12547--12554, 2022{\natexlab{a}}.

\bibitem[Bellegarda \& Ijspeert(2022{\natexlab{b}})Bellegarda and Ijspeert]{bellegarda2022visual}
Guillaume Bellegarda and Auke Ijspeert.
\newblock Visual cpg-rl: Learning central pattern generators for visually-guided quadruped navigation.
\newblock \emph{arXiv preprint arXiv:2212.14400}, 2022{\natexlab{b}}.

\bibitem[Campanaro et~al.(2021)Campanaro, Gangapurwala, De~Martini, Merkt, and Havoutis]{campanaro2021cpg}
Luigi Campanaro, Siddhant Gangapurwala, Daniele De~Martini, Wolfgang Merkt, and Ioannis Havoutis.
\newblock Cpg-actor: Reinforcement learning for central pattern generators.
\newblock In \emph{Towards Autonomous Robotic Systems: 22nd Annual Conference, TAROS 2021, Lincoln, UK, September 8--10, 2021, Proceedings}, pp.\  25--35. Springer, 2021.

\bibitem[Cho et~al.(2014)Cho, van Merri{\"e}nboer, Bahdanau, and Bengio]{cho2014}
Kyunghyun Cho, Bart van Merri{\"e}nboer, Dzmitry Bahdanau, and Yoshua Bengio.
\newblock On the properties of neural machine translation: Encoder{--}decoder approaches.
\newblock In \emph{Proceedings of {SSST}-8, Eighth Workshop on Syntax, Semantics and Structure in Statistical Translatio}, pp.\  103--11, Doha, Qatar, October 2014.

\bibitem[Cho1 et~al.(2019)Cho1, Manzoor, and Choi]{cho2019}
Younggil Cho1, Sajjad Manzoor, and Youngjin Choi.
\newblock Adaptation to environmental change using reinforcement learning for robotic salamander.
\newblock \emph{Intelligent Service Robotics}, 12:\penalty0 209 -- 2018, June 2019.

\bibitem[CoRL 2021()]{corl2021}
CoRL 2021.
\newblock \emph{Proceeding of the Conference on Robot Learning}, London, United Kingdom, November 2021.

\bibitem[D’Eramo et~al.(2021)D’Eramo, Tateo, Bonarini, Restelli, and Peters]{mushroomrl}
Carlo D’Eramo, Davide Tateo, Andrea Bonarini, Marcello Restelli, and Jan Peters.
\newblock Mushroomrl: Simplifying reinforcement learning research.
\newblock \emph{Journal of Machine Learning Research}, 22:\penalty0 1--5, 2021.

\bibitem[Efroni et~al.(2022)Efroni, Jin, Krishnamurthy, and Miryoosefi]{efroni2022provable}
Yonathan Efroni, Chi Jin, Akshay Krishnamurthy, and Sobhan Miryoosefi.
\newblock Provable reinforcement learning with a short-term memory.
\newblock In  \citet{icml2022}.

\bibitem[Espeholt et~al.(2018)Espeholt, Soyer, Munos, Simonyan, Mnih, Ward, Doron, Firoiu, Harley, Dunning, et~al.]{espeholt2018impala}
Lasse Espeholt, Hubert Soyer, Remi Munos, Karen Simonyan, Vlad Mnih, Tom Ward, Yotam Doron, Vlad Firoiu, Tim Harley, Iain Dunning, et~al.
\newblock Impala: Scalable distributed deep-rl with importance weighted actor-learner architectures.
\newblock In  \citet{icml2018}.

\bibitem[Fujimoto et~al.(2018)Fujimoto, van Hoof, and Meger]{fujimoto2018}
Scott Fujimoto, Herke van Hoof, and David Meger.
\newblock Addressing function approximation error in actor-critic methods.
\newblock In  \citet{icml2018}.

\bibitem[Haarnoja et~al.(2018)Haarnoja, Zhou, Abbeel, and Levine]{sac}
Tuomas Haarnoja, Aurick Zhou, Pieter Abbeel, and Sergey Levine.
\newblock Soft actor-critic: Off-policy maximum entropy deep reinforcement learning with a stochastic actor.
\newblock In  \citet{icml2018}.

\bibitem[Heess et~al.(2015)Heess, Hunt, Lillicrap, and Silver]{heess2015memory}
Nicolas Heess, Jonathan~J Hunt, Timothy~P Lillicrap, and David Silver.
\newblock Memory-based control with recurrent neural networks.
\newblock \emph{arXiv preprint arXiv:1512.04455}, 2015.

\bibitem[Ho \& Ermon(2016)Ho and Ermon]{Ho2016}
Jonathan Ho and Stefano Ermon.
\newblock Generative adversarial imitation learning.
\newblock In \emph{Proceeding of the Thirtieth Conference on Neural Information Processing Systems}, Barcelona, Spain, December 2016.

\bibitem[Hochreiter \& Schmidhuber(1997)Hochreiter and Schmidhuber]{hochreiter1997}
Sepp Hochreiter and Jürgen Schmidhuber.
\newblock Long short-term memory.
\newblock \emph{Neural Computation}, 9\penalty0 (8):\penalty0 1735 -- 1780, 1997.

\bibitem[ICML 2018()]{icml2018}
ICML 2018.
\newblock \emph{Proceeding of the International Conference on Machine Learning}, Stockholm, Sweden, July 2018.

\bibitem[ICML 2022()]{icml2022}
ICML 2022.
\newblock \emph{Proceeding of the International Conference on Machine Learning}, Baltimore, Maryland, July 2022.

\bibitem[Igl et~al.(2018)Igl, Zintgraf, Le, Wood, and Whiteson]{dvrl}
Maximilian Igl, Luisa Zintgraf, Tuan~Anh Le, Frank Wood, and Shimon Whiteson.
\newblock Deep variational reinforcement learning for pomdps.
\newblock In  \citet{icml2018}.

\bibitem[Ijspeert(2008)]{ijspeert2008central}
Auke~Jan Ijspeert.
\newblock Central pattern generators for locomotion control in animals and robots: a review.
\newblock \emph{Neural networks}, 21\penalty0 (4):\penalty0 642--653, 2008.

\bibitem[Ijspeert et~al.(2007)Ijspeert, Crespi, Ryczko, and Cabelguen]{ijspeert2007swimming}
Auke~Jan Ijspeert, Alessandro Crespi, Dimitri Ryczko, and Jean-Marie Cabelguen.
\newblock From swimming to walking with a salamander robot driven by a spinal cord model.
\newblock \emph{science}, 315\penalty0 (5817):\penalty0 1416--1420, 2007.

\bibitem[Kakade \& Langford(2002)Kakade and Langford]{kakade2002approximately}
Sham Kakade and John Langford.
\newblock Approximately optimal approximate reinforcement learning.
\newblock In \emph{Proceeding of the International Conference on Machine Learning}, Sydney, Australia, July 2002.

\bibitem[Lee et~al.(2020{\natexlab{a}})Lee, Nagabandi, Abbeel, and Levine]{slac}
Alex~X Lee, Anusha Nagabandi, Pieter Abbeel, and Sergey Levine.
\newblock Stochastic latent actor-critic: Deep reinforcement learning with a latent variable model.
\newblock In \emph{Proceeding of the Thirty-fourth Conference on Neural Information Processing Systems}, Virtual, December 2020{\natexlab{a}}.

\bibitem[Lee et~al.(2023)Lee, Xie, Pacchiano, Chandak, Finn, Nachum, and Brunskill]{lee2023supervised}
Jonathan~N Lee, Annie Xie, Aldo Pacchiano, Yash Chandak, Chelsea Finn, Ofir Nachum, and Emma Brunskill.
\newblock Supervised pretraining can learn in-context reinforcement learning.
\newblock \emph{arXiv preprint arXiv:2306.14892}, 2023.

\bibitem[Lee et~al.(2020{\natexlab{b}})Lee, Hwangbo, Wellhausen, Koltun, and Hutter]{lee2020learning}
Joonho Lee, Jemin Hwangbo, Lorenz Wellhausen, Vladlen Koltun, and Marco Hutter.
\newblock Learning quadrupedal locomotion over challenging terrain.
\newblock \emph{Science robotics}, 5\penalty0 (47):\penalty0 eabc5986, 2020{\natexlab{b}}.

\bibitem[Lillicrap et~al.(2016)Lillicrap, Hunt, Pritzel, Heess, Erez, Tassa, Silver, and Wierstra]{ddpg}
Timothy~P. Lillicrap, Jonathan~J. Hunt, Alexander Pritzel, Nicolas Heess, Tom Erez, Yuval Tassa, David Silver, and Daan Wierstra.
\newblock Continuous control with deep reinforcement learning.
\newblock In \emph{Proceeding of the International Conference on Learning Representations}, San Juan, Puerto Rico, May 2016.

\bibitem[Liu et~al.(2021)Liu, Tateo, Bou-Ammar, and Peters]{atacom}
Puze Liu, Davide Tateo, Haitham Bou-Ammar, and Jan Peters.
\newblock Robot reinforcement learning on the constraint manifold.
\newblock In  \citet{corl2021}.

\bibitem[Meng et~al.(2021)Meng, Gorbet, and Kuli{\'c}]{meng2021memory}
Lingheng Meng, Rob Gorbet, and Dana Kuli{\'c}.
\newblock Memory-based deep reinforcement learning for pomdps.
\newblock In \emph{Proceedings of the 2021 IEEE/RSJ International Conference on Intelligent Robots and Systems}, pp.\  5619--5626, Virtual, September 2021.

\bibitem[Miki et~al.(2022)Miki, Lee, Hwangbo, Wellhausen, Koltun, and Hutter]{miki2022learning}
Takahiro Miki, Joonho Lee, Jemin Hwangbo, Lorenz Wellhausen, Vladlen Koltun, and Marco Hutter.
\newblock Learning robust perceptive locomotion for quadrupedal robots in the wild.
\newblock \emph{Science Robotics}, 7\penalty0 (62):\penalty0 eabk2822, 2022.

\bibitem[Ni et~al.(2022)Ni, Eysenbach, and Salakhutdinov]{ni2022recurrent}
Tianwei Ni, Benjamin Eysenbach, and Ruslan Salakhutdinov.
\newblock Recurrent model-free rl can be a strong baseline for many pomdps.
\newblock In  \citet{icml2022}, pp.\  16691--16723.

\bibitem[Papini et~al.(2022)Papini, Pirotta, and Restelli]{papini2022smoothing}
Matteo Papini, Matteo Pirotta, and Marcello Restelli.
\newblock Smoothing policies and safe policy gradients.
\newblock \emph{Machine Learning}, 111\penalty0 (11):\penalty0 4081--4137, 2022.

\bibitem[Peng et~al.(2018)Peng, Andrychowicz, Zaremba, and Abbeel]{peng20182}
Xue~Bin Peng, Marcin Andrychowicz, Wojciech Zaremba, and Pieter Abbeel.
\newblock Sim-to-real transfer of robotic control with dynamics randomization.
\newblock In \emph{Proceeding of the IEEE International Conference on Robotics and Automation (ICRA)}, Brisbane, Australia, May 2018.

\bibitem[Pinto et~al.(2018)Pinto, Andrychowicz, Welinder, Zaremba, and Abbeel]{pinto2018asymmetric}
Lerrel Pinto, Marcin Andrychowicz, Peter Welinder, Wojciech Zaremba, and Pieter Abbeel.
\newblock Asymmetric actor critic for image-based robot learning.
\newblock In \emph{14th Robotics: Science and Systems, RSS 2018}. MIT Press Journals, 2018.

\bibitem[Radosavovic et~al.(2023)Radosavovic, Xiao, Zhang, Darrell, Malik, and Sreenath]{radosavovic2023learning}
Ilija Radosavovic, Tete Xiao, Bike Zhang, Trevor Darrell, Jitendra Malik, and Koushil Sreenath.
\newblock Learning humanoid locomotion with transformers.
\newblock \emph{arXiv preprint arXiv:2303.03381}, 2023.

\bibitem[Schulman et~al.(2015{\natexlab{a}})Schulman, Heess, Weber, and Abbeel]{schulman2015gradient}
John Schulman, Nicolas Heess, Theophane Weber, and Pieter Abbeel.
\newblock Gradient estimation using stochastic computation graphs.
\newblock In \emph{Proceeding of the Twenty-ninth Conference on Neural Information Processing Systems}, Montreal, Canada, December 2015{\natexlab{a}}.

\bibitem[Schulman et~al.(2015{\natexlab{b}})Schulman, Levine, Abbeel, Jordan, and Moritz]{schulman2015}
John Schulman, Sergey Levine, Pieter Abbeel, Michael Jordan, and Philipp Moritz.
\newblock Trust region policy optimization.
\newblock In \emph{Proceeding of the International Conference on Machine Learning}, Lille, France, July 2015{\natexlab{b}}.

\bibitem[Schulman et~al.(2017)Schulman, Wolski, Dhariwal, Radford, and Klimov]{ppo}
John Schulman, Filip Wolski, Prafulla Dhariwal, Alec Radford, and Oleg Klimov.
\newblock Proximal policy optimization algorithms.
\newblock \emph{arXiv}, August 2017.

\bibitem[Shi et~al.(2022)Shi, Zhou, Zeng, Wang, Dong, Li, Wang, Tian, , and Meng]{shi2022}
Haojie Shi, Bo~Zhou, Hongsheng Zeng, Fan Wang, Yueqiang Dong, Jiangyong Li, Kang Wang, Hao Tian, , and Max Q.-H. Meng.
\newblock Reinforcement learning with evolutionary trajectory generator: A general approach for quadrupedal locomotiong.
\newblock \emph{IEEE Robotics and Automation Letters}, 7\penalty0 (2):\penalty0 3085 -- 3092, April 2022.

\bibitem[Silver et~al.(2014)Silver, Lever, Heess, Degris, Wierstra, and Riedmiller]{dpg}
David Silver, Guy Lever, Nicolas Heess, Thomas Degris, Daan Wierstra, and Martin Riedmiller.
\newblock Deterministic policy gradient algorithms.
\newblock In \emph{Proceeding of the International Conference on Machine Learning}, Beijing, China, June 2014.

\bibitem[Sutton et~al.(1999)Sutton, McAllester, Singh, and Mansour]{sutton1999policy}
Richard~S Sutton, David McAllester, Satinder Singh, and Yishay Mansour.
\newblock Policy gradient methods for reinforcement learning with function approximation.
\newblock In \emph{Proceeding of the Thirteen Conference on Neural Information Processing Systems}, volume~12, Denver, Colorado, USA, November 1999.

\bibitem[Wierstra et~al.(2010)Wierstra, F{\"o}rster, Peters, and Schmidhuber]{wierstra2010recurrent}
Daan Wierstra, Alexander F{\"o}rster, Jan Peters, and J{\"u}rgen Schmidhuber.
\newblock Recurrent policy gradients.
\newblock \emph{Logic Journal of the IGPL}, 18\penalty0 (5):\penalty0 620--634, 2010.

\bibitem[Yang \& Nguyen(2021)Yang and Nguyen]{yang2021recurrent}
Zhihan Yang and Hai Nguyen.
\newblock Recurrent off-policy baselines for memory-based continuous control.
\newblock In \emph{Deep Reinforcement Learning Workshop, NeurIPS}, Sydney, Australia, December 2021.

\bibitem[Zhang et~al.(2016)Zhang, McCarthy, Finn, Levine, and Abbeel]{zhang2016learning}
Marvin Zhang, Zoe McCarthy, Chelsea Finn, Sergey Levine, and Pieter Abbeel.
\newblock Learning deep neural network policies with continuous memory states.
\newblock In \emph{Proceeding of the IEEE International Conference on Robotics and Automation (ICRA)}, pp.\  520--527, Stockholm, Sweden, May 2016.

\bibitem[Zhao et~al.(2011)Zhao, Hachiya, Niu, and Sugiyama]{zhao2011analysis}
Tingting Zhao, Hirotaka Hachiya, Gang Niu, and Masashi Sugiyama.
\newblock Analysis and improvement of policy gradient estimation.
\newblock In \emph{Proceeding of the Twenty-fifth Conference on Neural Information Processing Systems}, Granada, Spain, December 2011.

\end{thebibliography}
\bibliographystyle{iclr2024_conference}

\newpage
\appendix
\section{Proofs of Theorems}

Within this section, we refer to the following regularity conditions on the \gls{mdp}:

\begin{conditions}\label{cond:reg1}
 $\iota(s, z)$, $P(s'|s, a)$, $\nabla_a P(s'|s, a)$, $\pi_\vtheta(a, z' | s, z)$, $\nabla_\vtheta \pi_\theta(a, z' | s, z)$, $\mu_\vtheta(s,z)$,  $\nabla_\vtheta\mu_\vtheta(s,z)$, $r(s, a)$, $\nabla_a r(s, a)$, $\iota(s)$ are continuous in all parameters and variables $s, a, z, \vtheta$.
\end{conditions}

\begin{conditions}\label{cond:reg2}
  There exists a $b$ and $L$ such that $\sup_s \iota(s) < b$, $\sup_{s',a,s} p(s'|s, a) < b$, $\sup_{a,s} r(s, a) < b$, $\sup_{s',a,s} \lVert\nabla_a p(s'|s, a)\rVert < L$, and $\sup_{a,s} \lVert\nabla_a r(s, a)\rVert < L$.
\end{conditions}

\subsection{Stateful Policy Gradient with the Score Function}
\label{app:proof_pg}

\begin{proof}
Let
{\small
\begin{equation*}
    J(\tauh) = \sum_{t=0}^{T-1} \gamma^t r(s_t,a_t),
\end{equation*}}
be the discounted return of a trajectory $\tauh=\langle s_0, z_0, a_1, \dots s_{T-1}, z_{T-1}, a_{T-1}, s_T, z_T \rangle$.
Differently from the standard policy gradient, we define the trajectory as a path of both the environment and the stateful policy state.
The probability of a given trajectory can be written as
{\small
\begin{equation*}
    p(\tauh|\vtheta) = \iota(s_0,z_0)\prod_{t=0}^{T-1} P(s_{t+1}|s_t, a_t)\pi_\vtheta(a_t,z_{t+1}|s_t,z_t).
\end{equation*}}

The continuity of $\iota(s_0, z_0)$, $P(s_{t+1}|s_t, a_t)$, $\pi_\vtheta(a_t, z_{t+1} | s_t, z_t)$, $\nabla_\vtheta \pi_\vtheta(a_t, z_{t+1} | s_t, z_t)$ implies that $p(\tauh | \vtheta)$ and $\nabla_\vtheta p(\tauh | \vtheta)$ are continue as well. Hence, we can apply the Leibniz rule to exchange the order of integration and derivative, allowing us to follow the same steps as the classical policy gradient
{\small
\begin{equation*}
\nabla_\vtheta \mathcal{J}(\pi_\vtheta) = \nabla_\vtheta\int p(\tauh|\vtheta)J(\tauh)d\tauh = \int \nabla_\vtheta p(\tauh|\vtheta) J(\tauh)d\tauh.
\end{equation*}}

Applying the likelihood ratio trick, we obtain
{\small
\begin{equation}
    \nabla_\vtheta \mathcal{J}(\pi_\vtheta) = \int p(\tauh|\vtheta) \nabla_\vtheta \log p(\tauh|\vtheta) J(\tauh)d\tauh = \expect{\tauh}{\nabla_\vtheta \log p(\tauh|\vtheta)J(\tauh)} \, . \label{eq:spg1}
\end{equation}}

To obtain the REINFORCE-style estimator, we note that
{\small
\begin{equation*}
    \log p(\tau|\vtheta) = \log \iota(s_0,z_0) + \sum_{t=0}^{T-1} \log P(s_{t+1}|s_t, a_t) + \sum_{t=0}^{T-1} \log\pi_{\vtheta}(a_t,z_{t+1}|s_t,z_t) \, . 
\end{equation*}}

Therefore, for the vanilla policy gradient, the term $\nabla_\vtheta p(\tau|\vtheta)$ only depends on the policy and not the state transitions. We can then write the stateful policy gradient as
{\small
\begin{equation}
    \nabla_\vtheta\mathcal{J}(\pi_\vtheta) = \expect{\tau}{\sum_{t=0}^{T-1} \nabla_\vtheta\log\pi_\vtheta(a_t,z_{t+1}|s_t,z_t) J(\tau)} \, , \label{eq:spg2}
\end{equation}}
concluding the proof.
\end{proof}

\subsection{Stateful Policy Gradient in Partially Observable Environments}
\label{app:proof_pg_po}

\begin{proof}
The derivation follows exactly the one for the fully observable scenario, with the difference that we need to deal with the observations instead of states. The probability of observing a trajectory~$\tau$ in a \gls{pomdp} is
{\small
\begin{equation*}
    p(\tau|\vtheta) = \iota(s_0,z_0)O(o_0|s_0)\prod_{t=0}^{T-1} O(o_{t+1}|s_{t+1})P(s_{t+1}|s_t, a_t)\pi_\vtheta(a_t,z_{t+1}|o_t,z_t).
\end{equation*}}
Computing the term $\nabla_\vtheta \log p(\tau|\vtheta)$ and replacing it into \eqref{eq:spg1} we obtain
{\small
\begin{equation*}
    \nabla_\vtheta\mathcal{J}(\pi_\vtheta) = \expect{\tau}{\sum_{t=0}^{T-1} \nabla_\vtheta\log\pi_\vtheta(a_t,z_{t+1}|o_t,z_t) J(\tau)}\, , 
\end{equation*}}
using the same simplification as done in  \eqref{eq:spg2}, concluding the proof.
\end{proof}

\subsection{Stateful Policy Gradient Theorem}
\label{app:proof_pgt}
\begin{proof}
We can write the objective function as follows

{\small
\begin{equation*}
   \mathcal{J}_{\vtheta} = \expect{(s_0,z_0)\sim\iota}{V_{\pi_{\vtheta}}(s_0,z_0)} = \expect{(s_0,z_0)\sim\iota,\\ (a_0,z_1)\sim\pi_{\vtheta}}{Q^{\pi_{\vtheta}}(s_0,z_0,a_0,z_1)}.
\end{equation*}}
We can now compute the gradient as

{\small
\begin{align*}
   \nabla_\vtheta\mathcal{J}_{\vtheta} &= \nabla_\vtheta\expect{(s_0,z_0)\sim\iota,\\ (a_0,z_1)\sim\pi_{\vtheta}}{Q^{\pi_{\vtheta}}(s_0,z_0,a_0,z_1)}  \\    &=\nabla_\vtheta\int_\mathcal{S}\int_\mathcal{Z}\int_\mathcal{A}\int_\mathcal{Z} \iota(s_0,z_0)\pi_{\vtheta}(a_0,z_1|s_0,z_0)Q^{\pi_{\vtheta}}(s_0,z_0,a_0,z_1)da_0 dz_1 ds_0 dz_0.
\end{align*}}

The regularity conditions \ref{cond:reg1} imply that $Q^{\pi_\theta}(s, z', a, z)$ and $\nabla_\theta Q^{\pi_\theta}(s, z', a, z)$ are continuous functions of $\theta$, $s$, $a$, and $z$. Further, the compactness of $\mathcal{S}$, $\mathcal{A}$ and $\mathcal{Z}$ implies that for any $\theta$, $\lVert \nabla_\theta Q^{\pi_\theta}(s, z', a, z)\rVert$ and $\lVert \nabla_\theta \pi_\theta (a, z' | s, z)\rVert$ are bounded functions of $s$, $a$ and $z$. These conditions are required to exchange integration and derivatives, as well as the order of integration throughout the this proof. 

We can push the derivation inside the integral as $Q^{\pi_{\vtheta}}$ is a bounded function and the derivative w.r.t. $\vtheta$ exists everywhere
{\small
\begin{align*}
    \nabla_\vtheta\mathcal{J}_{\vtheta} =& \int_\mathcal{S}\int_\mathcal{Z}\int_\mathcal{A}\int_\mathcal{Z} \nabla_\vtheta\left[\iota(s_0,z_0)\pi_{\vtheta}(a_0,z_1|s_0,z_0)Q^{\pi_{\vtheta}}(s_0,z_0,a_0,a_1)\right]da_0 dz_1 ds_0 dz_0 \\
    =& \int_\mathcal{S}\int_\mathcal{Z}\int_\mathcal{A}\int_\mathcal{Z} \iota(s_0,z_0)\nabla_\vtheta\pi_{\vtheta}(a_0,z_1|s_0,z_0)Q^{\pi_{\vtheta}}(s_0,z_0,a_0,z_1) \\
    &+\iota(s_0,z_0)\pi_{\vtheta}(a_0,z_1|s_0,z_0)\nabla_\vtheta Q^{\pi_{\vtheta}}(s_0,z_0,a_0,z_1)da_0 dz_1 ds_0 dz_0.
\end{align*}}

We now focus on the gradient of the Q-function w.r.t. the policy parameters. 
Using the following relationship:
{\small
\begin{equation*}
    Q^{\pi_{\vtheta}}(s_t,z_t,a_t,z_{t+1}) = r(s_t, a_t) + \gamma \int_\mathcal{S} P(s_{t+1}|s_t,a_t) V^{\pi_{\vtheta}}(s_{t+1},z_{t+1})ds_{t+1} dz_{t+1}.
\end{equation*}}

We can write:
{\small
\begin{align*}
    \nabla_\vtheta Q^{\pi_{\vtheta}}(s_t,a_t) & = \nabla_\vtheta\left[r(s_t, a_t) + \gamma \int_\mathcal{S} P(s_{t+1}|s_t,a_t) V^{\pi_{\vtheta}}(s_{t+1},z_{t+1})ds_{t+1} dz_{t+1} \right]\\
    & = \gamma \int_\mathcal{S} P(s_{t+1}|s_t,a_t) \nabla_\vtheta V^{\pi_{\vtheta}}(s_{t+1},z_{t+1})ds_{t+1} dz_{t+1}. 
\end{align*}}
Similarly, we can compute the gradient of the value function as:
{\small
\begin{align*}
\nabla_\vtheta V^{\pi_{\vtheta}}(s_t,z_t) =& \int_\mathcal{A}\int_\mathcal{Z} \left(\vphantom{\int} \nabla_\vtheta\pi_{\vtheta}(a_tz_{t+1}|s_t,z_t)Q^{\pi_{\vtheta}}(s_t,z_t,a_t,z_{t+1})\right.\\ 
&\left.+\gamma\int_\mathcal{S}\pi_{\vtheta}(a_tz_{t+1}|s_t,z_t) P(s_{t+1}|s_t,a_t) \nabla_\vtheta V^{\pi_{\vtheta}}(s_{t+1},z_{t+1})\right)ds_{t+1}da_t dz_{t+1}.
\end{align*}}
\vspace{-0.5cm}

We now focus again on $\nabla_\vtheta\mathcal{J}_{\vtheta}$. Using the linearity of integral we write
{\small
\begin{align*}
\nabla_\vtheta\mathcal{J}_{\vtheta} = & \int_\mathcal{S}\int_\mathcal{Z}\int_\mathcal{A}\int_\mathcal{Z} \iota(s_0,z_0)\nabla_\vtheta\pi_{\vtheta}(a_0,z_1|s_0,z_0)Q^{\pi_{\vtheta}}(s_0,z_0,a_0,z_1)da_0 dz_1 ds_0 dz_0\\
&+\int_\mathcal{S}\int_\mathcal{Z}\int_\mathcal{A}\int_\mathcal{Z}\iota(s_0,z_0)\pi_{\vtheta}(a_0,z_1|s_0,z_0)\nabla_\vtheta Q^{\pi_{\vtheta}}(s_0,z_0,a_0,z_1)da_0 dz_1 ds_0 dz_0.
\end{align*}}

Now, using the expressions for $\nabla_\vtheta V^{\pi_{\vtheta}}$ and $\nabla_\vtheta Q^{\pi_{\vtheta}}$, we can expand the second term of the previous sum for one step (limited to one step to highlight the structure)
{\small
\begin{align*}
    &\int_\mathcal{S}\int_\mathcal{Z}\int_\mathcal{A}\int_\mathcal{Z}\iota(s_0,z_0)\pi_{\vtheta}(a_0,z_0|s_0,z_0)\nabla_\vtheta Q^{\pi_{\vtheta}}(s_0,z_0,a_0,z_1)da_0 dz_1 ds_0 dz_0\\
    &\quad=\gamma \int_\mathcal{S}\int_\mathcal{Z}\int_\mathcal{A}\int_\mathcal{Z} \int_\mathcal{S} \iota(s_0,z_0)\pi_{\vtheta}(a_0,z_0|s_0)P(s_1|s_0,a_0)\nabla_\vtheta V^{\pi_{\vtheta}}(s_1,z_1)ds_{1}da_0 dz_1 ds_0 dz_0 \\
    &\quad=\gamma \int_\mathcal{S}\int_\mathcal{Z}\int_\mathcal{A}\int_\mathcal{Z} \int_\mathcal{S} \iota(s_0,z_0)\pi_{\vtheta}(a_0,z_1|s_0,z_0)P(s_1|s_0,a_0)\cdot\\
    &\quad\quad\quad\cdot\left[\int_\mathcal{A}\int_\mathcal{Z} \nabla_\vtheta\pi_{\vtheta}(a_1,z_2|s_1,z_1)Q^{\pi_{\vtheta}}(s_1,z_1,a_1,z_2) \right.\\
    &\quad\quad\quad\quad\left.+\gamma\int_\mathcal{S}\pi_{\vtheta}(a_1,z_2|s_1,z_1) P(s_2|s_1,a_1) \nabla_\vtheta V^{\pi_{\vtheta}}(s_2,z_2)ds_2 da_1 dz_2\right]ds_1 da_0 dz_1 ds_0 dz_0.
\end{align*}}

By splitting the integral again and rearranging the terms, we obtain:
{\small
\begin{align*}
&\int_\mathcal{S}\int_\mathcal{Z}\int_\mathcal{A}\int_\mathcal{Z}\iota(s_0,z_0)\pi_{\vtheta}(a_0,z_1|s_0,z_0)\nabla_\vtheta Q^{\pi_{\vtheta}}(s_0,z_0,a_0,z_1)da_0 dz_1 ds_0 dz_0 = \\
&\quad\quad\gamma \int_\mathcal{S}\int_\mathcal{Z}\int_\mathcal{A}\int_\mathcal{Z} \int_\mathcal{S}\int_\mathcal{Z} \iota(s_0)\pi_{\vtheta}(a_0|s_0)P(s_{1}|s_0,a_0)da_0ds_0\cdot \\
&\quad\quad\quad\cdot\left[\int_\mathcal{A}\int_\mathcal{Z} \nabla_\vtheta\pi_{\vtheta}(a_1|s_1)Q^{\pi_{\vtheta}}(s_1,a_1)da_1\right]ds_1\\    
&\quad\quad\quad\quad\quad+\gamma^2\int_\mathcal{S}\int_\mathcal{Z}\int_\mathcal{A}\int_\mathcal{Z} \int_\mathcal{S}\int_\mathcal{Z}
\iota(s_0)\pi_{\vtheta}(a_0|s_0)P(s_{1}|s_0,a_0)\pi_{\vtheta}(a_1|s_1) P(s_{2}|s_1,a_1)\cdot\\
&\quad\quad\quad\quad\quad\cdot\nabla_\vtheta V^{\pi_{\vtheta}}(s_2,z_2)ds_2dz_2da_1ds_1da_0ds_0.
\end{align*}}
To proceed and simplify the notation, we introduce the notion of t-step transition density $\tndensop{t}$, i.e. the density function of transitioning to a given state-policy state in $t$ steps under the stochastic policy $\pi_\vtheta$.  In particular, we can write

{\small
\begin{align*}
    &\tndensity{0} = \iota(s_0,z_0),\\
    &\tndensity{1} = \int_\mathcal{A}\int_\mathcal{Z} \int_\mathcal{S} \iota(s_0,z_0)\pi_{\vtheta}(a_0,z_1|s_0,z_0)P(s_1|s_0,a_0) ds_0 dz_0 da_0, \\
    &\tndensity{t} = \int_{\mathcal{S}^t\times\mathcal{Z}^{t}\times\mathcal{A}^t} \iota(s_0,z_0)\prod_{k=0}^{t-1}\pi_{\vtheta}(a_k, z_{k+1}|s_k,z_k)
    P(s_{k+1}|s_k,a_k) ds_k dz_k da_k.
\end{align*}}

Now we can put everything together, expanding the term containing $\nabla_\vtheta V^{\pi_{\vtheta}}$ an infinite amount of times
{\small
\begin{align*}
    \nabla_\vtheta\mathcal{J}_{\vtheta} =& \int_\mathcal{S}\int_\mathcal{Z} \tndensity{0}\int_\mathcal{A}\int_\mathcal{Z}\nabla_\vtheta\pi_{\vtheta}(a_0,z_1|s_0,z_0)Q^{\pi_{\vtheta}}(s_0,z_0,a_0,z_1)da_0 ds_0\\
    & +\gamma \int_\mathcal{S}\int_\mathcal{Z} \tndensity{1}\int_\mathcal{A}\int_\mathcal{Z} \nabla_\vtheta\pi_{\vtheta}(a_1,z_2|s_1,z_1)Q^{\pi_{\vtheta}}(s_1,z_1,a_1,z_2)da_1ds_1\\
    & +\gamma^2\int_\mathcal{S}\int_\mathcal{Z} \tndensity{2} \int_\mathcal{A}\int_\mathcal{Z} \nabla_\vtheta\pi_{\vtheta}(a_2,z_3|s_2,z_2)Q^{\pi_{\vtheta}}(s_2,z_2,a_2,z_3)da_2ds_2 \\
    &+ \gamma^3 \int_\mathcal{S}\int_\mathcal{Z} \tndensity{3}\int_\mathcal{A}\int_\mathcal{Z} \nabla_\vtheta\pi_{\vtheta}(a_3,z_4|s_3,z_3)Q^{\pi_{\vtheta}}(s_3,z_3,a_3,z_4)da_3ds_3 \\
    & + \dots \\
    =& \sum_{t=0}^{\infty} \gamma^t\int_\mathcal{S}\int_\mathcal{Z} \tndensity{t}\cdot\\
     &\hspace{1.5cm}\cdot\int_\mathcal{A}\int_\mathcal{Z}\nabla_\vtheta\pi_{\vtheta}(a_t,z_{t+1}|s_t,z_t)Q^{\pi_{\vtheta}}(s_t,z_t,a_t,z_{t+1})da_t dz_{t+1} ds_t dz_t.
\end{align*}
}

Exchanging the integral and the series we obtain:
{\small
\begin{align*}
     \nabla_\vtheta\mathcal{J}_{\vtheta} =& \int_\mathcal{S}\int_\mathcal{Z} \sum_{t=0}^{\infty}\gamma^t \tndensity{t}\cdot\\
     & \hspace{1.5cm}\cdot\int_\mathcal{A}\int_\mathcal{Z}\nabla_\vtheta\pi_{\vtheta}(a_t,z_{t+1}|s_t,z_t)Q^{\pi_{\vtheta}}(s_t,z_t,a_t,z_{t+1})da_t dz_{t+1} ds_t dz_t.
\end{align*}}
Now we remember that the occupancy metric is defined as:
{\small
\begin{equation*}
    \rho^{\pi_{\vtheta}}(s,z) = \sum_{t=0}^{\infty} \gamma^t \tndensop{t}(s,z),
\end{equation*}}
By replacing this definition and relabelling $s_t$, $z_t$, $a_t$, $z_{t+1}$ into $s$, $z$, $a$ and $z'$ we obtain:
{\small
\begin{equation*}
\nabla_\vtheta\mathcal{J}_{\vtheta} = \int_\mathcal{S}\int_\mathcal{Z} \rho^{\pi_{\vtheta}}(s,z) \int_\mathcal{A}\int_\mathcal{Z} \nabla_\vtheta\pi_{\vtheta}(a,z'|s,z) Q^{\pi_{\vtheta}}(s, z, a, z')da ds.
\end{equation*}}
\end{proof}

\subsection{Stateful Performance Difference Lemma}
\label{app:proof_perfomance_diff}

This lemma is particularly useful because it allows using the advantage function $A^q$ computed for an arbitrary policy $q$ to  evaluate the performance of our parametric policy $\pi_\vtheta$. As $q$ is an arbitrary policy not depending on $\vtheta$, the left-hand side of the Performance difference lemma can be used as a surrogate loss. This lemma allows to easily derive \gls{s2pg} version of \gls{ppo} and~\gls{trpo}.

\begin{lemma}{\textbf{(Performance Difference Lemma for Stateful Policies})}
For any stateful policy $\pi$ and $q$ in an arbitrary MDP, the difference of performance of the two policies in terms of expected discounted return can be computed as
{\small
\begin{equation*}
    \mathcal{J}(\pi) - \mathcal{J}(q) = \expect{\bar{\tau}\sim\pi,P }{\sum_{t=0}^{\infty}\gamma^t A^{q}(s_t,z_t,a_t,z_{t+1})}
\end{equation*}}
\label{lemma:performance_difference}
\end{lemma}

\begin{proof}
Starting from the lemma:
{\small
\begin{equation*}
    \mathcal{J}(\pi) - \mathcal{J}(q) = \expect{\tau\sim\pi,P }{\sum_{t=0}^{\infty}\gamma^t A^{q}(s_t,z_t,a_t,z_{t+1})}.
\end{equation*}}
We can rewrite the advantage as:
{\small
\begin{equation*}
    A^{q}(s_t,z_t,a_t,z_{t+1}) = \expect{s_{t+1}\sim P }{r(s_t,a_t) + \gamma V^{q}(s_{t+1}, z_{t+1}) - V^{q}(s_t,z_t)}.
\end{equation*}}
We can then manipulate the right-hand side as follows:
{\small
\begin{align*}
    \expect{\tau\sim\pi,P }{\sum_{t=0}^{\infty}\gamma^t A^{q}(s_t,z_t,a_t,z_{t+1})} =&  \expect{\tau\sim\pi,P }{\sum_{t=0}^{\infty}\gamma^t\left( r(s_t,a_t) + \gamma V^{q}(s_{t+1}, z_{t+1}) - V^{q}(s_t, z_t)\right)} \\
    =& \expect{\tau\sim\pi,P }{\sum_{t=0}^{\infty}\gamma^t r(s_t,a_t)} \\
    &+  \expect{\tau\sim\pi,P }{\sum_{t=0}^{\infty}\gamma^{t+1} V^{q}(s_{t+1}, z_{t+1})- \sum_{t=0}^{\infty}\gamma^{t} V^{q}(s_t,z_t)} \\
    =& \mathcal{J}(\pi) + \expect{\tau\sim\pi,P}{\sum_{t=1}^{\infty}\gamma^{t} V^{q}(s_t,z_t)- \sum_{t=0}^{\infty}\gamma^{t} V^{q}(s_t,z_t)} \\
    = & \mathcal{J}(\pi) +  \expect{(s,z)\sim\iota}{V^{q}(s)} \\
    = & \mathcal{J}(\pi) - \mathcal{J}(q).
\end{align*}}
This concludes the proof, as the right-hand side is equivalent to the left-hand side.
\end{proof}

\subsection{Stateful Deterministic Policy Gradient Theorem} 
\label{app:proof_dpgt}
\begin{proof}

The proof proceeds in a very similar fashion as the one for the standard \gls{dpgt}. We start by computing the gradient of a value function at an arbitrary state $s$ and an arbitrary policy state $z$

{\small
\begin{align}
\nabla_\vtheta V^{\mu_\vtheta}(s,z) =& \nabla_\vtheta Q^{\mu_\vtheta}(s,z, \mu^a_\vtheta(s,z), \mu^z_\vtheta(s,z)) \nonumber\\
=& \nabla_\vtheta \left(r(s,\mu^a_\vtheta(s,z)) + \int_\mathcal{S} \gamma P(s'|s, \mu^a_\vtheta(s,z)V^{\mu_\vtheta}(s',\mu^z_\vtheta(s,z)) ds' \right) \nonumber\\
=& \nabla_\vtheta \mu^a_\vtheta(s,z) \nabla_a r(s,a)\vert_{a=\mu^a_\vtheta(s,z)} + \nabla_\vtheta\int_\mathcal{S}\gamma P(s'|s, \mu^a_\vtheta(s,z)V^{\mu_\vtheta}(s',\mu^z_\vtheta(s,z)) ds' \nonumber \\
=& \nabla_\vtheta \mu^a_\vtheta(s,z) \nabla_a r(s,a)\vert_{a=\mu^a_\vtheta(s,z)} + \int_\mathcal{S} \gamma \left(\vphantom{\Big(} P(s'|s, \mu^a_\vtheta(s,z))\nabla_\vtheta V^{\mu_\vtheta}(s',\mu^z_\vtheta(s,z))\right. \nonumber\\
&+ \left.\nabla_\vtheta \mu^a_\vtheta(s,z) \nabla_a P(s'|s,a)\vert_{a=\mu^a_\vtheta(s,z)}V(s,\mu_\vtheta^z(s, z)) \right) ds' \nonumber \\
=& \nabla_\vtheta \mu^a_\vtheta(s,z)\nabla_a\left.\left(r(s,a) + \int_{\mathcal{S}}\gamma P(s'|s,a)V^{\mu_\vtheta}(s', \mu_\vtheta^z(s,z)) ds' \right)\right\vert_{a=\mu_\vtheta^a(s,z)} \nonumber\\
& + \nabla_\vtheta \mu^z_\vtheta(s,z)\left.\left(\int_{\mathcal{S}}\gamma P(s'|s,a)\vert_{a=\mu_\vtheta^a(s,z)}\nabla_{z'}V^{\mu_\vtheta}(s', z')ds'\right)\right\vert_{z'=\mu_\vtheta^z(s,z)} \nonumber \\
& + \int_{\mathcal{S}}\gamma P(s'|s, \mu_\vtheta^a(s,z))\left.\nabla_\vtheta V^{\mu_\vtheta}(s',z')\right\vert_{z'=\mu_\vtheta^z(s,z)}ds'
\label{eq:sdpg_first}
\end{align}}

Now we can observe that:
{\small
\begin{align}
\nabla_{z'}Q^{\mu_\vtheta}(s,z,a,z') &= \nabla_{z'}\left( r(s,a) + \int_{\mathcal{S}}\gamma P(s'|s,a)V^{\mu_\vtheta}(s',z')ds' \right) \nonumber\\
& = \int_{\mathcal{S}}\gamma P(s'|s,a)\nabla_{z'}V^{\mu_\vtheta}(s',z')ds'
\label{eq:sdpg_z}
\end{align}}

Replacing the left hand side of \eqref{eq:sdpg_z} into \eqref{eq:sdpg_first}, we can write:
{\small
\begin{align}
\nabla_\vtheta V^{\mu_\vtheta}(s,z) =& \nabla_\vtheta \mu^a_\vtheta(s,z)\nabla_a\left.Q^{\mu_\vtheta}(s,z,a,\mu_\vtheta^z(s,z))\right\vert_{a=\mu_\vtheta^a(s,z)} \nonumber\\
& + \nabla_\vtheta \mu^z_\vtheta(s,z)\left.\nabla_{z'}Q^{\mu_\vtheta}(s,z,\mu_\vtheta^a(s,z),z')\right\vert_{z'=\mu_\vtheta^z(s,z)} \nonumber \\
& + \int_{\mathcal{S}}\gamma P(s'|s,\mu^a_\vtheta(s,z))\left.\nabla_\vtheta V^{\mu_\vtheta}(s',z')\right\vert_{z'=\mu_\vtheta^z(s,z)}ds'
\label{eq:sdpg_second}
\end{align}}

To simplify the notation we write:
{\small
\begin{align*}
    G^{\mu_\vtheta}(s,z) =&  \nabla_\vtheta \mu^a_\vtheta(s,z)\nabla_a\left.Q^{\mu_\vtheta}(s,z,a,\mu_\vtheta^z(s,z))\right\vert_{a=\mu_\vtheta^a(s,z)} \\
    & + \nabla_\vtheta \mu^z_\vtheta(s,z)\left.\nabla_{z'}Q^{\mu_\vtheta}(s,z,\mu_\vtheta^a(s,z),z')\right\vert_{z'=\mu_\vtheta^z(s,z)}
\end{align*}}

Therefore, \eqref{eq:sdpg_second} can be written compactly as:
{\small
\begin{equation}
\nabla_\vtheta V^{\mu_\vtheta}(s,z) = G^{\mu_\vtheta}(s,z) + \int_{\mathcal{S}}\gamma P(s'|s,\mu^a_\vtheta(s,z))\left.\nabla_\vtheta V^{\mu_\vtheta}(s',z')\right\vert_{z'=\mu_\vtheta^z(s,z)}ds'
\label{eq:sdpg_third}
\end{equation}}

Using \eqref{eq:sdpg_third}, we can write the gradient of the objective function as:
{\small
\begin{align}
\nabla_\vtheta\mathcal{J}(\mu_\vtheta) =&\nabla_\vtheta \int_\mathcal{S}\int_\mathcal{Z} \iota(s,z)V^{\mu_\vtheta}(s,z)ds dz \nonumber \\
=&\int_\mathcal{S}\int_\mathcal{Z} \iota(s,z) G^{\mu_\vtheta}(s,z) ds dz \nonumber\\
& +\int_\mathcal{S}\int_\mathcal{Z} \iota(s,z)\int_{\mathcal{S}}\gamma P(s'|s,\mu^a_\vtheta(s,z))\left.\nabla_\vtheta V^{\mu_\vtheta}(s',z')\right\vert_{z'=\mu_\vtheta^z(s,z)}ds' ds dz.
\label{eq:sdpg_fourth}
\end{align}}

Before continuing the proof, we need to notice that the variables $s'$ and $z'$ are conditionally independent if the previous states $s$ and $z$ are given. Furthermore, it is important to notice that, even if the policy state $z$ is a deterministic function of the previous state and previous policy state, the variable $z$ is still a random variable if we consider its value after more than two steps of the environment, as its value depends on a set of random variables, namely the states encountered during the path.
Using these two assumptions, we extract the joint occupancy measure of $s$ and $z$.
Unfortunately, to proceed with the proof we need to abuse the notation (as commonly done in engineering) and use the Dirac's delta distribution $\delta(x)$.

Due to conditional independence, we can write the joint transition probability under the deterministic policy $\mu_\vtheta$ as
{\small
\begin{equation*}
    P^{\mu_\vtheta}(s', z'|s,z) = P^{\mu_\vtheta}(s'|s,z)\delta^{\mu_\vtheta}(z'|s,z)=P(s'|s,\mu^a_\vtheta(s,z))\delta(z'-\mu^z_\vtheta(s,z)),
\end{equation*}}
with $P^{\mu_\vtheta}(s'|s,z)=P(s'|s,\mu^a_\vtheta(s,z))$ and $\delta^{\mu_\vtheta}(z'|s,z)=\delta(z'-\mu^z_\vtheta(s,z))$.

Now we can write:
{\small
\begin{equation*}
    P^{\mu_\vtheta}(s'|s,z) = \int_\mathcal{Z}P^{\mu_\vtheta}(s', z'|s,z) dz' = \int_\mathcal{Z}P^{\mu_\vtheta}(s'|s,z)\delta^{\mu_\vtheta}(z'|s,z) dz'.
\end{equation*}}

We introduce, as done for the \gls{pgt}, the t-steps transition density $\dettndensop{t}$ under the deterministic policy $\mu_\vtheta$
{\small
\begin{align*}
    \dettndensity{0} &= \iota(s_0,z_0), \\
    \dettndensity{1} &= \int_\mathcal{S}\int_\mathcal{Z}\iota(s_0,z_0)P^{\mu_\vtheta}(s_1|s_0, z_0)\delta^{\mu_\vtheta}(z_1|s_0,z_0)ds_0dz_0, \\
    \dettndensity{t} &= \int_{\mathcal{S}^t\times\mathcal{Z}^t}\iota(s_0,z_0)\prod_{k=0}^{t-1}P^{\mu_\vtheta}(s_{k+1}|s_k, z_k)\delta^{\mu_\vtheta}(z_{k+1}|s_k,z_k)ds_kdz_k. \\
\end{align*}}


Using the previously introduced notation, we can expand \eqref{eq:sdpg_fourth}, by recursion, infinitely many times:
{\small
\begin{align}
\int_\mathcal{S}\int_\mathcal{Z}& \iota(s_0,z_0) \left(G^{\mu_\vtheta}(s_0,z_0)  + \iota(s_0,z_0)\int_{\mathcal{S}}\int_{\mathcal{Z}}\gamma P^{\mu_\vtheta}(s_1, z_1|s_0,z_0)\nabla_\vtheta V^{\mu_\vtheta}(s_1,z_1)ds_1dz_1\right) ds_0 dz_0 \nonumber\\
= & \int_\mathcal{S}\int_\mathcal{Z}\dettndensity{0}G^{\mu_\vtheta}(s_0,z_0)ds_0 dz_0 +\int_\mathcal{S}\int_\mathcal{Z}\int_\mathcal{S}
\int_\mathcal{Z}\gamma\iota(s_0,z_0)P^{\mu_\vtheta}(s_1, z_1|s_0,z_0) \cdot \nonumber \\
&\cdot\left(G^{\mu_\vtheta}(s_1,z_1) +\int_\mathcal{S}
\int_\mathcal{Z}\gamma P^{\mu_\vtheta}(s_2, z_2 | s_1, z_1)\nabla_\vtheta V^{\mu_\vtheta}(s_2,z_2)ds_2dz_2 \right) ds_1 dz_1\nonumber\\ 
= & \int_\mathcal{S}\int_\mathcal{Z}\dettndensity{0}G^{\mu_\vtheta}(s_0,z_0)ds_0 dz_0 + \int_\mathcal{S}\int_\mathcal{Z}\gamma\dettndensity{1}G^{\mu_\vtheta}(s_1,z_1)ds_1 dz_1 \nonumber \\
&+\int_{\mathcal{S}^3\times\mathcal{Z}^3}\gamma^2\iota(s_0,z_0)P^{\mu_\vtheta}(s_1, z_1|s_0,z_0)P^{\mu_\vtheta}(s_2, z_2|s_1,z_1) \nabla_\vtheta V^{\mu_\vtheta}(s_2,z_2)ds_2dz_2 ds_1 dz_1 ds_0 dz_0 \nonumber\\
= & \sum_{t=0}^{\infty} \int_\mathcal{S}\int_\mathcal{Z} \gamma^t\dettndensity{t}G^{\mu_\vtheta}(s_t,z_t) ds_t dz_t 
\label{eq:sdpg_fifth}
\end{align}
}

Now we notice that the occupancy measure is defined as
{\small
\begin{equation*}
\rho^{\mu_\vtheta}(s,z) = \sum_{t=0}^{\infty}\gamma^t\dettndensity{t}
\end{equation*}}

Therefore, by exchanging the order of series and integration we obtain:
{\small
\begin{align*}
\nabla_\vtheta\mathcal{J}(\mu_\vtheta) = & \int_\mathcal{S}\int_\mathcal{Z}\sum_{t=0}^{\infty}\gamma^t\dettndensop{t}(s,z)G^{\mu_\vtheta}(s,z)ds dz \\
= & \int_\mathcal{S}\int_\mathcal{Z} \rho^{\mu_\vtheta}(s,z) \left( \nabla_\vtheta \mu^a_\vtheta(s,z)\nabla_a\left.Q^{\mu_\vtheta}(s,z,a,\mu_\vtheta^z(s,z))\right\vert_{a=\mu_\vtheta^a(s,z)} \right. \\
& \hspace{2.5cm}\left. + \nabla_\vtheta \mu^z_\vtheta(s,z)\left.\nabla_{z'}Q^{\mu_\vtheta}(s,z,\mu_\vtheta^a(s,z),z')\right\vert_{z'=\mu_\vtheta^z(s,z)} \right)
\end{align*}}
which concludes the proof.
\end{proof}

\subsection{Variance Analysis of Policy Gradient Estimates}
\label{app:variance_analysis}
In the following section, we define the variance of a random vector $A = (A_1, \dots, A_l)^\top$ as in~\cite{papini2022smoothing,zhao2011analysis}
{\small
\begin{align}
    \var[A] =& \tr\left( \mathbb{E}\left[ (A- \expect{}{A} )(A- \expect{}{A} )^\top\right] \right) \label{eq:var_trace}\\
    =& \sum_{m=1}^l \expect{}{(A_m - \expect{}{A_m})^2} \label{eq:var_trace2} 
\end{align}}

Using this definition, we can compute the variance of the policy gradient estimators for stateful policies using \gls{bptt} and \gls{s2pg}.

\subsubsection{Backpropagation-Through-Time}\label{proof:bptt_var}
\begin{proof}
We consider a policy of the following form
{\small
\begin{equation}
    \nu(a_t | h_t, \Theta) = \mathcal{N}(a_t|\mu_\theta(h_t), \Sigma) = \frac{1}{\sqrt{2\pi^d|\Sigma|}}\exp\left(-\frac{1}{2}(a_t - \mu_\theta(h_t))^\top \Sigma^{-1} (a_t - \mu_\theta(h_t))\right)\, , 
\end{equation}}
where $\Theta = \{\theta, \Sigma\}$. For a fixed covariance $\Sigma$, a REINFORCE-style gradient estimator for this policy depends on the following  gradient
{\small
\begin{align}
    f(\tau) =& \sum_{t=0}^{T-1} \nabla_{\theta} \log  \pi(a_t | h_t, \Theta) = \sum_{t=0}^{T-1} \nabla_{\theta} \mu_\theta(h_t)\nabla_{m_t} \log  \mathcal{N}(a_t|m_t, \Sigma)|_{m_t=\mu_\theta (h_t)} \nonumber\\
    =& \sum_{t=0}^{T-1}  \nabla_{\theta} \mu_\theta(h_t)  \left(\Sigma^{-1} (a_t - m_t)\right)|_{m_t=\mu_\theta (h_t)}  \nonumber\, . 
\end{align}}\noindent

Given the total discounted return of a trajectory
{\small
\begin{equation*}
    G(\tau) = \sum_{t=0}^{T-1} \gamma^{t-1} r(s_t, a_t, s_{t+1})  \, , 
\end{equation*}}\noindent
the variance of the empirical gradient approximator is given by 
{\small
\begin{align}
    \var\left[ \nabla_\theta \hat{J}(\Theta) \right] = \frac{1}{N} \var \left[ G(\tau) f(\tau) \right] \, ,
\end{align}}\noindent
where $N$ is the number of trajectories used for the empirical gradient estimator. Therefore, we can just focus on $\var \left[ G(\tau) f(\tau) \right]$. 
Using \eqref{eq:var_trace2}, we can define an upper bound on $\var \left[ G(\tau) f(\tau) \right]$
{\small
\begin{align*}
    \var \left[ G(\tau) f(\tau) \right] \leq& \sum_{i=1}^l \expect{}{(G f_i)^2}\\
    =& \expect{}{G^2f^\top f}\\
    =& \int_\tau p(\tau) \left(\sum_{t=0}^{T-1} \gamma^{t-1} r(s_t, a_t, s_{t+1})\right)^2 \\
    &\cdot \left(\sum_{t=0}^{T-1}  \nabla_{\theta} \mu_\theta(h_t)  \left( \Sigma^{-1} (a_t - m_t)\right)|_{m_t=\mu_\theta (h_t)} \right)^\top \\
    &\cdot \left(\sum_{t=0}^{T-1}  \nabla_{\theta} \mu_\theta(h_t)  \left( \Sigma^{-1} (a_t - m_t)\right)|_{m_t=\mu_\theta (h_t)}  \right) \, . 
\end{align*}}\noindent
Let $\xi_{t} =  \Sigma^{-1} (a_t - m_t)$ where $a_t\sim \mathcal{N}(m_t, \Sigma)$, then $\xi_{t}(m_t)$ are random variables drawn from a Gaussian distribution $\mathcal{N}(0, \Sigma^{-1})$. Hence, we define $p(\xi_{t}) = \mathcal{N}(0, \Sigma^{-1})$ and treat $\xi_{t}$ as an independent random variable sampled from $p(\xi_{t})$. 
For compactness of notation, we define $\xi=\left[\xi_0, \dots, \xi_T\right]^\top$, where $p(\xi)=\prod_{t=0}^{T}p(\xi_t)$. In a similar way we define \mbox{$p(\tau|\xi)=\iota(s_0)\prod_{t=0}^{T}P(s_{t+1}|s_{t},\xi_{t})$.} Using this notation, we can write
{\small
\begin{align}
        \var \left[ G(\tau) f(\tau) \right] \leq& 
    \int_{\tau, \xi} p(\xi)\, p(\tau|\xi)\left(\sum_{t=0}^{T-1} \gamma^{t-1} r(s_t, a_t, s_{t+1})\right)^2  \left(\sum_{t=0}^{T-1}  \nabla_{\theta} \mu_\theta(h_t) \,  \xi_{t} \right)^\top  \left(\sum_{t=0}^{T-1}  \nabla_{\theta} \mu_\theta(h_t)  \, \xi_{t} \right)\nonumber\\
    \leq& \frac{R^2 (1-\gamma^T)^2}{(1-\gamma)^2}
     \expect{\tau, \xi}{\sum_{t=0}^{T-1}\sum_{t'=1}^T \xi_{t}^\top \nabla_{\theta}\mu_\theta(h_t)^\top \, \nabla_{\theta} \mu_\theta(h_{t'}) \, \xi_{t} } \nonumber\\
     =& \frac{R^2 (1-\gamma^T)^2}{(1-\gamma)^2} \, \left( \expect{\tau, \xi}{\sum_{t=0}^{T-1} \xi_{t}^\top  \nabla_{\theta} \mu_\theta(h_t)^\top \,   \nabla_{\theta} \mu_\theta(h_{t})\, \xi_{t}  } \right. \label{eq:var_bptt_inter_1}\\
      & \phantom{ \frac{R^2 (1-\gamma^T)^2}{(1-\gamma)^2}}\quad+ \left. \expect{\tau, \xi}{\sum_{t, t'=0, t\neq t'}^{T-1}\xi_{t}^\top  \nabla_{\theta} \mu_\theta(h_t)^\top \,  \,  \nabla_{\theta} \mu_\theta(h_{t'}) \, \xi_{t}  }\right) \, . \nonumber
\end{align}}\noindent
Now, we take a look at the last term. Without loss of generality, we will assume $t'>t$.
As $\xi_{t}$ is a random variable drawn from a Gaussian distribution with zero mean, and as $h_{t'}$ and $h_t$ are independent from $\xi_{t'}$, we can write
{\small
\begin{align*}
    &\expect{\tau, \xi}{\sum_{t, t'=0, t\neq t'}^{T-1}\xi_{t}^\top\nabla_{\theta} \mu_\theta(h_t)^\top \,  \,  \nabla_{\theta} \mu_\theta(h_{t'}) \xi_{t'} }\\
    &\qquad = \sum_{t, t'=0, t\neq t'}^{T-1}\expect{\tau, \xi}{ \xi_{t}^\top \nabla_{\theta} \mu_\theta(h_t)^\top \,  \nabla_{\theta} \mu_\theta(h_{t'})  \xi_{t'} } \\
    &\qquad = \sum_{t, t'=0, t\neq t'}^{T-1}\expect{\tau, \xi}{ \xi_{t}^\top \nabla_{\theta} \mu_\theta(h_t)^\top\nabla_{\theta} \mu_\theta(h_{t'}) } \, \expect{\xi_{t'}}{\xi_{t'}} \\
    &\qquad  = \quad 0 \, .
\end{align*}}\noindent
where in the last line we used the fact that $\expect{\xi_{t'}}{\xi_{t'}}=\vzero$. The same result holds for $t>t'$, following a similar derivation.

Continuing from \eqref{eq:var_bptt_inter_1}, let $\Xi_t = \nabla_{\theta} \mu_\theta(h_t)^\top \,   \nabla_{\theta} \mu_\theta(h_t)\, $. Noticing that $\Xi_t$ is a square positive semi-definite matrix, we can write 
{\small
\begin{align*}
    & \frac{R^2 (1-\gamma^T)^2}{(1-\gamma)^2} \, \expect{\tau, \xi}{\sum_{t=0}^{T-1} \xi_{t}^\top\nabla_{\theta} \mu_\theta(h_t)^\top \,   \nabla_{\theta} \mu_\theta(h_{t}) \xi_{t} } \\
    &\qquad= \frac{R^2 (1-\gamma^T)^2}{(1-\gamma)^2} \, \sum_{t=0}^{T-1}  \expect{\tau,\xi}{\expect{\xi_{t}}{\xi_{t}^\top\nabla_{\theta} \mu_\theta(h_t)^\top \,   \nabla_{\theta} \mu_\theta(h_{t}) \xi_{t} }} \\
    &\qquad= \frac{R^2 (1-\gamma^T)^2}{(1-\gamma)^2} \, \sum_{t=0}^{T-1}  \expect{\tau,\xi}{\tr\left(\nabla_{\theta} 
 \mu_\theta(h_t)^\top \,  \nabla_{\theta} 
 \mu_\theta(h_{t}) \,  \Sigma^{-1}\right) }    \\
&\qquad\leq \frac{R^2 (1-\gamma^T)^2}{(1-\gamma)^2} \, \sum_{t=0}^{T-1}  \expect{\tau,\xi}{ \fnorm{\nabla_{\theta} \mu_\theta(h_t)^\top} \cdot  \fnorm{\nabla_{\theta} \mu_\theta(h_t)} \cdot \fnorm{\Sigma^{-1}} }\\
&\qquad\leq \frac{R^2 (1-\gamma^T)^2}{(1-\gamma)^2} \, \sum_{t=0}^{T-1}  \expect{\tau,\xi}{ \fnorm{\nabla_{\theta} \mu_\theta(h_t)}^2 \cdot \fnorm{\Sigma^{-1}} }\,, 
\end{align*}}\noindent
where $\fnorms{.}$ is the Frobenius norm. When $\mu(h_t)$ is implemented as a recursive function  $\mu(h_t) = f_\theta (s_t, z_{t+1})$ where $z_{t+1} = \eta_\theta (s_t, z_t)$ the gradient is given by 
{\small
\begin{align}
    \nabla_\theta \mu_\theta(h_t) = \frac{\partial}{\partial \theta} f_\theta(s_t, z_t) + \left( \sum_{i=0}^{t-1} \frac{\partial}{\partial \theta} \eta_\theta(s_{i}, z_{i}) \prod_{j=i+1}^t\frac{\partial}{\partial z_j} \eta_\theta(s_j, z_j) \right) \frac{\partial}{\partial z_t} f_\theta(s_t, z_t) \, . \label{eq:grad_bppt}
\end{align}}

Using this gradient and using the Assumptions \ref{assump:grad1} and \ref{assump:grad2}, we bound the norm of the gradient of $\mu_\theta(h_t)$ as follows
{\small
\begin{align*}
\fnorm{ \nabla_\theta \mu_\theta(h_t)} \leq & \fnorms{ \frac{\partial}{\partial \theta} f_\theta(s_t, z_t)} + \left( \sum_{i=0}^{t-1}\fnorms{\frac{\partial}{\partial \theta} \eta_\theta(s_{i}, z_{i}) } \prod_{j=i+1}^t \fnorms{\frac{\partial}{\partial z_j} \eta_\theta(s_j, z_j)} \right) \fnorms{\frac{\partial}{\partial z_t} f_\theta(s_t, z_t)} \\
\leq &\,  F + \left(\sum_{i=0}^{t-1} H \prod_{j=i+1}^t Z \right) K\\
= & \,  F + H \left(\sum_{i=0}^{t-1} Z^{t-i-1} \right) K\,.  
\end{align*}}

Hence, the final bound is given by 

{\small
\begin{align*}
     \var\left[ \nabla_\theta \hat{J}_{\text{BPTT}}(\Theta) \right]  &\leq \frac{R^2 \fnorm{\Sigma^{-1}} (1-\gamma^T)^2}{N (1-\gamma)^2} \, \sum_{t=0}^{T-1} \left( F + H \left(\sum_{i=0}^{t-1} Z^{t-i-1} \right) K \right)^2  \\
      &= \frac{R^2 \fnorm{\Sigma^{-1}} (1-\gamma^T)^2}{N (1-\gamma)^2} \, \bigl( TF^2 +2FHK\tilde{Z} + H^2K^2\bar{Z}^2 \bigr) \, , 
\end{align*}\noindent
}
where we define $\tilde{Z} = \sum_{t=0}^{T-1} \sum_{i=0}^{t-1} Z^{t-i-1}$ and $\bar{Z} = \sum_{t=0}^{T-1} \left(\sum_{i=0}^{t-1} Z^{t-i-1}\right)^2$ for brevity, concluding the proof.
\end{proof}

We can further analyze the properties of the \gls{bptt} estimator with the following lemma:

\begin{lemma}
\label{lemma:z_constants}
For $Z<1$ the two constants $\tilde{Z}$ and $\bar{Z}$ grow linearly with $T$, while for $Z>1$ the two constants grow exponentially with $T$. 
\end{lemma}

\begin{proof}
The series expressing $\bar{Z}$ can be written as 
{\small
\begin{align}
    \hat{Z}=\sum_{t=0}^{T-1} \left(\sum_{i=0}^{t-1} Z^{t-i-1}\right)^2 &= \sum_{t=0}^{T-1} \left(\sum_{k=0}^{t-1} Z^{k}\right)^2 =  \sum_{t=0}^{T-1} \frac{(1-Z^t)^2}{(1-Z)^2} \nonumber\\
    & = \frac{1}{(1 - Z)^2} \left(T + \sum_{t=0}^{T-1} Z^{2t} -2 \sum_{t=0}^{T-1} Z^{t}\right) \nonumber \\
    & = \frac{T}{(1 - Z)^2} + \frac{1-Z^{2T}}{1-Z^2} -2 \frac{1-Z^{T}}{1-Z} \nonumber \\
    & = \frac{T}{(1 - Z)^2} + \frac{1}{(1-Z)^2}\left(\frac{1-Z^{2T}-2(1-Z^{T})(1+Z)}{(1-Z)(1+Z)}\right) \nonumber \\
    & = \frac{T}{(1 - Z)^2} + \frac{(1-Z^T)(1+Z^T)-2(1-Z^{T})(1+Z)}{(1-Z)^3(1+Z)} \nonumber \\
    & = \frac{T}{(1 - Z)^2} + \frac{(1-Z^T)(Z^{T}-2Z-1)}{(1-Z)^3(1+Z)} \nonumber \\
    & = \frac{T}{(1 - Z)^2} + \frac{(Z^T-1)(Z^{T}-2Z-1)}{(Z-1)^3(1+Z)} \label{eq:sum_z_bar}
\end{align}}
\clearpage
From~\eqref{eq:sum_z_bar} we see that $\bar{Z}$ is composed of two terms, where the first is $\mathcal{O}(T)$ while the second one is $\mathcal{O}(Z^{2T})$.
The series expressing $\tilde{Z}$ can be written as 
{\small
\begin{align}
    \tilde{Z}=\sum_{t=0}^{T-1} \sum_{i=0}^{t-1} Z^{t-i-1} &= \sum_{t=0}^{T-1} \sum_{k=0}^{t-1} Z^{k} =  \sum_{t=0}^{T-1} \frac{1-Z^t}{1-Z} \nonumber\\
    & = \frac{1}{1 - Z} \left(T -\sum_{t=0}^{T-1} Z^t\right) = \frac{T}{1 - Z} + \frac{Z^T - 1}{(1-Z)^2} \label{eq:sum_z_tilde}
\end{align}}
From~\eqref{eq:sum_z_tilde} we see that $\tilde{Z}$ is composed of two terms, where the first is $\mathcal{O}(T)$ while the second one is $\mathcal{O}(Z^T)$. These observations conclude the proof.
\end{proof}

\subsubsection{Stochastic Stateful Policies}\label{proof:s2pg_var}
 \begin{proof}
 For stochastic stateful policies, we consider similar policies of the form 
{\small
\begin{align}
    \pi(a_t, z_{t+1} | s_t, z_{t}, \tilde{\Theta}) &= \mathcal{N}(\vartheta_{{\theta}}(s_t, z_{t+1}), \Gamma)\label{eq:policy_s2pg}\\
    &= \frac{1}{\sqrt{2\pi^d|\Gamma|}}\exp\left(-\frac{1}{2}(a_t - \vartheta_{{\theta}}(s_t, z_{t}))^\top \Gamma^{-1} (a_t - \vartheta_{{\theta}}(s_t , z_{t}))\right)\, , \nonumber
\end{align}}

where $\tilde{\Theta} = \{ \theta, \Gamma \}$, where $\vartheta_\theta$ is the mean vector containing the means corresponding to an action and a hidden state such that $\vartheta_\theta (s_t, z_t)= \left[ \mu_{\theta}^a (s_t, z_t)^\top, \, \mu_{\theta}^z (s_t, z_t)^\top \right]^\top$, and $\Gamma$ is a covariance matrix containing the covariances corresponding to an action and a hidden state such that \mbox{$\Gamma = \begin{bmatrix}
    \Sigma & 0 \\
    0 & \Upsilon
\end{bmatrix}$}.

Then, we can write the gradient with respect to a trajectory as follows 
{\small
\begin{align*}
    \tilde{f}(\tau) =& \sum_{t=0}^{T-1} \nabla_{\theta} \log \pi(a_t, z_{t+1} | s_t, z_{t}, \tilde{\Theta}) = \sum_{t=0}^{T-1} \nabla_{\theta} \left(\log \pi(a_t| s_t, z_{t}, \tilde{\Theta}) + \log \pi(z_{t+1} | s_t, z_{t}, \tilde{\Theta}) \right) \\
    =& \sum_{t=0}^{T-1} \nabla_{\theta} \mu_{\theta}^a (s_t, z_t)\nabla_{m_t} \log  \mathcal{N}(m_t, \Sigma)|_{m_t=\mu_\theta (s_t, z_t)} \\
    &+ \nabla_{\theta} \mu_{\theta}^z (s_t, z_t)\nabla_{b_t} \log  \mathcal{N}(b_t, \Upsilon)|_{b_t=\mu_{\theta}^z (s_t, z_t)} \\
    =& \sum_{t=0}^{T-1}  \nabla_{\theta} \mu_{\theta}^a (s_t, z_t)  \left( \Sigma^{-1} (a_t - m_t)\right)|_{m_t=\mu_\theta (s_t, z_t)} \\
    &+ \nabla_{\theta} \mu_{\theta}^z (s_t, z_t)  \left( \Upsilon^{-1} (z_{t+1} - b_t)\right)|_{b_t=\mu_{\theta}^z (s_t, z_t)}\, , 
    \end{align*}}

where we have used the fact that $a_t$ and $z_{t+1}$ are independent for decomposing the gradient into a part corresponding to the mean of the Gaussian used for sampling action and a part corresponding to the mean of the Gaussian used to sample the next hidden state. 

Let $\tilde{\xi}_{t} =  \Upsilon^{-1} (z_{t+1} - b_t)$ where $z_{t+1}\sim \mathcal{N}(b_t, \Upsilon)$, then $\tilde{\xi}_{t}(b_t)$ are random variables drawn from a Gaussian distribution $\mathcal{N}(0, \Upsilon^{-1})$. Hence, we define $p(\tilde{\xi}_{t}) = \mathcal{N}(0, \Upsilon^{-1})$ and treat $\tilde{\xi}_{t}$ as an independent random variable sampled from $p(\tilde{\xi}_{t})$. Then we can use \eqref{eq:var_trace2}
again to define an upper bound on the variance such that 
{\small
\begin{align*}
\var \left[ G(\tau) \tilde{f}(\tau) \right] \leq& 
\int_{\tau, \xi} p(\tau)\, p(\xi)\, p(\tilde{\xi}) \left(\sum_{t=0}^{T-1} \gamma^{t-1} r(s_t, a_t, s_{t+1})\right)^2  \\
& \cdot \left(\sum_{t=0}^{T-1}  \nabla_{\theta} \mu_{\theta}^a (s_t, z_t) \,  \xi_{t}  + \nabla_\theta \mu_{\theta}^z (s_t, z_t) \tilde{\xi}_t \right)^\top \\
& \cdot \left(\sum_{t=0}^{T-1}  \nabla_{\theta} \mu_{\theta}^a (s_t, z_t) \,  \xi_{t}  + \nabla_\theta \mu_{\theta}^z (s_t, z_t) \tilde{\xi}_t \right)\\
 \leq& \frac{R^2 (1-\gamma^T)^2}{(1-\gamma)^2}
 \expect{\tau, \xi, \tilde{\xi}}{\sum_{t=0}^{T-1}\sum_{t'=1}^T \left(  \nabla_{\theta} \mu_{\theta}^a (s_t, z_t) \,  \xi_{t}  + \nabla_\theta \mu_{\theta}^z (s_t, z_t) \tilde{\xi}_t \right)^\top \right.\\
 & \cdot \left. \vphantom{\sum_{t'=1}^T} \left( \nabla_{\theta} \mu_{\theta}^a(s_{t'}, z_{t'})  \,  \xi_{t'}  + \nabla_\theta \mu_{\theta}^z (s_{t'}, z_{t'}) \tilde{\xi}_{t'}\right)} \\
  =& \frac{R^2 (1-\gamma^T)^2}{(1-\gamma)^2}
 \expect{\tau, \xi, \tilde{\xi}}{\sum_{t=0}^{T-1} \left(  \nabla_{\theta} \mu_{\theta}^a (s_t, z_t) \,  \xi_{t}  + \nabla_\theta \mu_{\theta}^z (s_t, z_t) \tilde{\xi}_t \right)^\top \right.\\
 & \cdot \left. \vphantom{\sum_{t=0}^{T-1}} \left( \nabla_{\theta} \mu_{\theta}^a(s_t, z_t) \,  \xi_{t}  + \nabla_\theta \mu_{\theta}^z (s_t, z_t)\tilde{\xi}_{t}\right)} \\
  & \cdot  
 \expect{\tau, \xi, \tilde{\xi}}{\sum_{t, t'=0, t\neq t'}^{T-1}\left(  \nabla_{\theta} \mu_{\theta}^a (s_t, z_t) \,  \xi_{t}  + \nabla_\theta \mu_{\theta}^z (s_t, z_t) \tilde{\xi}_t \right)^\top \right.\\
 & \cdot \left. \vphantom{\sum_{t'=1}^T} \left( \nabla_{\theta} \mu_{\theta}^a(s_{t'}, z_{t'}) \,  \xi_{t'}  + \nabla_\theta \mu_{\theta}^z (s_{t'}, z_{t'}) \tilde{\xi}_{t'}\right)} \, . 
\end{align*}}

Again, the last term can be dropped due to
{\small
\begin{align*}
& \expect{\tau, \xi, \tilde{\xi}}{\sum_{t, t'=0, t\neq t'}^{T-1}\left(  \nabla_{\theta} \mu_{\theta}^a (s_t, z_t) \,  \xi_{t}  + \nabla_\theta \mu_{\theta}^z (s_t, z_t) \tilde{\xi}_t \right)^\top \vphantom{\sum_{t'=1}^T} \left( \nabla_{\theta} \mu_{\theta}^a(s_{t'}, z_{t'}) \,  \xi_{t'}  + \nabla_\theta \mu_{\theta}^z (s_{t'}, z_{t'}) \tilde{\xi}_{t'}\right)} \\
&\quad= \sum_{t, t'=0, t\neq t'}^{T-1}\expect{\tau, \xi, \tilde{\xi}}{ \left(  \nabla_{\theta} \mu_{\theta}^a (s_t, z_t) \,  \xi_{t}  + \nabla_\theta \mu_{\theta}^z (s_t, z_t) \tilde{\xi}_t \right)^\top \vphantom{\sum_{t'=1}^T} \left( \nabla_{\theta} \mu_{\theta}^a(s_{t'}, z_{t'}) \,  \xi_{t'}  + \nabla_\theta \mu_{\theta}^z (s_{t'}, z_{t'}) \tilde{\xi}_{t'}\right)} \\
&\quad= \sum_{t, t'=0, t\neq t'}^{T-1}\expect{\tau, \xi, \tilde{\xi}}{ \left(  \nabla_{\theta} \mu_{\theta}^a (s_t, z_t) \,  \xi_{t}  + \nabla_\theta \mu_{\theta}^z (s_t, z_t) \tilde{\xi}_t \right)^\top} \\
&\qquad \cdot \expect{\tau, \xi, \tilde{\xi}}{\vphantom{\sum_{t'=1}^T} \left( \nabla_{\theta} \mu_{\theta}^a(s_{t'}, z_{t'}) \,  \xi_{t'}  + \nabla_\theta \mu_{\theta}^z (s_{t'}, z_{t'}) \tilde{\xi}_{t'}\right)} \\
&\quad = \sum_{t, t'=0, t\neq t'}^{T-1}\biggl( \underbrace{\expect{\tau, \xi, }{ \vphantom{\tilde{\xi}_t^\top} \nabla_{\theta} \mu_{\theta}^a (s_t, z_t)^\top \,  \xi_{t}^\top}}_{=0} + \underbrace{\expect{\tau, \tilde{\xi}}{  \nabla_\theta \mu_{\theta}^z (s_t, z_t)^\top \, \tilde{\xi}_t^\top}}_{=0} \biggr) \\
&\qquad \cdot \biggl(\underbrace{\expect{\tau, \xi}{\vphantom{\tilde{\xi}_{t'}}  \nabla_{\theta} \mu_{\theta}^a(s_{t'}, z_{t'}) \,  \xi_{t'}}}_{=0}  + \underbrace{\expect{\tau,\tilde{\xi}}{\nabla_\theta \mu_{\theta}^z (s_{t'}, z_{t'}) \tilde{\xi}_{t'}}}_{=0}\biggr)= 0 \, .  \\
\end{align*}}

Noticing that $\nabla_{\theta} \mu_{\theta}^a (s_t, z_t)^\top \,   \nabla_{\theta} \mu_{\theta}^a (s_t, z_t)\, $ and $\nabla_{\theta} \mu_{\theta}^z (s_t, z_t)^\top \,   \nabla_{\theta} \mu_{\theta}^z (s_t, z_t)\, $ are square positive semi-definite matrices, allows us to write

{\small
\begin{align*}
 & \frac{R^2 (1-\gamma^T)^2}{(1-\gamma)^2}
 \expect{\tau, \xi, \tilde{\xi}}{\sum_{t=0}^{T-1} \left(  \nabla_{\theta} \mu_{\theta}^a (s_t, z_t) \,  \xi_{t}  + \nabla_\theta \mu_{\theta}^z (s_t, z_t) \tilde{\xi}_t \right)^\top  \left( \nabla_{\theta} \mu_{\theta}^a (s_t, z_t)\,  \xi_{t}  + \nabla_\theta \tilde{\mu}_\theta (s_{t}, z_{t}) \tilde{\xi}_{t}\right)} \\
&\quad = \frac{R^2 (1-\gamma^T)^2}{(1-\gamma)^2}
 \sum_{t=0}^{T-1} \expect{\tau, \xi, \tilde{\xi}}{ \left(  \nabla_{\theta} \mu_{\theta}^a (s_t, z_t) \,  \xi_{t}  + \nabla_\theta \mu_{\theta}^z (s_t, z_t) \tilde{\xi}_t \right)^\top  \left( \nabla_{\theta} \mu_{\theta}^a (s_t, z_t)\,  \xi_{t}  + \nabla_\theta \tilde{\mu}_\theta (s_{t}, z_{t}) \tilde{\xi}_{t}\right)} \\
 &\quad =  \frac{R^2 (1-\gamma^T)^2}{(1-\gamma)^2}\\
 &\qquad\cdot  
  \biggl(\sum_{t=0}^{T-1} \expect{\tau, \xi}{ 
\left(\nabla_{\theta} \mu_{\theta}^a (s_t, z_t) \,  \xi_{t} \vphantom{\tilde{\xi}_t}\right)^\top  
\left(\nabla_{\theta} \mu_{\theta}^a (s_t, z_t) \,  \xi_{t} \vphantom{\tilde{\xi}_t}\right)} + 
 \expect{\tau, \tilde{\xi}}{ \left(\nabla_\theta \mu_{\theta}^z (s_t, z_t) \tilde{\xi}_t \right)^\top 
\left(\nabla_\theta \mu_{\theta}^z (s_t, z_t) \tilde{\xi}_t \right) }\biggr. \\
&\qquad\quad  \biggl.+   \underbrace{\expect{\tau, \xi, \tilde{\xi}}{\left(\nabla_{\theta} \mu_{\theta}^a (s_t, z_t) \,  \xi_{t} \vphantom{\tilde{\xi}_t}\right)^\top  
\left(\nabla_\theta \mu_{\theta}^z (s_t, z_t) \tilde{\xi}_t \right) }}_{=0} + \underbrace{\expect{\tau, \xi, \tilde{\xi}}{
\left(\nabla_\theta \mu_{\theta}^z (s_t, z_t) \tilde{\xi}_t \right)^\top \left(\nabla_{\theta} \mu_{\theta}^a (s_t, z_t) \,  \xi_{t} \vphantom{\tilde{\xi}_t}\right)  }}_{=0}\biggr) \\
 &\quad= \frac{R^2 (1-\gamma^T)^2}{(1-\gamma)^2}
  \biggl(\sum_{t=0}^{T-1} \expect{\tau}{
  \tr\left(  \nabla_{\theta} \mu_{\theta}^a (s_t, z_t)^\top \nabla_{\theta} \mu_{\theta}^a (s_t, z_t)  \, \Sigma^{-1}\right)
} \\
 &\qquad\qquad\qquad\qquad\qquad\, \, \,  + \expect{\tau}{ \tr\left( \nabla_{\theta} \mu_{\theta}^z (s_t, z_t)^\top \nabla_{\theta} \mu_{\theta}^z (s_t, z_t) \Upsilon^{-1}\right) }\biggr) \\
&\quad\leq \frac{R^2 (1-\gamma^T)^2}{(1-\gamma)^2}
  \biggl(\sum_{t=0}^{T-1} \expect{\tau}{ \fnorm{\nabla_{\theta} \mu_{\theta}^a (s_t, z_t)}^2 \cdot \fnorm{\Sigma^{-1}}} + 
  \expect{\tau}{\fnorm{\nabla_{\theta} \mu_{\theta}^z (s_t, z_t)}^2 \cdot \fnorm{\Upsilon^{-1}}} \biggr) \, . 
\end{align*}
}

To align with the notation from the previous proof, we define $\mu_{\theta}^a (s_t, z_t) = f_\theta(s_t, z_t)$ and  $\mu_{\theta}^z (s_t, z_t) = \eta_\theta(s_t, z_t)$ and get
%
%
{\small 
\begin{align*}
& \frac{R^2 (1-\gamma^T)^2}{(1-\gamma)^2}
  \biggl(\sum_{t=0}^{T-1} \expect{\tau}{ \fnorm{\nabla_{\vtheta} f_{\vtheta} (s_t, z_t)}^2 \cdot \fnorm{\Sigma^{-1}}} + 
  \expect{\tau}{\fnorm{\nabla_{\vtheta} \nu_{\vtheta}(s_t, z_t)}^2 \cdot \fnorm{\Upsilon^{-1}}} \biggr)\\
& \quad \leq \frac{R^2 (1-\gamma^T)^2}{(1-\gamma)^2}
  \biggl(\sum_{t=0}^{T-1}  F^2 \cdot \fnorm{\Sigma^{-1}} + 
 H^2 \cdot \fnorm{\Upsilon^{-1}} \biggr)
\end{align*}
}
such that the final bound on the variance of the \gls{s2pg} is given by 
{\small
\begin{align*}
    \var\left[ \nabla_\theta \hat{J}_{\text{S2PG}}(\Theta) \right] \leq \frac{R^2 \fnorm{\Sigma^{-1}}(1-\gamma^T)^2}{N(1-\gamma)^2}
  \biggl(\sum_{t=0}^{T-1}  F^2 + 
 H^2 \cdot \frac{\fnorm{\Upsilon^{-1}}}{ \fnorm{\Sigma^{-1}} } \biggr)\, , 
\end{align*}
}\noindent
concluding the proof.
\end{proof}

\begin{lemma}
Given a policy, similarily to \eqref{eq:policy_s2pg},  $\pi(a_t, z_{t+1} | s_t, z_{t}, \tilde{\Theta}) = \mathcal{N}(\vartheta_{{\theta}}(s_t, z_{t+1}), \Gamma)$  with $\Gamma = \begin{bmatrix}
    \Sigma & 0 \\
    0 & \Upsilon
\end{bmatrix}$ and limiting the covariance matrices $\Sigma$ and $\Upsilon$ to be diagonal, then the bound from Theorem \ref{lemma:variance_s2pg} can be replaced by a tighter bound 
{\small
\begin{align*}
        \var\left[ \nabla_\theta \hat{J}_{\text{S2PG}}(\Theta) \right] \leq \frac{R^2 \tr(\Sigma^{-1}) (1-\gamma^T)^2T}{N(1-\gamma)^2} \left(F^2_d + H^2_d \frac{\tr(\Upsilon^{-1})}{\tr(\Sigma^{-1})} \right)\, . 
\end{align*}}
\label{lemma:tighter_bound_s2pg}
\end{lemma}
\begin{proof}
Indeed, when $B$ is diagonal it is easy to show that
{\small
\begin{equation}
\tr({AB}) = \diag(A)^\top \diag(B) \, . 
\label{eq:trace_property}
\end{equation}
}

Let $\diag(\Sigma)= \vsigma = [\sigma_0 , \dots, \sigma_{|a|}]^\top$, and $\diag(\Upsilon)= \vupsilon = [\upsilon_0 , \dots, \upsilon_{|z|}]^\top$. 
Using~\eqref{eq:trace_property} we can write
{\small
\begin{align*}
    \tr\left( \nabla_{\theta} \mu_{\theta}^a (s_t, z_t)^\top \nabla_{\theta} \mu_{\theta}^a (s_t, z_t)  \, \Sigma^{-1}\right) &= \diag\left(\nabla_{\theta} \mu_{\theta}^a (s_t, z_t)^\top \nabla_{\theta} \mu_{\theta}^a(s_t, z_t) \right)^\top\diag\left(\Sigma^{-1}\right)   \\
    \tr\left( \nabla_{\theta} \mu_{\theta}^z (s_t, z_t)^\top \nabla_{\theta} \mu_{\theta}^z (s_t, z_t) \Upsilon^{-1}\right) &= \diag\left(\nabla_{\theta} \mu_{\theta}^z (s_t, z_t)^\top \nabla_{\theta} \mu_{\theta}^z(s_t, z_t)\right)^\top \diag\left( \Upsilon^{-1}\right) \, . 
\end{align*}
}

Let $X^{[i]}$ be the $i$-th row of the matrix $X$.  We observe that

{\small
\begin{align*}
 \diag\left(\nabla_{\theta} \mu_{\theta}^a (s_t, z_t)^\top \nabla_{\theta} \mu_{\theta}^a(s_t, z_t) \right) &= 
 \left[\begin{array}{c}
     \nabla_{\theta} (\mu_{\theta}^a (s_t, z_t)^{[0]})^\top \nabla_{\theta} \mu_{\theta}^a(s_t, z_t)^{[0]}\\
     \dots \\
     \nabla_{\theta} (\mu_{\theta}^a (s_t, z_t)^{[|a|-1]})^\top \nabla_{\theta} \mu_{\theta}^a(s_t, z_t)^{[|a|-1]}\\
 \end{array} \right]  \\
 & = \left[ 
 \begin{array}{c}
      \lVert \nabla_{\theta} \mu_{\theta}^a (s_t, z_t)^{[0]} \rVert_2^2 \\
      \dots \\
      \lVert \nabla_{\theta} \mu_{\theta}^a (s_t, z_t)^{[|a|-1]} \rVert_2^2
 \end{array}
 \right].
\end{align*}
}
The same observation can be done by exchanging $\mu_{\theta}^a$ with $\mu_{\theta}^z$. Notice that the resulting vector is a vector of the norms of the gradient of the mean functions w.r.t. a specific action (or internal state component).  

If we assume the existence of two positive constants $F_d$ and $H_d$ such that $\lVert \nabla_{\theta} \mu_{\theta}^a (s_t, z_t)^{[i]} \rVert_2^2 \leq F^2_d$ and $\lVert \nabla_{\theta} \mu_{\theta}^z (s_t, z_t)^{[i]} \rVert_2^2 \leq H^2_d$, $\forall i, s_t, z_t$ then we can write

{\small
\begin{align}
    \tr\left( \nabla_{\theta} \mu_{\theta}^a (s_t, z_t)^\top \nabla_{\theta} \mu_{\theta}^a (s_t, z_t)  \, \Sigma^{-1}\right) &\leq F^2_d\cdot \left(\vone^\top \diag(\Sigma^{-1})\right) \nonumber\\
    & = F_d^2 \sum_{i=0}^{|a|-1}\dfrac{1}{\sigma_i} = F_d^2 \tr(\Sigma^{-1}) \nonumber\\
    \tr\left( \nabla_{\theta} \mu_{\theta}^z (s_t, z_t)^\top \nabla_{\theta} \mu_{\theta}^z (s_t, z_t) \Upsilon^{-1}\right) &\leq H^2_d\cdot \left(\vone^\top \diag\left( \Upsilon^{-1}\right)\right) \nonumber\\
    & = H_d^2 \sum_{i=0}^{|a|-1}\dfrac{1}{\upsilon_i} = H_d^2 \tr(\Upsilon^{-1}).
    \label{eq:diag_covariance_trace_bound}
\end{align}}

Using~\eqref{eq:diag_covariance_trace_bound} the bound of the variance of the single gradient estimator is
{\small
\begin{align*}
\var \left[ G(\tau) \tilde{f}(\tau) \right] \leq& \frac{R^2 (1-\gamma^T)^2}{(1-\gamma)^2}
  \biggl(\sum_{t=0}^{T-1} \expect{\tau}{
  \tr\left(  \nabla_{\theta} \mu_{\theta}^a (s_t, z_t)^\top \nabla_{\theta} \mu_{\theta}^a (s_t, z_t)  \, \Sigma^{-1}\right)
} \\
 & \phantom{ \frac{R^2 (1-\gamma^T)^2}{(1-\gamma)^2}}\qquad\,  + 
 \expect{\tau}{ \tr\left( \nabla_{\theta} \mu_{\theta}^z (s_t, z_t)^\top \nabla_{\theta} \mu_{\theta}^z (s_t, z_t) \Upsilon^{-1}\right) }\biggr) \\
 \leq &  \frac{R^2 (1-\gamma^T)^2}{(1-\gamma)^2} 
 \sum_{t=0}^{T-1}\left( F_d^2 \tr(\Sigma^{-1}) +   H_d^2 \tr(\Upsilon^{-1}) \right) \\
 = & \frac{R^2 (1-\gamma^T)^2T}{(1-\gamma)^2} \left(F^2_d\tr(\Sigma^{-1}) + H^2_d \tr(\Upsilon^{-1}) \right).
\end{align*}
}

Therefore, the gradient of the \gls{s2pg} estimator in the diagonal gaussian setting is
{\small
\begin{equation*}
    \var\left[ \nabla_\theta \hat{J}_{\text{S2PG}}(\Theta) \right] \leq \frac{R^2 \tr(\Sigma^{-1}) (1-\gamma^T)^2T}{N(1-\gamma)^2} \left(F^2_d + H^2_d \frac{\tr(\Upsilon^{-1})}{\tr(\Sigma^{-1})} \right)
\end{equation*}
}\noindent
concluding the proof.
\end{proof}
Notice that, when the action is a scalar, we get
{\small
\begin{equation*}
    \var\left[ \nabla_\theta \hat{J}_{\text{S2PG}}(\Theta) \right] \leq \frac{R^2 (1-\gamma^T)^2T}{N(1-\gamma)^2\sigma^2 } \left(F^2_d + H^2_d \tr(\Upsilon^{-1})\sigma^2  \right),
\end{equation*}\noindent
}
matching closely both the previous results of~\cite{papini2022smoothing}--- with the addition of the term depending on the internal state variance ---and the special scalar action case of the generic bound derived with a full covariance matrix.
\clearpage

\section{Qualitative Comparison of the Gradient Estimators}\label{app:gradient_comparison}
In this section, we analyze the difference between the stochastic gradient estimator for stateful policies and the \gls{bptt} approach in greater detail.

In \gls{bptt}, the gradient needs to be propagated back to the initial state. This requires the algorithm to store the history of execution up to the current timestep, if we want to compute an unbiased estimate of the policy gradient.
The common approach in the literature is to truncate the gradient propagation for a fixed history length. This approach is particularly well-suited if the recurrent policy does not require to remember long-term information.

Instead, our approach incorporates the learning of the policy state transitions into the value function, compensating for the increased variance of the policy gradient estimate (coming from the stochasticity of the policy state kernel) and allowing for the parallelization of the gradient computation. As shown in Figure~\ref{fig:grad_comp}, the stochastic gradient for stateful policies only uses local information at timestep $t$ to perform the update, while the \gls{bptt} uses the full history until the starting state $s_0$.
The choice between the two estimators is non-trivial: a high dimensionality of the policy state may cause the Q-function estimation problem challenging. However, in practical scenarios, tasks can be solved with a relatively low-dimensional policy state vector.
In contrast, a long trajectory implies multiple applications of the chain rule, that may consequently produce exploding or vanishing gradients, causing issues during the learning. This problem is not present when using stochastic stateful policies.
\vspace{0.5cm}

\begin{figure}[ht]
    \centering
    \includegraphics[width=\textwidth]{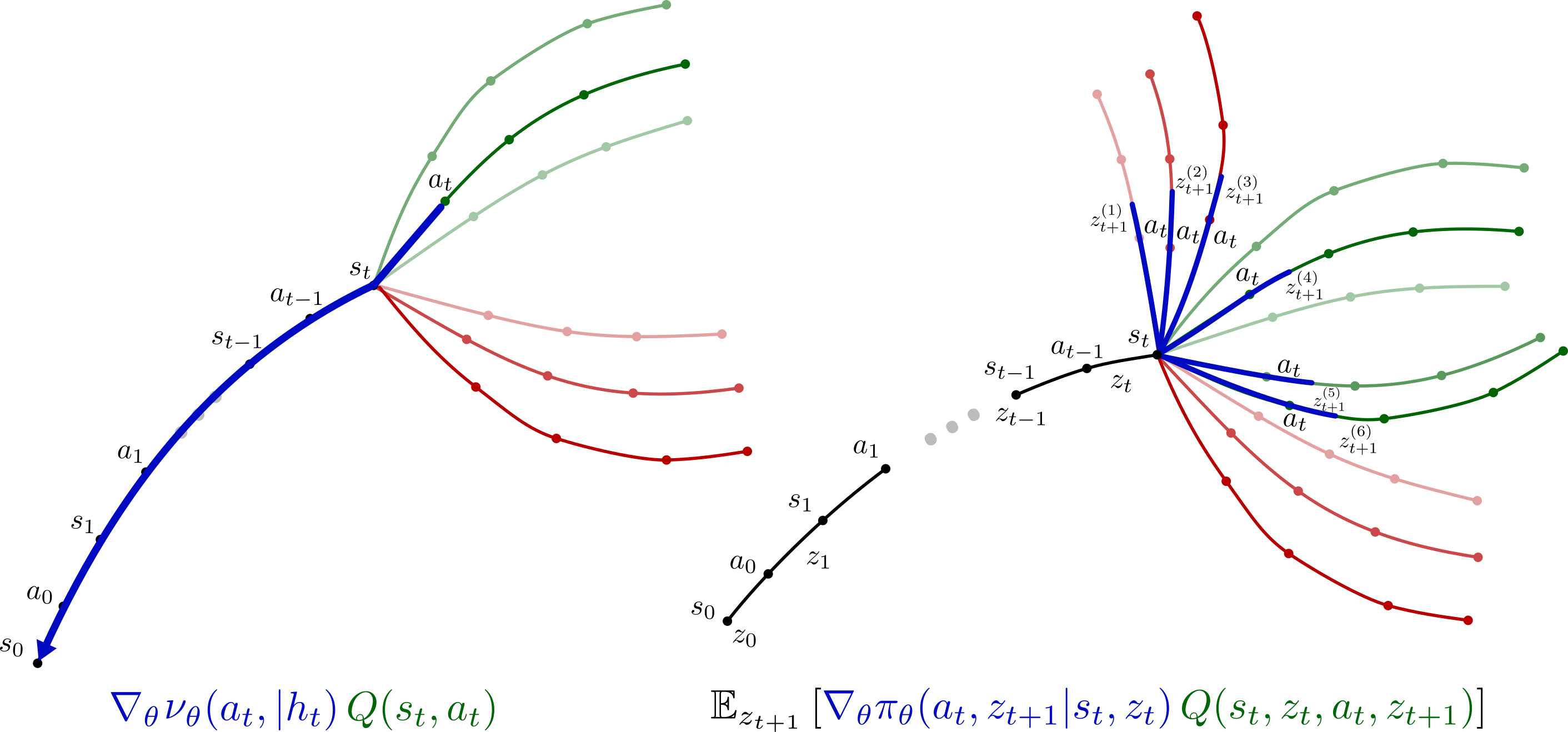}
    \caption{Comparison of the gradient at the state $s_t$ and action $a_t$ in \gls{bptt} --left -- and our stochastic gradient estimator -- right-- using the \gls{pgt}. The red and green colors represents the value of future state, where green means high $Q$-value and red means low $Q$-value.}
    \label{fig:grad_comp}
\end{figure}

\clearpage
\section{Algorithm Pseudocode}\label{app:algos}
In the following, we present the pseudocode for SAC-RS, TD3-RS and PPO-RS algorithms. The algorithms are a straightforward modification of the vanilla SAC, TD3, and PPO approaches to support the internal policy state.
In all cases, the algorithm receives in input a dataset of interaction (in our experiments one single step for SAC-RS and TD3-RS, while for PPO-RS is a batch of experience composed of multiple trajectories) coming from the current policy rollouts on the environment. In the pseudocode below we denote the parameter of the policy as $\vtheta$, the target policy parameters as $\check{\vtheta}$. Furthermore, we name the parameters of the i-th value function as $\vpsi_i$ and the corresponding target network $\check{\vpsi}_i$.
Furthermore, we use the square bracket notation to denote the component of a vector. In this context, we refer to the vector of states, actions, and internal states inside the dataset (seen as an ordered list).

\begin{algorithm}[hbt]
    \begin{algorithmic}[1]
        \Function{updateTD3}{$\mathcal{D}$}
            \State \Call{$M$.add}{$\mathcal{D}$} \Comment{Add the dataset of transitions $\mathcal{D}$ to the replay memory $M$}
            \If{$\Call{\text{size}}{M} > S_\text{min}$ } \Comment{Wait until the replay memory has $S_\text{min}$ transitions}
                \State Sample a minibatch $\mathcal{B}$ from $M$
                \State $\hat{q}\gets$\Call{computeTarget}{$\mathcal{B}$}
                \State \Call{$Q_{\vpsi_{i}}$.fit}{$\mathcal{B}, \hat{q}$}, for each $i\in\{0,1\}$
                \If{$\textrm{it} \bmod D = 0$} \Comment{Perform the policy update every $D$ iterations }
                    \State Compute the loss
                    {\small
                    \begin{equation*}
                        L_{\vtheta}(\mathcal{B}) = -\dfrac{1}{N}\sum_{s\in\mathcal{B}}Q_{\vpsi_{0}}(s, z, \mu^a_{\vtheta}(s), \mu^z_{\vtheta}(s))
                    \end{equation*}}
                    \State $\vtheta \gets \Call{optimize}{L_{\vtheta}(\mathcal{B}), \vtheta}$
                \EndIf
            
                \State $\check{\theta}\gets \tau \vtheta + (1-\tau)\check{\vtheta}$
                \State $\check{\vpsi}_i\gets \tau \vpsi_i + (1-\tau)\check{\vpsi}_i$
                \State $\textrm{it}\gets\textrm{it}+1$ \Comment{Update the iteration counter}
             \EndIf
            \State\Return $\vtheta$
        \EndFunction
        \Statex
        \Function{computeTarget}{$\mathcal{B}$}
            \For{$(s, z, a, r, z', s') \in \mathcal{B}$}
                \If{$s'$ is absorbing}
                    \State $v_\text{next}(s',z')\gets 0$
                \Else
                    \State $a', z''\gets \mu_{\check{\vtheta}}(s', z')$
                    \State Sample $\epsilon^a, \epsilon^z\sim\mathcal{N}(0, \sigma_\epsilon)$
                    \State $\epsilon^a_\text{clp} \gets$ \Call{clip}{$\epsilon^a,-\delta_{\epsilon},\delta{\epsilon}$}
                    \State $\epsilon^z_\text{clp} \gets$ \Call{clip}{$\epsilon^z,-\delta_{\epsilon},\delta{\epsilon}$}
                    \State $a_\text{smt}\gets a'+\epsilon^a_\text{clp}$
                    \State $z_\text{smt}\gets z''+\epsilon^z_\text{clp}$
                    \State $a_\text{clp}\gets$\Call{clip}{$a_\text{smt},a_\text{min},a_\text{max}$}
                    \State $z_\text{clp}\gets$\Call{clip}{$z_\text{smt},z_\text{min},z_\text{max}$}
                    \State $v_\text{next}(s',z')\gets \min_{i}Q_{\check{\vpsi}_{i}}(s', z', a_\text{clp}, z_\text{clp})$
                \EndIf
                    
                \State $\hat{q}(s, z, a, z') \gets r + \gamma v_\text{next}(s',z')$
            \EndFor
            \State\Return $\hat{q}$
        \EndFunction
    \end{algorithmic}
    \caption{TD3-RS}
    \label{alg:td3}
\end{algorithm}

\begin{algorithm}[ht]
    \caption{SAC-RS}
    \begin{algorithmic}[1]
        \Function{updateSAC}{$\mathcal{D}$}
            \State \Call{$M$.add}{$\mathcal{D}$}
            \If{$\Call{\text{size}}{M} > S_\text{min}$ } 
                \State Sample a minibatch $\mathcal{B}$ from $M$
                
                \State $\hat{q}\gets$\Call{computeSoftTarget}{$\mathcal{B}$}
                
                \If{$\Call{\text{size}}{M} > S_\text{warm}$} \Comment{Perform the policy update after $S_\text{warm}$ samples}
                    \State Sample $a',z'\sim\pi_\vtheta(\cdot|s,z)\, \,  \forall (s,z)\in\mathcal{B}$
                    \State Compute the policy loss
                     {\small
                     \begin{equation*}
                         L_{\vtheta}(\mathcal{B}) = \alpha^a\log\pi^a_\vtheta(a'|s, z) + \alpha^z\log\pi^z_\vtheta(z'|s, z) -\dfrac{1}{N}\sum_{s\in\mathcal{B}}
                         \min_{i}Q_{\vpsi_i}(s, z, a', z')
                     \end{equation*}}
                     \State $\vtheta \gets \Call{optimize}{L_{\vtheta}(\mathcal{B}), \vtheta}$
                     \State Compute the $\alpha^a$ and $\alpha^z$ losses with target entropies $\bar{\mathcal{H}^a}$ and $\bar{\mathcal{H}^z}$
                     {\small
                     \begin{equation*}
                         L_{\alpha^a}(\mathcal{B}) = -\dfrac{1}{N}\sum_{(s,z,a)\in\mathcal{B}} \alpha^a \left(\log\pi^a_\vtheta(a|s,z)+\bar{\mathcal{H}^a}\right)
                     \end{equation*}}
                     {\small
                     \begin{equation*}
                         L_{\alpha^z}(\mathcal{B}) = -\dfrac{1}{N}\sum_{(s,z,z')\in\mathcal{B}} \alpha^z \left(\log\pi^z_\vtheta(z'|s,z)+\bar{\mathcal{H}^z}\right)
                     \end{equation*}}
                     \State $\alpha^a\gets\Call{optimize}{L_{\alpha^a}(\mathcal{B}), \alpha}$  
                     \State $\alpha^z\gets\Call{optimize}{L_{\alpha^z}(\mathcal{B}), \alpha}$  
                \EndIf
                \State $\hat{q}\gets$\Call{computeSoftTarget}{$\mathcal{B}$}
                \State \Call{$Q_{\vpsi}$.fit}{$\mathcal{B}, \hat{q}$}
                \State $\check{\vpsi}_i\gets \tau \vpsi_i + (1-\tau)\check{\vpsi}_i$  
            \EndIf
            \State\Return $\vtheta$
        \EndFunction
        \Statex
        \Function{computeSoftTarget}{$\mathcal{B}$}
            \For{$(s, z, r, a',  z', s') \in \mathcal{B}$}
                    \If{$s'$ is absorbing}
                        \State $v_\text{next}(s',z')\gets 0$
                    \Else
                        \State Sample $a',z''\sim\pi_\vtheta(\cdot|s')$
                        \State $h_\text{bonus} \gets  - \alpha^a \log\pi_\vtheta(a'|s',z') - \alpha^z \log\pi_\vtheta(z''|s',z')$
                        \State $v_\text{next}(s',z')\gets \min_i Q_{\check{\psi}_i}(s', z', a', z'') + h_\text{bonus}$ 
                    \EndIf
                        
                    \State $\hat{q}(s,z,a,z') \gets r + \gamma v_\text{next}(s',z')$
                \EndFor
                \State\Return $\hat{q}$
        \EndFunction
    \end{algorithmic}
\end{algorithm}

\begin{algorithm}[ht]
\label{alg:ppo}
\caption{PPO-RS}
\begin{algorithmic}[1]
    \Function{updatePPO}{$\vtheta$, $\mathcal{D}$}
        \State $V', A\gets$\Call{computeGAE}{$V_{\vpsi},\mathcal{D}$}
        
        \State \Call{$V_{\vpsi}$.fit}{$\mathcal{D}, V'$} \Comment{Update the value function}
        
        \For{$i \gets 0$ to $N$}
            \State Split the dataset $\mathcal{D}$ in $K$ minibatches $\lbrace\mathcal{B}_k \vert k\in[0,K-1]\rbrace$
             \For{$k \gets 0$ to $K$}
                 \State compute the surrogate loss on the minibatch
                 {\small
                 \begin{align*}
                     L_{\vtheta}(\mathcal{B}_k) =  \dfrac{1}{N}\sum_{(s,z,a,z')\in\mathcal{B}_k}\min&\left(\dfrac{\pi_{\vtheta}(a,z'|s,z)}{q(a,z'|s,z)}A(s,z,a,z'),\right. \\
                     & \left.\quad \textrm{clip}\left(\dfrac{\pi_{\vtheta}(a,z'| s,z)}{q(a,z'|s,z)},1-\varepsilon,1+\varepsilon\right) A(s,z,a,z')\right)
                 \end{align*}}
                \State $\vtheta \gets \Call{optimize}{L_{\vtheta}(\mathcal{B}_k), \vtheta}$
             \EndFor
        \EndFor
        \State\Return $\vtheta$
    \EndFunction
    \Statex
    \Function{computeGAE}{$V_{\vpsi},\mathcal{D}$}
        \For{$k\gets 0 \dots \text{len}(\mathcal{D})$}
            \State $\cmp{v}{k}\gets V_{\psi}(\cmp{s}{k},\cmp{z}{k})$            
            \State $\cmp{v_\text{next}}{k}\gets V_{\psi}(\cmp{s'}{k},\cmp{z'}{k})$       
        \EndFor
        \For{$k_\text{rev}\gets 0 \dots \text{len}(\mathcal{D})$}
            \State $k\gets \text{len}(\mathcal{D}) - k_\text{rev} -1$
            \If{$\cmp{s'}{k}$ \textbf{is} last}
                \If{$\cmp{s'}{k}$ \textbf{is} absorbing}
                \State $A(\cmp{s}{k},\cmp{z}{k},\cmp{a}{k},\cmp{z'}{k}) \gets \cmp{r}{k} - \cmp{v}{k}$
                \Else
                    \State $A(\cmp{s}{k},\cmp{z}{k},\cmp{a}{k},\cmp{z'}{k})  \gets \cmp{r}{k} + \gamma \cmp{v_\text{next}}{k} - \cmp{v}{k}$
                \EndIf
            \Else
            \State $\delta \gets A(\cmp{s}{k+1},\cmp{z}{k+1},\cmp{a}{k+1},\cmp{z'}{k+1})$
            \State $A(\cmp{s}{k},\cmp{z}{k},\cmp{a}{k},\cmp{z'}{k}) \gets \cmp{r}{k} + \gamma \cmp{v_\text{next}}{k} - \cmp{v}{k} + \gamma \lambda \delta$
            \EndIf
        \EndFor
        \State \Return $v, A$
    \EndFunction
    
\end{algorithmic}
\end{algorithm}

\clearpage

\section{Network Structures}\label{app:net_archi} 
In this section, we describe the structure of the networks used in this paper. Figure \ref{fig:network_architectures} presents all networks used for the policies and the critics in the \gls{rl} setting. For the \gls{bptt}, the policy networks were used, yet the red paths were dropped. Also, the critics did not use the hidden states as inputs. For the \gls{il} setting, the networks from \gls{sac} are used for \gls{lsiq}, and the networks from \gls{ppo} are used for \gls{gail}. For the recurrent networks, we generally use \glspl{gru} even though any other recurrent network, such as \glspl{lstm}, can also be used. 

\vspace{3cm}
\begin{figure}[h]
    \centering
    \includegraphics[width=\textwidth]{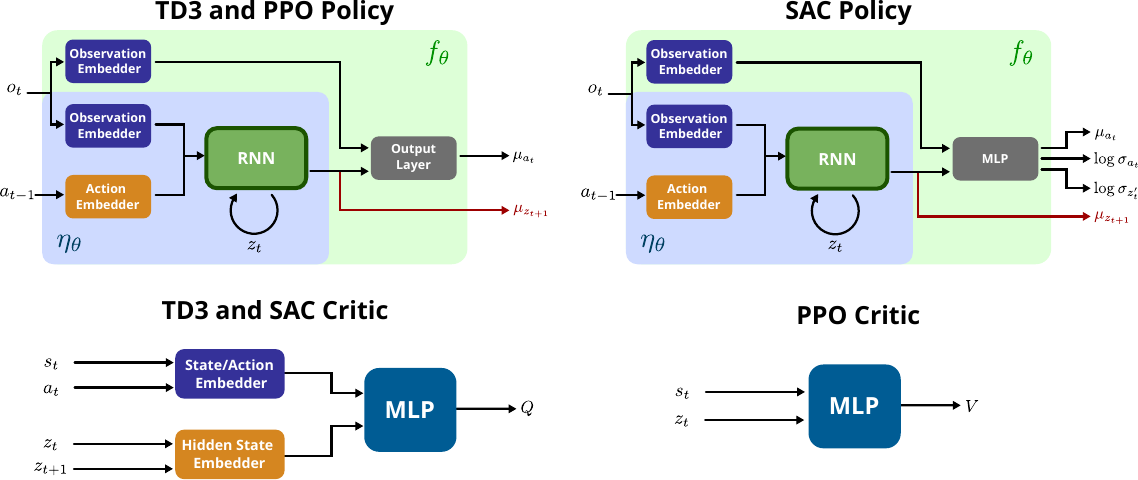}

    \caption{Network architectures used for the policies and critics in PPO-RS, TD3-RS, SAC-RS, PPO-BPTT, TD3-BPTT, and SAC-BPTT. For the \gls{bptt} variants, the red paths do not exist. For the \gls{s2pg} approach, the internal state is sampled from the noisy distribution. These architectures of the policies were originally used in \cite{ni2022recurrent}.}
    \label{fig:network_architectures}
\end{figure}

\clearpage
\section{Additional Experiments}\label{app:add_results}

\subsection{Memory Task}\label{app:mem_task}
These tasks are about moving a point mass in a 2D environment. They are shown in Figure \ref{fig:point_mass} and \ref{fig:point_mass_door}. For the first task, a random initial state and a goal are sampled at the beginning of each episode. The agent reads the desired goal information while it is close to the starting position. Once the agent is sufficiently far from the starting position, the information about the goal position is zeroed-out. The task consists of learning a policy that remembers the target goal even after leaving the starting area. As the distance to the goal involves multiple steps, long-term memory capabilities are required. Similarly, the second task hides the position of two randomly sampled doors in a maze once the agent leaves the observable area. Once the agents touches the black wall, an absorbing state is reached, and the environment is reset. We test our approach in the privileged information setting using \gls{sac}. We compare our approach with \gls{sac}-\gls{bptt} with a truncation length of 5 and 10. For the first task, our approach -- \gls{sac}-RS -- outperforms \gls{bptt}, while it performs slightly weaker compared to \gls{bptt} with a truncation length of 10 on the second task. We expect that the reason for the drop in performance is caused by the additional variance of our method in combination with the additional absorbing states in the environment. The additional variance in our policy leads to increased exploration, which increases the chances of touching the wall and reaching an absorbing state. Nonetheless, our approach has a significantly shorter training time, analogously to the results shown in Figure \ref{fig:overall_results}, for both tasks.

\begin{figure}
    \centering
    
    \begin{multicols}{3}
        \null \vfill
        
        \includegraphics[width=0.8\linewidth]{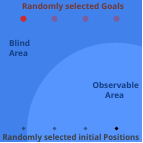}
        
        \columnbreak
        \null \vfill
        
        \includegraphics[width=\linewidth]{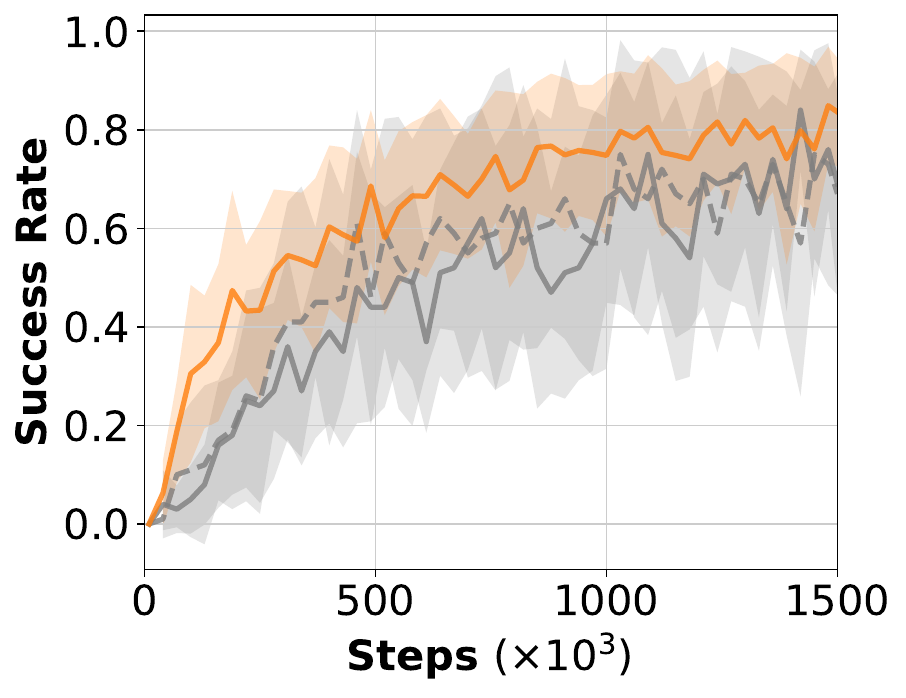}
        
        \columnbreak
        \null \vspace{0.2cm}
        
        \caption{Point mass experiments. In the blind area, the goal is not visible. The task involves remembering its position. Initial states, goals, and blind areas are randomly sampled. SAC-RS is a recurrent-stochastic version of \gls{sac}, using \gls{s2pg}.}
        \label{fig:point_mass}
    \end{multicols}
    
        \includegraphics[scale=0.385]{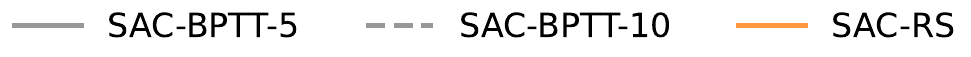}
\end{figure}

\begin{figure}
    \centering
    
    \begin{multicols}{3}
        \null \vfill
        
        \includegraphics[width=0.8\linewidth]{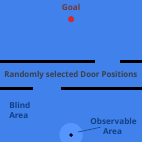}
        
        \columnbreak
        \null \vfill
        
        \includegraphics[width=\linewidth]{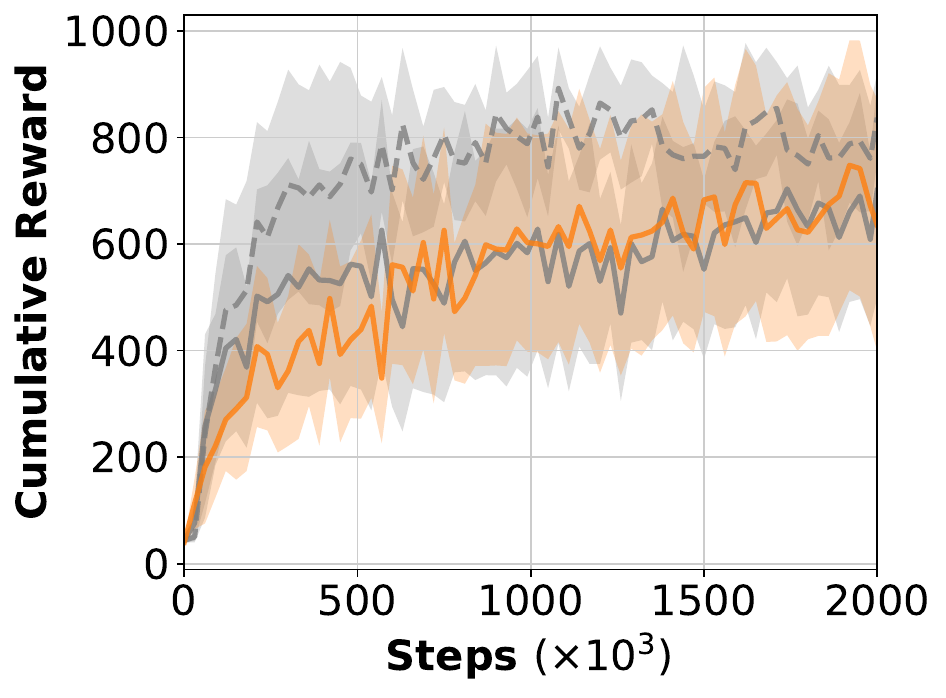}
        
        \columnbreak
        \null \vspace{0.2cm}
        
        \caption{Point mass door experiments. In the blind area, the positions of the two doors are not visible. The task involves remembering its position, which are randomly sampled. SAC-RS is a recurrent-stochastic version of \gls{sac}, using \gls{s2pg}.}
        \label{fig:point_mass_door}
    \end{multicols}
    
        \includegraphics[scale=0.385]{results/point_mass/legend.pdf}
\end{figure}

\subsection{Stateful Policies as Inductive Biases}\label{sec:cpg_exp}

Adding inductive biases into the policy is a common way to impose prior knowledge into the policy \cite{ijspeert2007swimming,bellegarda2022cpg, alhafez2021, atacom}. To show that stateful policies are of general interest in \gls{rl} -- not only for \glspl{pomdp}
-- we show in this section that they can be used to elegantly impose an arbitrary inductive bias directly into the policy. In contrast to commonly used biases, our approach allows learning the parameters of the inductive bias as well.
We conduct experiments with a policy that includes oscillators -– coupled Central Pattern Generators (CPG) -- to impose an oscillation into the policy. The latter is common for locomotion tasks \cite{miki2022learning}. We model the oscillators as a set of ordinary differential equations and simulate the latter using the explicit Euler integrator. Figure \ref{fig:cpg-exp} presents the results. Using our algorithm, we are able to train the parameters of the CPGs and an the additional network to learn simple locomotion skills on HalfCheetah and Ant.
These policies solely rely on their internal state and do not use the environment state making them blind. With this simple experiment, we want to
show that we can impose arbitrary dynamics on the policy and train the latter using our approach. Note that plain oscillators do not constitute a very good inductive bias. More sophisticated dynamical models that also take the environment state into account are needed to achieve better results. 

\paragraph{Related Work.} \citet{shi2022} train \glspl{cpg} using a black-box optimizer instead of \gls{bptt}, while simultaneously learning an \gls{rl} policy for residual control. \citet{cho2019} optimize a \gls{cpg} policy using \gls{rl} without \gls{bptt} by exposing the internal state to the environment but do not provide a theoretical derivation of the resulting policy gradient.  \citet{campanaro2021cpg} train a \gls{cpg} policy using \gls{ppo} without \gls{bptt}. However, the authors neglect the influence of the internal state on the value function. 

\begin{figure}
\centering
\includegraphics[width=0.85\textwidth]{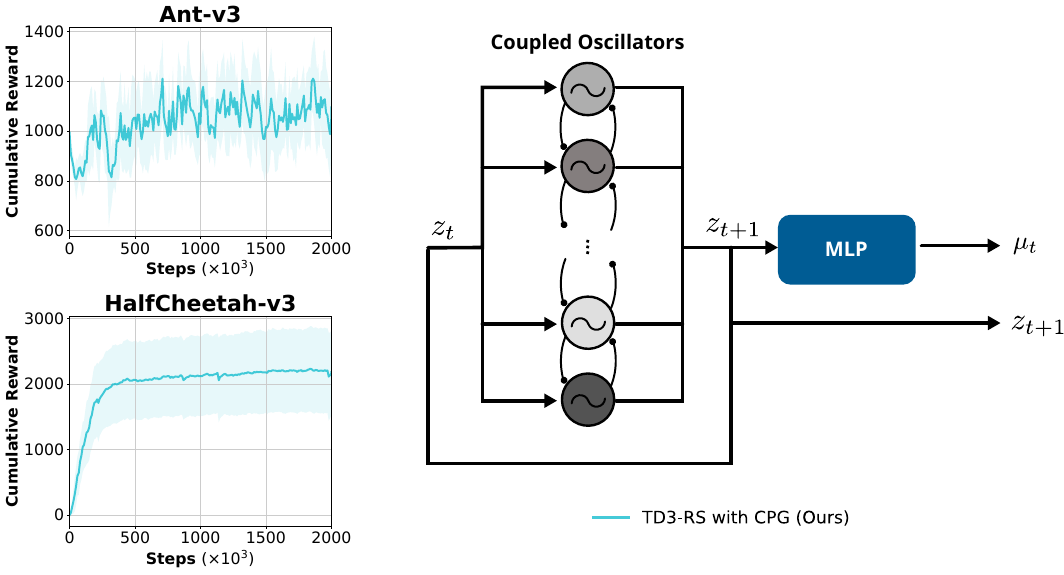} 
    \caption{Experiments in which a set of coupled central pattern generators are used as a policy. The left side shows the cumulative reward on HalfCheetah and Ant. The right side shows the structure of the policy. Note that the policy does not use the environment state making it blind.}
    \label{fig:cpg-exp}
\end{figure}

\subsection{Comparison on Standard POMDP Benchmarks}\label{app:ppo_res}
In the following, we evaluate the two gradient estimators in standard MuJoCo \gls{pomdp} benchmarks, in the on-policy setting. These tasks are taken from \cite{ni2022recurrent} and hide the velocity.
All experiments use the \gls{ppo} algorithm as a learning method. Figure \ref{fig:ppo_results_steps_all} presents the results. 
Our approach outperforms \gls{bptt} in all benchmarks in terms of the number of used samples and time.
Unfortunately, using on-policy approaches on these benchmarks, we do not achieve satisfactory performances in all environments (e.g. Hopper and Walker). This is probably due to the increased difficulty of exploring with partial observation, rather than an issue in gradient estimation. Indeed, in the other tasks, where more information is available, or the exploration is less problematic -- e.g., the failure state is more difficult to reach -- we outperform the baseline. It should be noted that the performance gain in terms of computation time is not massive: this is due to the reduced number of gradient computations in the on-policy scenario, which is performed in batch after many environments steps.

\begin{figure}[tb]
    \begin{multicols}{3}
    \includegraphics[scale=0.29]{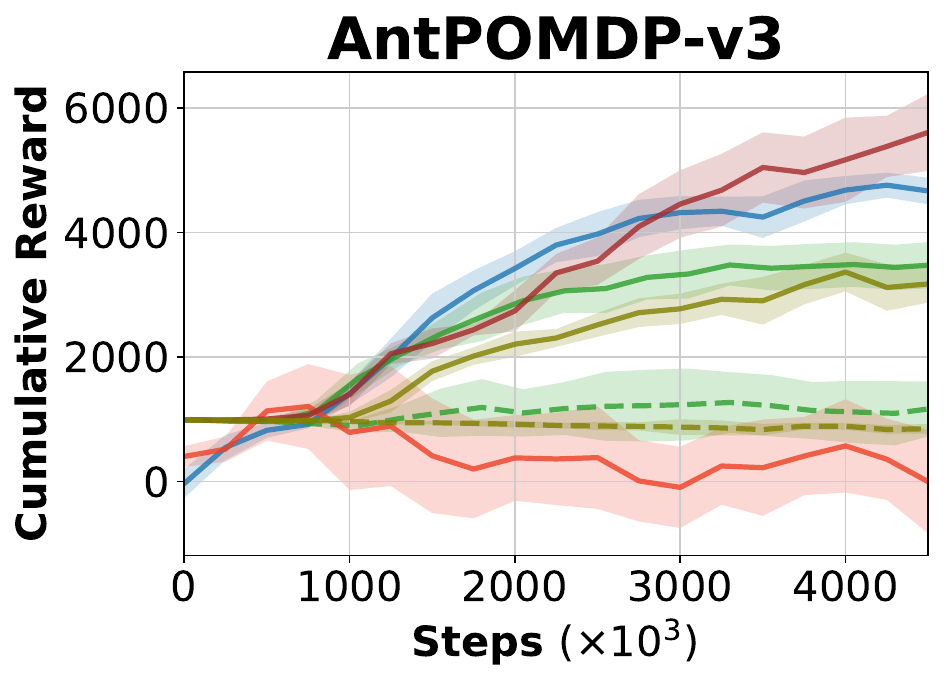}\columnbreak
    \includegraphics[scale=0.29]{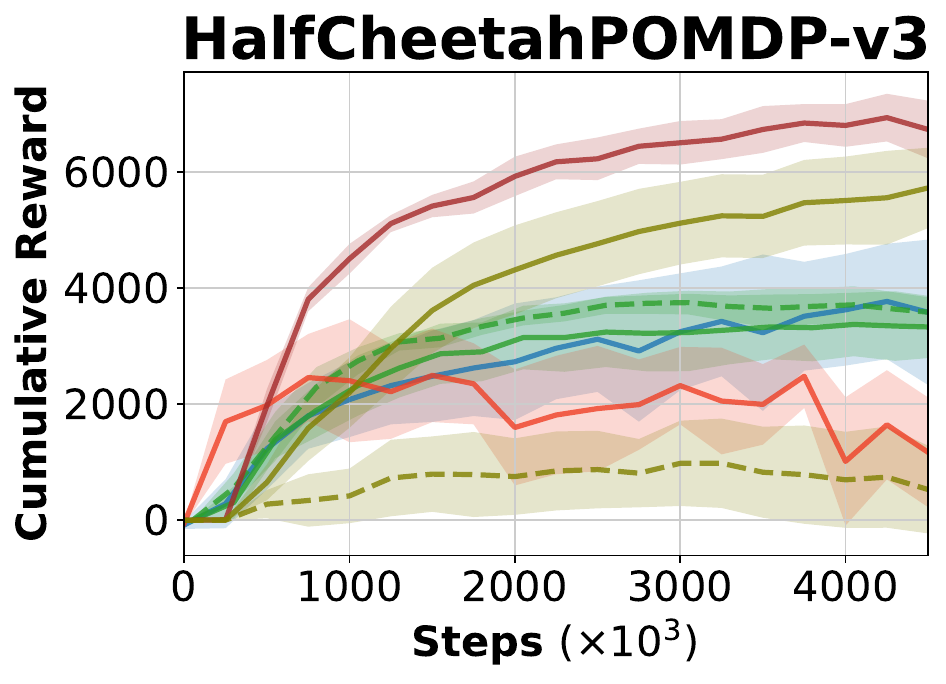}\columnbreak
    \includegraphics[scale=0.29]{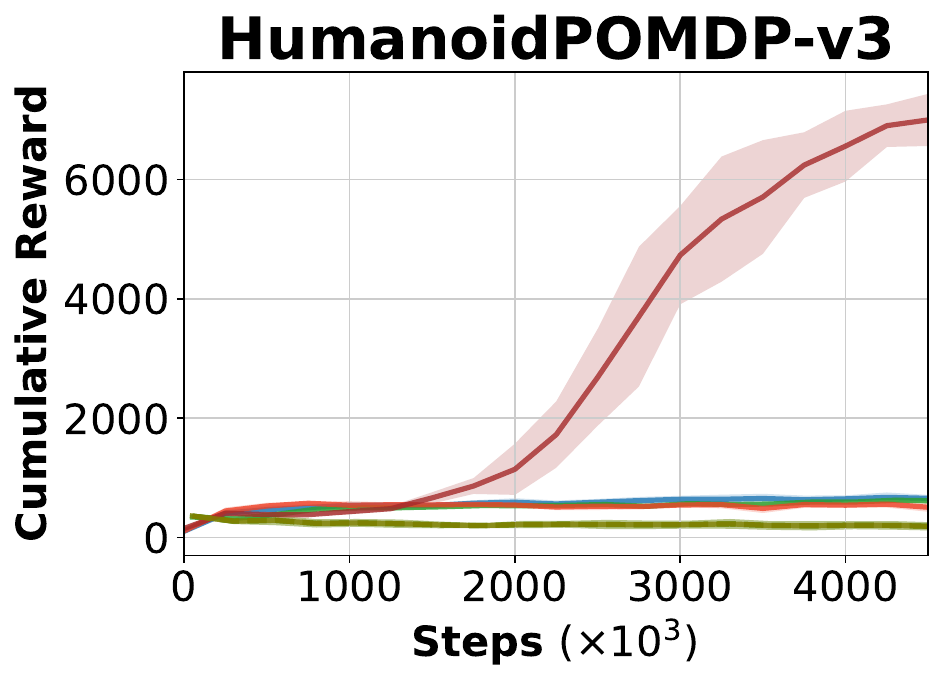}
    \end{multicols}
    \begin{multicols}{3}
    \includegraphics[scale=0.29]{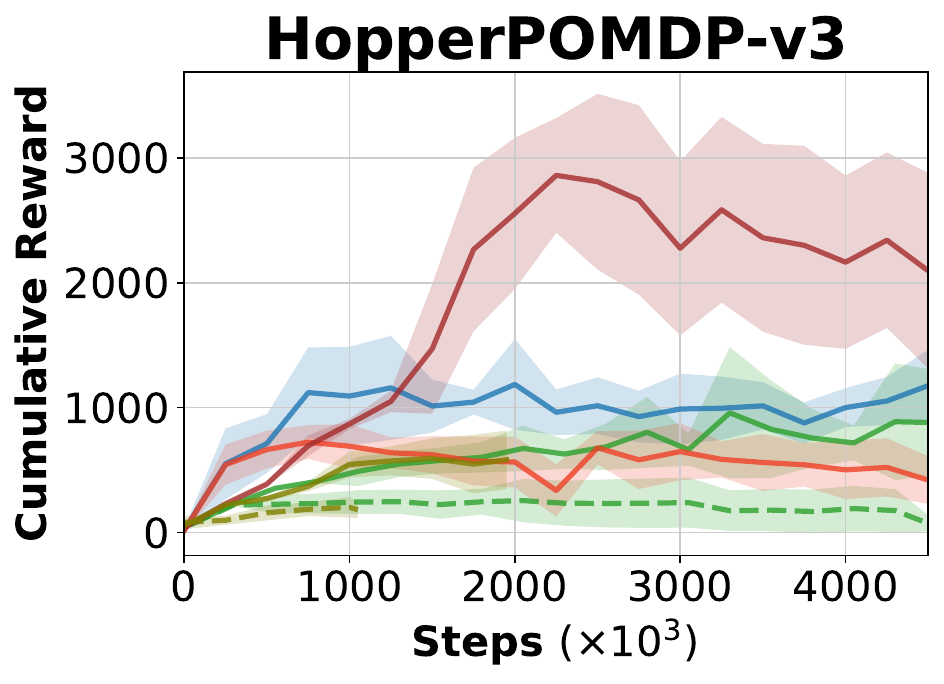}\columnbreak
    \includegraphics[scale=0.29]{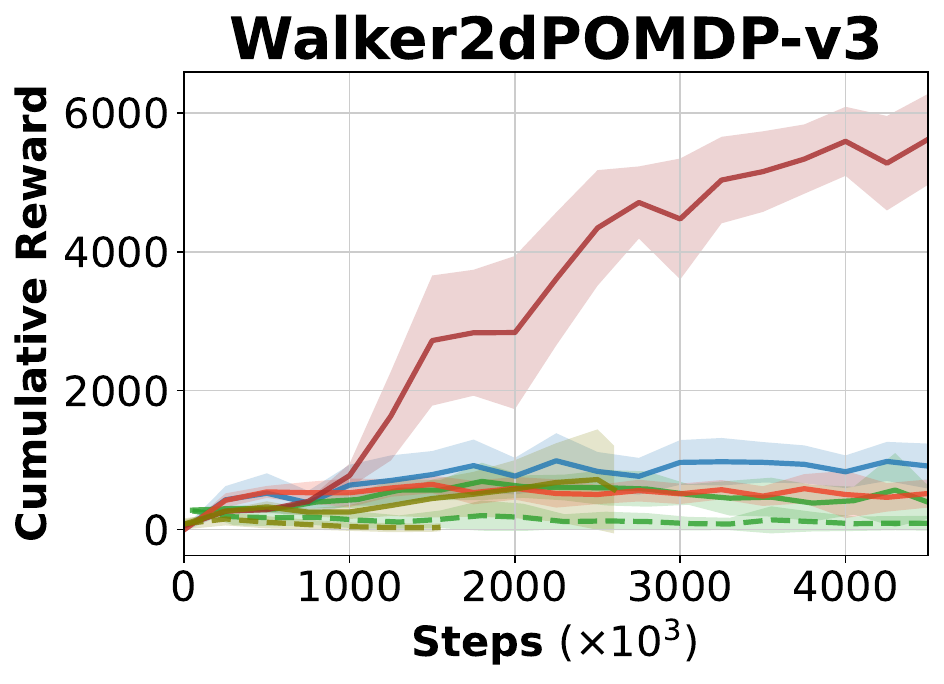}\columnbreak
    \end{multicols}
    \vspace{-15pt}
    \begin{center}
        \includegraphics[scale=0.3]{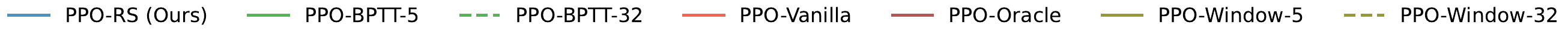}
    \end{center}
    \vspace{-10pt}
    \caption{Comparison of \gls{ppo} with the proposed stochastic policy state kernel and \gls{ppo} with \gls{bptt} on fully partially observable MuJoCo Gym Tasks. The number behind the \gls{bptt} implementations indicates the truncation length. Abscissa shows the cumulative reward. Ordinate shows the number of training steps ($\times 10^3$).}
    \label{fig:ppo_results_steps_all}
\end{figure}

\clearpage

\subsection{Comparison on Robustness Benchmarks with Privilege Information}\label{app:rand_dyn_rl_tasks}
In this section, we evaluate the performance of \gls{s2pg} against \gls{bptt} on the \gls{pomdp} setting with additional randomized masses. For the latter, masses of each link in a model are randomly sampled in a $\pm 30\%$ range. Figure \ref{fig:exem_fig_rand_dyn} shows examples of randomly sampled Ant and Humanoid environments. Here, we consider the scenario where the critic has privileged information (knowledge of velocities and mass distributions), while the policy uses only partial observability.

Figure \ref{fig:rl_results_mass_steps} presents the reward plots. Figure \ref{fig:overall_results} also presents the average runtime of each approach used for these tasks. 
The experiment run for a maximum of two million steps and four days of computation. The results do not show an approach clearly outperforming the others in all tasks. \gls{s2pg} based approaches seem to work more robustly in high-dimensional tasks, such as AntPOMDP-v3 and HumanoidPOMDP-v3. Curiously, the \gls{bptt} version of \gls{td3} is not able to learn in the AntPOMDP-v3 task. In general, \gls{s2pg} seems to struggle in the HalfCheetahPOMDP-v3,  HopperPOMDP-v3 and in WalkerPOMDP-v3 tasks. To better compare the performance of all approaches on the low-dimensional tasks -- Hopper, Walker and HalfCheetah -- and the high-dimensional tasks -- Humanoid and Ant --, we presents the averaged performance plots in Figure \ref{fig:overall_results_low_high_dim}. As can be seen, our approach performs significantly better on the high-dimensional tasks while being worse on the low-dimensional ones. We believe that the reason for the worse performance of \gls{bptt} on high-dimensional tasks could be the potential risk of exploding gradients, which in turn cause exploding variance (c.f., Section \ref{sec:var_analysis}). In contrast, we believe that the worse performance of our method on low-dimensional tasks can be traced back to the increased amount of variance in the policy in conjunction with the presence of absorbing states in the case of Hopper and Walker, similarily to the memory tasks.

\begin{figure}
    \centering
    \includegraphics[width=0.7\textwidth]{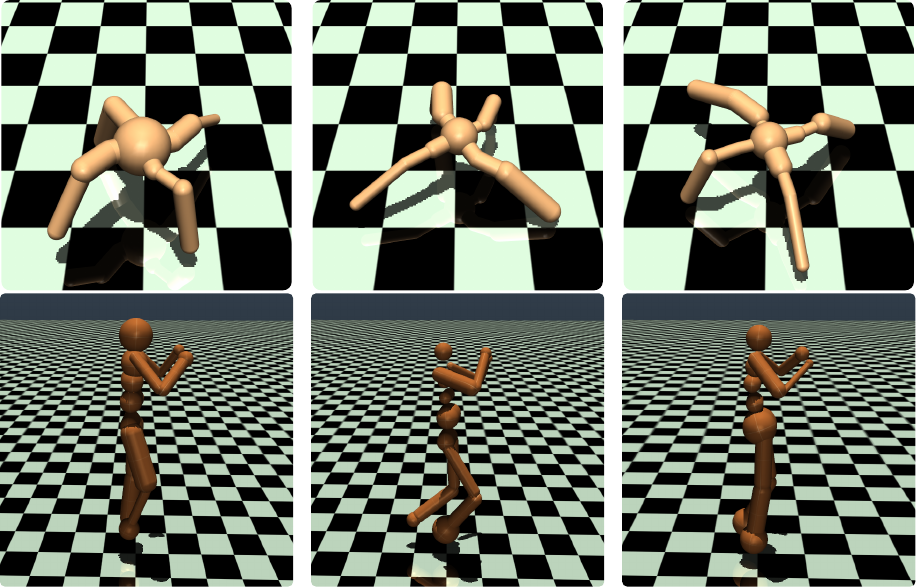}
    \caption{Example images of the Mujoco Gym randomized mass environment used within this work. The upper row shows three randomly sampled Ant environments, while the lower row show three Humanoid environments.}
    \label{fig:exem_fig_rand_dyn}
\end{figure}


\begin{figure}[bt]
    \begin{multicols}{3}
    \includegraphics[scale=0.275]{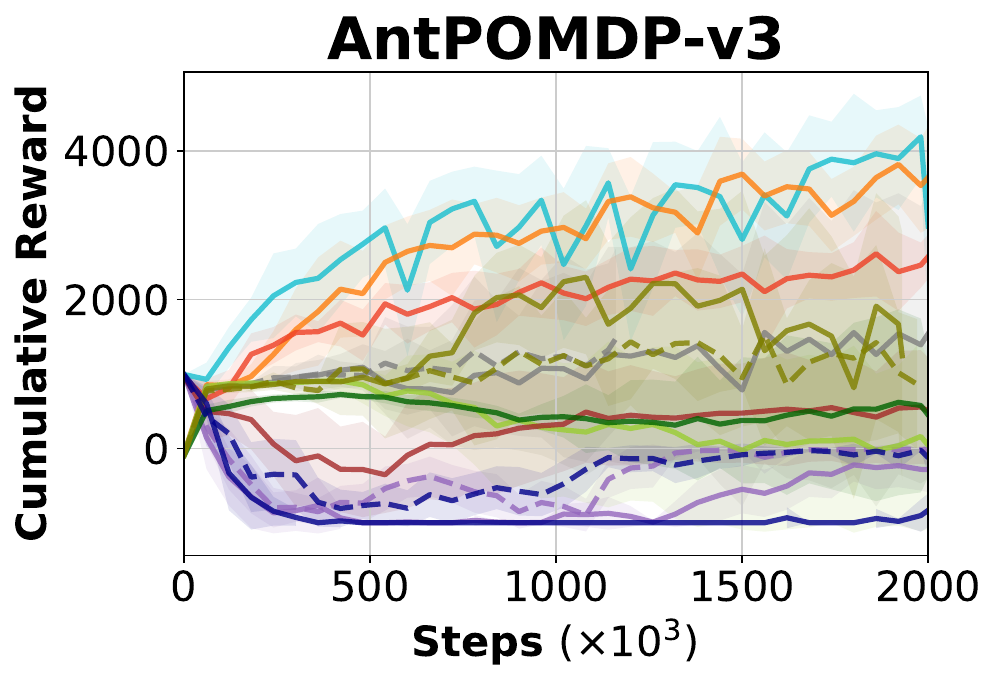}\columnbreak
    \includegraphics[scale=0.275]{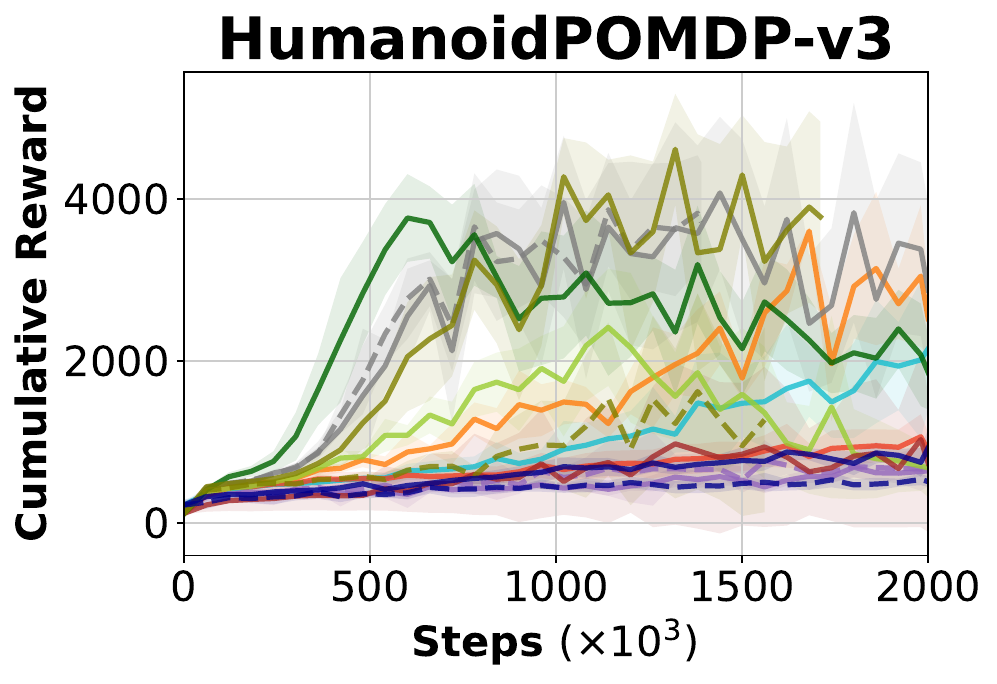}\columnbreak
    \includegraphics[scale=0.275]{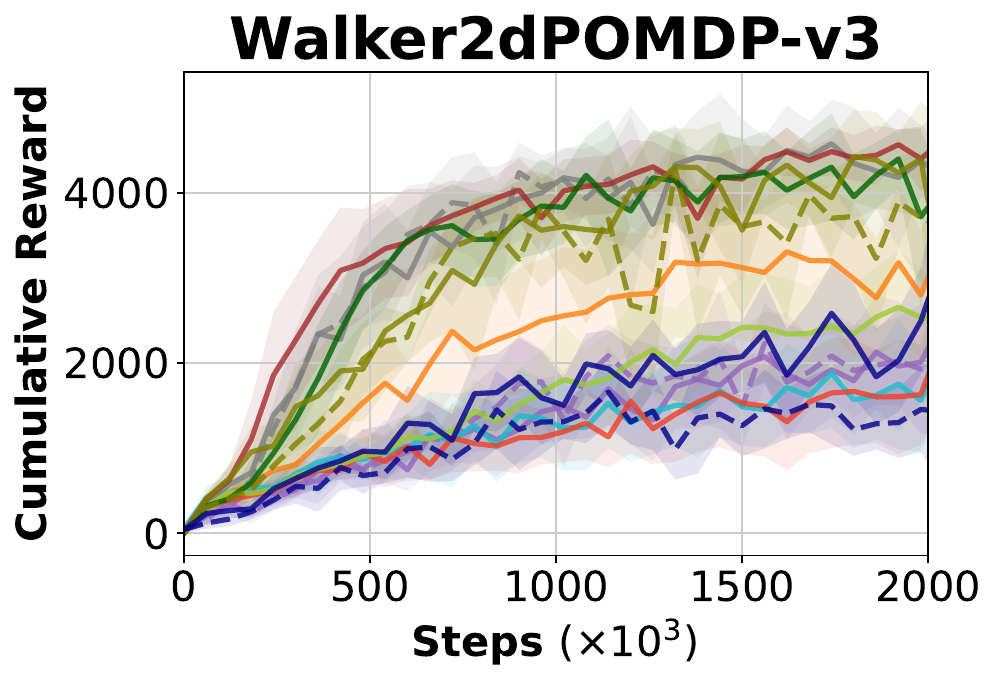}
    \end{multicols}
    \begin{multicols}{3}
    \includegraphics[scale=0.275]{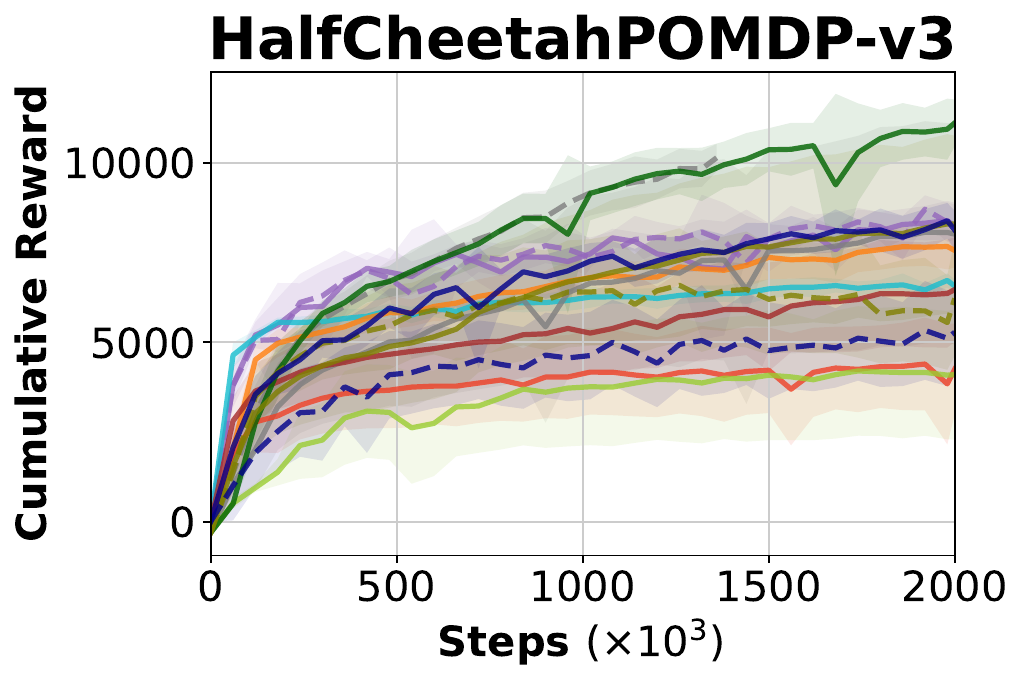}\columnbreak
    \includegraphics[scale=0.275]{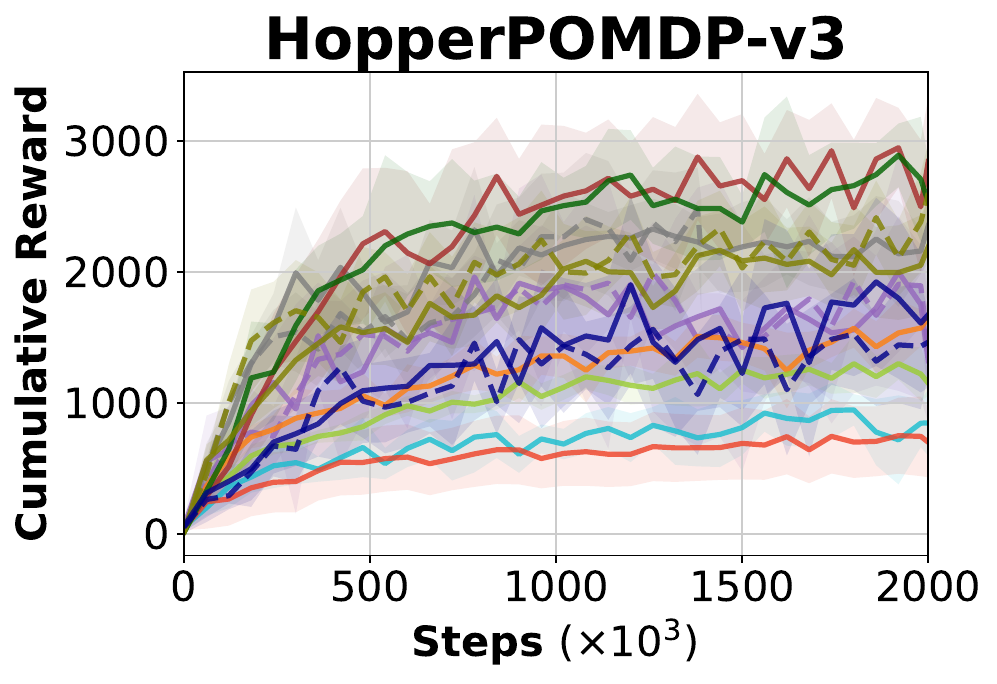}
    \end{multicols}
    \vspace{-15pt}
    \begin{center}
        \includegraphics[scale=0.3]{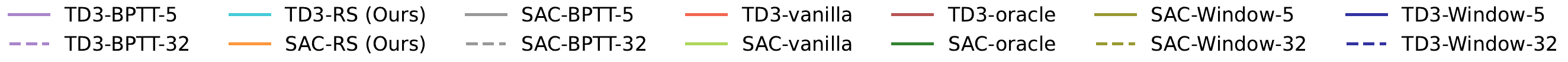}
    \end{center}
    \vspace{-10pt}
    \caption{\gls{sac} and \gls{td3} results on the MuJoCo Gym Tasks with partially observable policies and privileged critics. Additionally, masses of each link are randomly sampled at the beginning of each episode. The number behind the \gls{bptt} implementations indicates the truncation length. Abscissa shows the cumulative reward. Ordinate shows the training time in hours.}
    \label{fig:rl_results_mass_steps}
\end{figure}

\begin{figure}[t]
    \begin{center}
    \begin{tabular}{ c c }
    \hspace{2cm}\textbf{Privileged Information}\hspace{1cm} & \hspace{2cm} \textbf{No Privileged Information}   \\ 
    \end{tabular}
    \vspace{-5pt}
    \end{center} 
     \includegraphics[width=\textwidth]{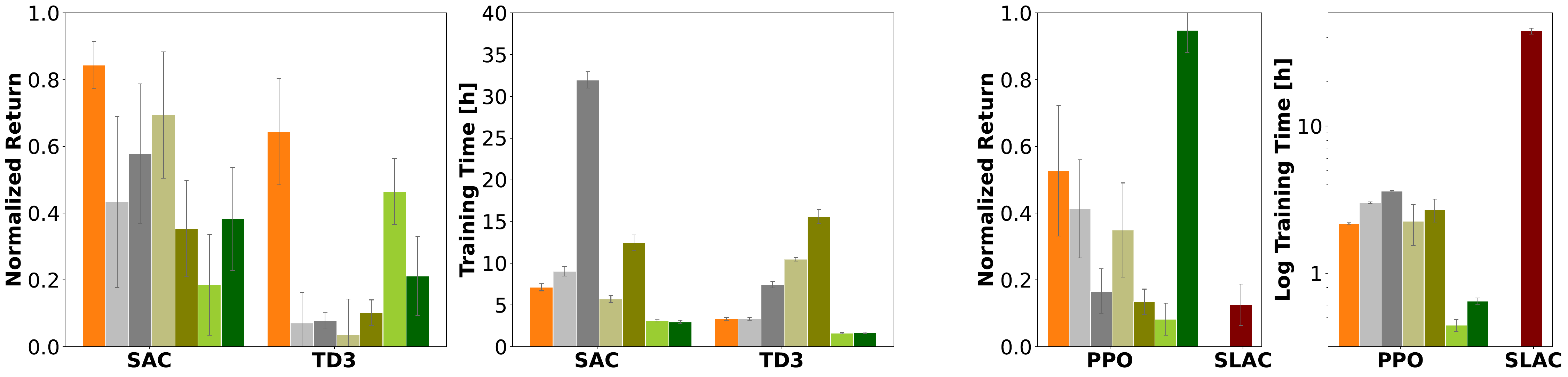}
     \includegraphics[width=\textwidth]{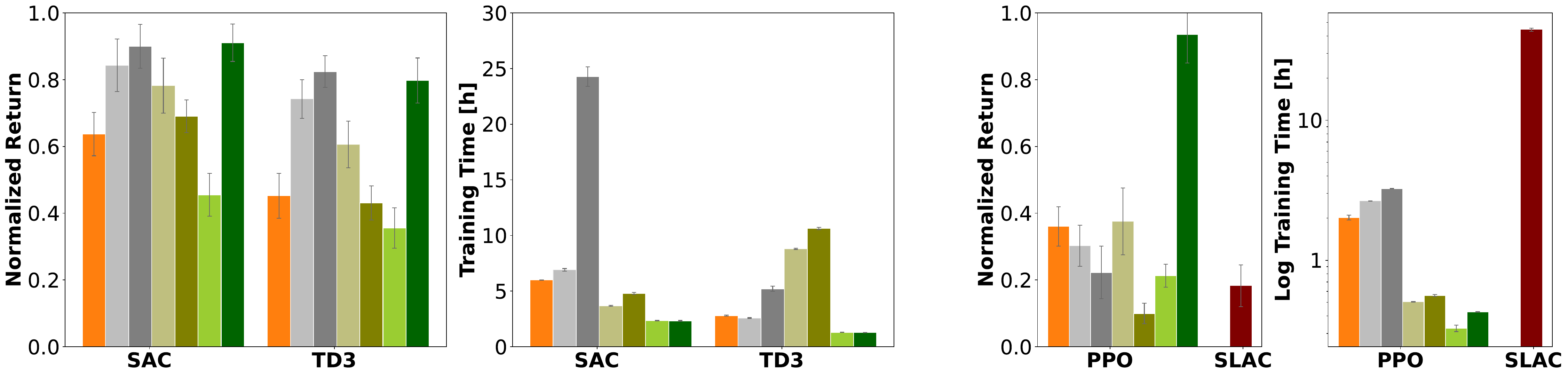}
     \includegraphics[width=\textwidth]{results/legend.pdf}
     \vspace{-15pt}
    \caption{
    Results show the mean and confidence interval of the normalized return and the training time needed for 1 Mio. steps across high-dimensional -- \textbf{TOP} row; Ant and Humanoid -- and low-dimensional \mbox{-- \textbf{BOTTOM} row; Hopper, Walker and HalfCheetah --} \gls{pomdp} Gym tasks. Comparison of different \gls{rl} agents on \gls{pomdp} tasks using our stateful gradient estimator, \gls{bptt} with a truncation length of 5 and 32, the SLAC algorithm and stateless versions of the algorithms using a window of observation with length 5 and 32. The "Oracle" approach is a vanilla version of the algorithm using full-state information.  }
    \label{fig:overall_results_low_high_dim}
\end{figure}

\clearpage
\subsection{Imitation Learning Experiments}\label{app:imitaion_learning}
In this section, we present further results of the \gls{il} tasks from the main paper.
As can be seen in Figure \ref{fig:imitation_learning_gail_steps} and \ref{fig:imitation_learning_lsiq_steps}, our results show that the \gls{s2pg} version of \gls{gail} and \gls{lsiq} -- GAIL-RS and LSIQ-RS -- are competitive with the \gls{bptt} variants, even when the truncation length is increased. In the HumanoidPOMDP-Walk task, GAIL-RS is even able to outperform by a large margin the \gls{bptt} variant.
As stressed in the paper, time the advantage of using \gls{s2pg} instead of \gls{bptt} is more evident in terms of computation time. As \gls{lsiq} runs \gls{sac} in the inner \gls{rl} loop, and \gls{gail} runs \gls{ppo}, the runtimes of these algorithms are very similar to their \gls{rl} counterparts. For the Atlas experiments, it can be seen that the vanilla policy in \gls{gail} is already performing well, despite being outperformed by our method. This shows that this task does not require much memory. In contrast, the humanoid tasks require much more memory, as can be seen by the poor results of the vanilla approaches. Interestingly, we found the \gls{gail}-\gls{bptt} variants are often very difficult to train on these tasks, which is why they got unstable. 

\begin{figure}[bt]
    \begin{multicols}{3}
    \includegraphics[scale=0.283]{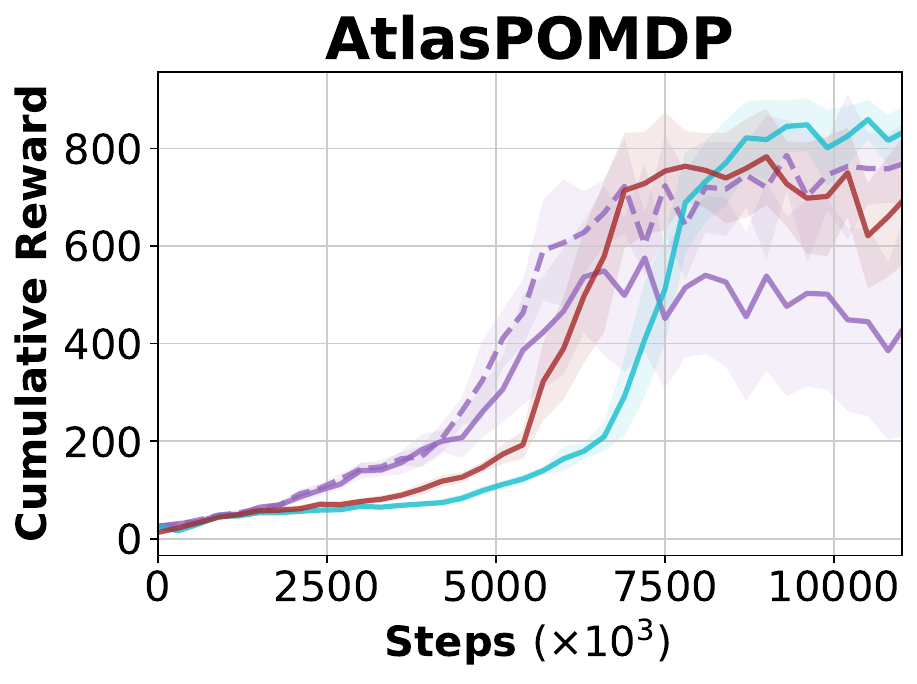}\columnbreak
    \includegraphics[scale=0.283]{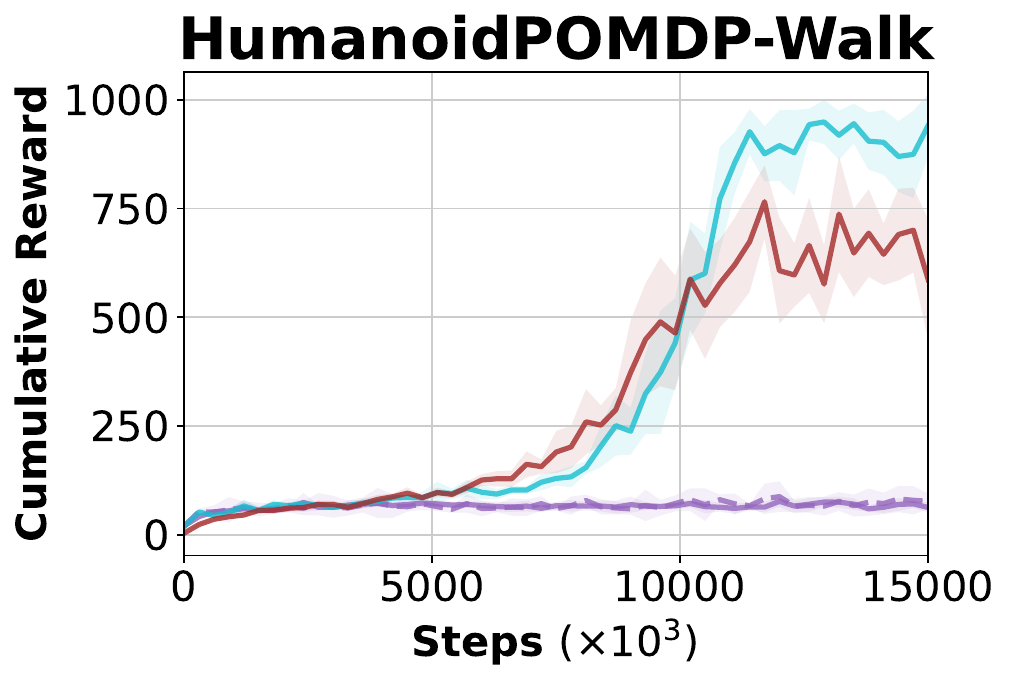}\columnbreak
    \includegraphics[scale=0.283]{results/imitation_learning/gail/env___HumanoidPOMDP-Run_deterministic_R_mean_step.pdf}
    \end{multicols}
    \vspace{-15pt}
    \begin{center}
        \includegraphics[scale=0.35]{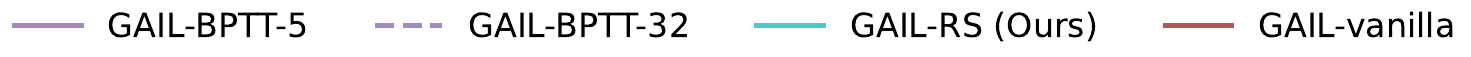}
    \end{center}
    \vspace{-10pt}
    \caption{\gls{gail} results on all \gls{il} tasks with partially observable policies and privileged critics. The number behind the \gls{bptt} implementations indicates the truncation length. Abscissa shows the cumulative reward. Ordinate shows the number of training steps ($\times 10^3$).}
    \label{fig:imitation_learning_gail_steps}
\end{figure}

\begin{figure}[bt]
    \begin{multicols}{3}
    \includegraphics[scale=0.29]{results/imitation_learning/lsiq/env___AtlasPOMDP_deterministic_R_mean_step.pdf}\columnbreak
    \includegraphics[scale=0.29]{results/imitation_learning/lsiq/env___HumanoidPOMDP-Walk_deterministic_R_mean_step.pdf}\columnbreak
    \includegraphics[scale=0.29]{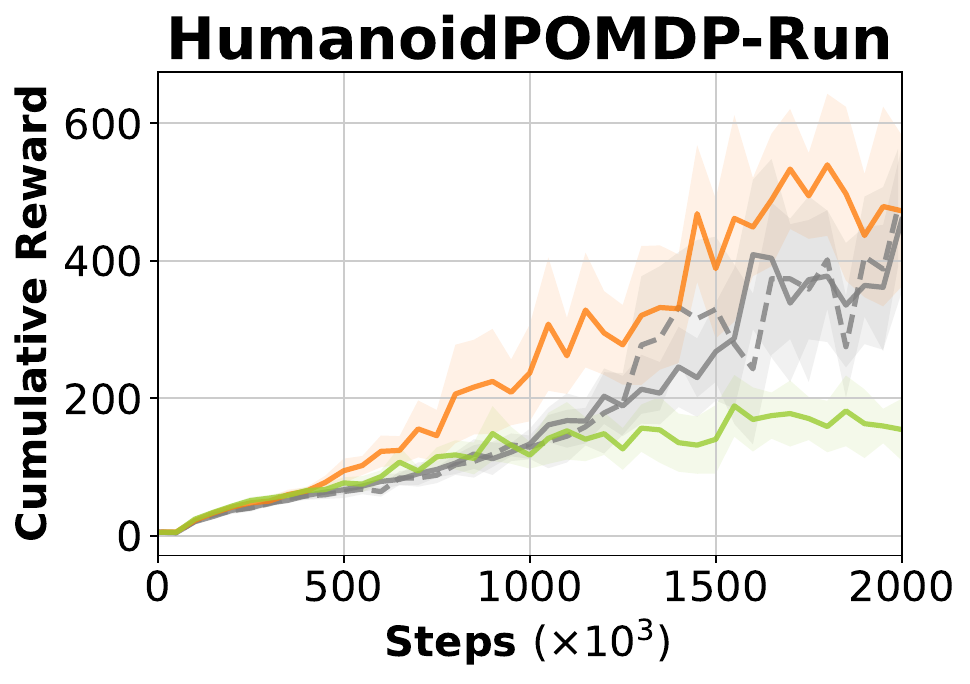}
    \end{multicols}
    \vspace{-15pt}
    \begin{center}
        \includegraphics[scale=0.35]{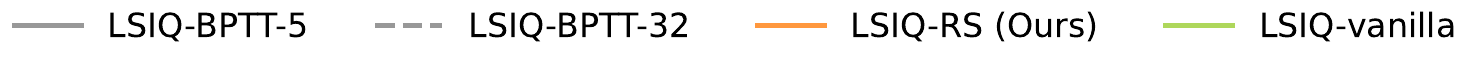}
    \end{center}
    \vspace{-10pt}
    \caption{\gls{lsiq} results on all \gls{il} tasks with partially observable policies and privileged critics. The number behind the \gls{bptt} implementations indicates the truncation length. Abscissa shows the cumulative reward. Ordinate shows the number of training steps ($\times 10^3$).}
    \label{fig:imitation_learning_lsiq_steps}
\end{figure}

\end{document}